\documentclass[letterpaper,twocolumn,10pt]{article}
\usepackage{usenix}

\usepackage{tikz}
\usepackage{amsmath}

\usepackage{filecontents}
\usepackage[available,functional,reproduced]{usenixbadges}
\usepackage{amssymb}
 \usepackage{subcaption}
  \usepackage{threeparttable}
\usepackage{enumitem}
 \usepackage{graphicx}
 \usepackage{algorithm}
\usepackage{algorithmic}
\usepackage{bbm,bm}
 \usepackage{multicol,multirow}
 \usepackage{lipsum,color}
\usepackage{amsthm}
 \usepackage[most]{tcolorbox}
 \usepackage{colortbl}
 \usepackage[dvipsnames]{xcolor}
\newtheorem{theorem}{Theorem}
\newtheorem{definition}{Definition}
\newtheorem{lemma}{Lemma}

 \newtheorem{assumption}{Assumption}
 
 \newtheorem{corollary}{Corollary}
\usepackage{makecell}

\usepackage[english]{babel}
\begin{document}

\date{}

\title{\Large \bf Low-Cost Hard-Label Adversarial Attack with Theoretical Foundations}
\author{
{\rm Jun Liu}$^{\dagger\ddagger}$,
{\rm Leo Yu Zhang}$^{\S}$,
{\rm Fengpeng Li}$^{\dagger}$,
{\rm Isao Echizen}$^{\mathparagraph\ddagger}$,
{\rm Jiantao Zhou}$^{\dagger *}$\\[0.5em]
$^{\dagger}$University of Macau,
$^{\ddagger}$National Institute of Informatics,
$^{\S}$Griffith University,
$^{\mathparagraph}$University of Tokyo\\
$^{*}$Corresponding author
}

\maketitle

\begin{abstract}
Hard-label black-box attacks, relying solely on top-1 predictions, represent one of the most challenging yet practically threat models. Despite recent progress, existing approaches face two key limitations: (1) they overlook the critical role of initialization, focusing primarily on optimization strategies; and (2) they rely heavily on empirical heuristics without theoretical guarantees. To bridge this gap, we establish a unified theoretical framework showing that existing sign-flipping hard-label attacks can be understood as approximating the true gradient sign. Guided by this principled analysis, we propose a novel attack framework featuring a zero-query initialization strategy and a Pattern-Driven Optimization (PDO) algorithm. We provide theoretical guarantees that our initialization yields higher cosine similarity to the true gradient sign than random baselines, and our PDO module achieves significantly lower query complexity than baseline search methods. Extensive experiments across CIFAR-10, ImageNet, and ObjectNet—covering standard and adversarially trained models, commercial APIs, and CLIP models—demonstrate that our method consistently outperforms SOTA hard-label attacks in both success rate and efficiency, particularly under low query budgets. Furthermore, our method demonstrates robust generalization across corrupted data (ImageNet-C), biomedical images (PathMNIST), and dense prediction tasks such as segmentation. Notably, it bypasses the stateful defense Blacklight, achieving a $0\%$ detection rate. 
\end{abstract}

\section{Introduction}\label{sec:intro}

\begin{figure}
\centering
\includegraphics[width=0.35\textwidth]{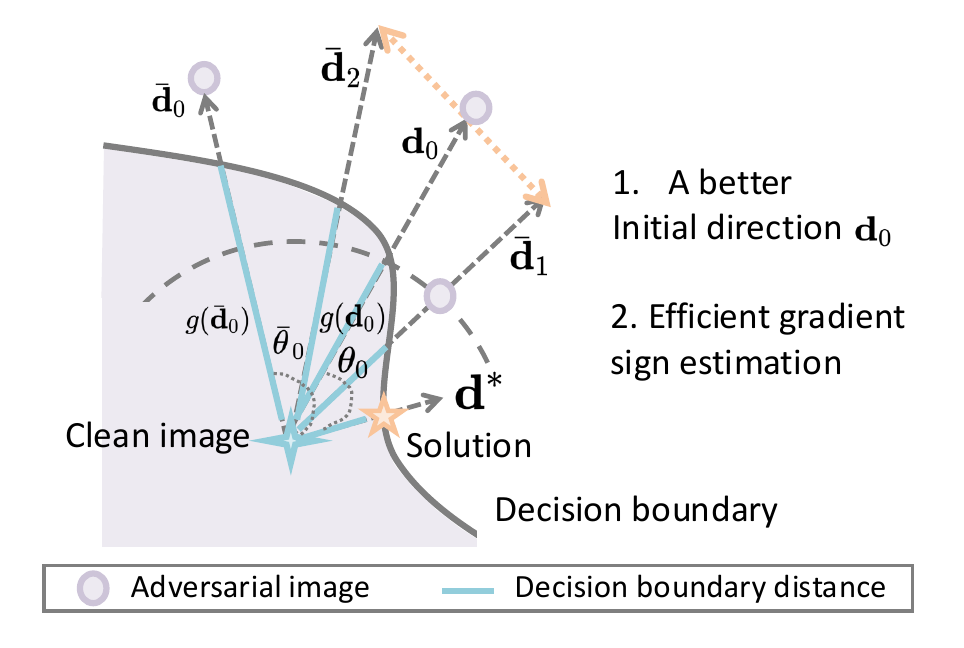}
  \caption{Comparison of our approach with traditional methods. Our method achieves a better initialization $\mathbf{d}_0$	
  with smaller angular deviation $\theta_0$	
  and boundary distance $g(\mathbf{d_0})$ compared to the traditional initialization  
$\bar{\mathbf{d}}_0$. Furthermore, it employs a more query-efficient strategy to approximate the ideal descent direction $\mathbf{d}^*$.}
  \label{fig:optill}
\end{figure}
Deep Neural Networks (DNNs) have been widely applied in real-world applications such as image classification \cite{GoogleVision,Baidu}, speech recognition \cite{zheng2021black}, face recognition \cite{deng2019arcface,jin2024faceobfuscator}, and etc. However, previous researches \cite{xu2025one,li2024dat,DBLP:journals/tifs/LiLWTZ25,fang2024zero} reveal that DNNs can fail catastrophically when processing Adversarial Examples (AEs)—carefully crafted inputs that deceive DNNs while appearing benign to human observers. In image classification tasks, adversarial attacks for generating AEs can be categorized based on the adversary's knowledge of the victim model. Broadly, they fall into three classes: white-box \cite{dong2018boosting,croce2020minimally,zhang2020walking}, gray-box \cite{guo2018countering}, and black-box attacks \cite{brendel2018decision,andriushchenko2020square,Wang_Zuo_Huang_Chen_2025}. In \textit{white-box} attacks, the adversary has complete access to the victim model's architecture and parameters, while \textit{gray-box} attacks assume partial access. In contrast, \textit{black-box} attacks assume no knowledge of either the architecture or internal parameters of the victim model and can be further classified by the type of output they leverage: \textit{soft-label} (namely score-based) attacks \cite{al2020sign,guo2019simple} exploiting classification confidence scores (e.g., softmax outputs) and \textit{hard-label} (a.k.a. decision-based) attacks \cite{chen2020hopskipjumpattack,wan2024bounceattack,vo2022ramboattack} relying solely on the top-$1$ predicted labels. 
Among these attacks, hard-label attacks represent the most challenging and practical scenario~\cite{chen2020hopskipjumpattack,wan2024bounceattack,cheng2020sign,chengquery,Wang_Zuo_Huang_Chen_2025}. Real-world applications---such as industry AI systems (e.g., in medical imaging \cite{mediacy}) and commercial APIs ({e.g., Core ML \cite{coreml}})---intentionally suppress output confidences due to data privacy regulations and mobile resource constraints, thereby necessitating decision-only attack strategies and making them crucial for evaluating and improving the robustness of DNNs. Hence, this paper primarily focuses on hard-label attacks.

\textbf{Traditional hard-label attack workflow.} Current hard-label attacks are broadly divided into surrogate-assisted and surrogate-free methods. Surrogate-assisted methods (or transfer-query attacks) leverage gradient priors from local models to guide the search \cite{shen2024transferability}. Although efficient, their performance degrades significantly under architectural mismatch (e.g., Convolution Neural Networks (CNNs) vs. Transformers) or domain shifts, and can be mitigated by ensemble training \cite{tramer2018ensemble}. In contrast, surrogate-free (query-based) methods operate independently of external priors, relying solely on the victim model's hard-label feedback to estimate gradients or refine perturbations. The workflow of these query-based methods typically involves two stages:

\noindent\textit{Stage 1: Perturbation direction initialization.} This stage aims to generate an initial search direction. Traditional methods randomly sample noise or another image until an AE is found. The vector from the clean image to this AE defines the initial direction. Then, a binary search is performed along this vector to locate a \textit{starting point} close to the decision boundary.

\noindent\textit{Stage 2: Direction refinement and perturbation reduction.} Starting from the boundary point, this stage iteratively refines the initial direction and reduces the perturbation magnitude based on hard-label feedback. The objective is to identify an AE that satisfies the specified perturbation constraints. If no valid example is found within the allocated query budget, the attack is deemed unsuccessful.

\textbf{Observations and insights.}
\textbf{1) Better initialization.}
We observe that most existing black-box attacks devote substantial effort to Stage~2 (the search phase), while largely overlooking the importance of Stage~1 (initialization).
In practice, they typically rely on random strategies, such as Gaussian or Uniform noise~\cite{chen2020hopskipjumpattack,wan2024bounceattack,maho2021surfree,cheng2019query}, random images~\cite{wan2024bounceattack,ma2021finding}, or all $\bm1$ perturbations~\cite{Wang_Zuo_Huang_Chen_2025,10.1145/3394486.3403225}.
These choices ignore prior knowledge about model vulnerabilities, often producing initial directions that are far from the decision boundary and difficult to refine under limited query budgets (e.g., $<100$ queries).
This motivates us to design a prior-guided initialization that yields a superior starting direction $\mathbf{d}_0$.
As illustrated in Fig.~\ref{fig:optill}, our goal is to obtain an initialization that is closer to the ideal descent direction $\mathbf{d}^*$, both in angular deviation and in decision-boundary distance, compared to conventional random starts. \textbf{2) More efficient gradient sign estimation.}
The next challenge is to efficiently approach the ideal descent direction $\mathbf{d}^*$. Inspired by the score-based method SignHunter~\cite{al2020sign}, we identify $\mathbf{d}^*$ with the true gradient sign utilized in FGSM~\cite{fgsm}, a pragmatic approximation given that exact gradient retrieval is intractable in the black-box setting. Specifically, SignHunter demonstrates that, with sufficient queries, it can approximate the gradient sign used in FGSM. Similarly, we observe that RayS~\cite{10.1145/3394486.3403225} can asymptotically achieve FGSM performance by iteratively refining sign estimates through sign-flipping. However, in practice, the efficiency of such methods deteriorates rapidly as the dimensionality grows.
Moreover, when refining our structured initial directions, RayS-style rigid equal partitioning ignores the underlying spatial pattern, frequently breaking coherent regions into fragmented pieces.
This leads to \emph{gain cancellation}: when a partition contains conflicting gradient signs, their contributions offset one another, thereby weakening the effective descent signal and slowing the search process. These observations suggest that efficient gradient sign approximation requires not only reducing the search dimensionality, but also preserving spatial coherence throughout the search process.

\textbf{Our work.}
From a \emph{greedy optimization} perspective, we propose an efficient black-box attack that jointly improves initialization and search.
First, we design a superior initialization strategy via Block-wise Discrete Cosine Transform (BDCT) analysis~\cite{wallace1991jpeg}.
By characterizing model sensitivity across frequency bands, we derive a data-driven sampling distribution from the clean image's frequency statistics and generate an initial direction via inverse BDCT.
Both theoretical analysis and empirical results show that this initialization is significantly closer to the decision boundary and exhibits higher cosine similarity with the true gradient sign than random baselines (Sec.~\ref{sec:betinit}). Building upon this structured initialization, we further introduce a \emph{Pattern-Driven Optimization} (PDO) module.
PDO explicitly preserves the spatial coherence of initial directions throughout the search process, thereby reducing the effective search dimensionality and mitigating gain cancellation.
As a result, it approximates the true gradient sign with substantially fewer queries than the equal partitioning searches used in state-of-the-art (SOTA) baselines (Sec.~\ref{sec:pdo}).

We summarize our main {contributions} as follows:
\begin{itemize}
    \item We provide the first theoretical analysis showing that RayS-like sign-flipping hard-label black-box attacks can asymptotically match white-box FGSM. Our analysis identifies search dimensionality and cumulative sign-estimation errors as the primary bottlenecks limiting practical efficiency.

    \item We propose a zero-query initialization strategy grounded in image frequency statistics. Unlike blind random sampling, this method leverages prior knowledge to locate a starting point with a smaller decision boundary distance and better alignment with the true gradient sign.

    \item We design the PDO method, which reduces the search space and mitigates the gain cancellation effect inherent in baseline equal-partitioning search. We provide theoretical guarantees on its improved query complexity and experimentally demonstrate its faster convergence and higher cosine similarity with the true gradient sign.

    \item Extensive experiments demonstrate that our proposed DPAttack significantly outperforms existing algorithms in terms of ASR and query efficiency. It proves highly effective against traditional DNNs, foundation models (e.g., CLIP), and real-world APIs, while successfully bypassing SOTA defenses like Blacklight.

\end{itemize}

\section{Related Work}
\label{sec:related_work}

\subsection{Problem Formulation \& Background}
Given a victim model $\mathcal{F}:\mathbb{R}^d\to\{1,\cdots,M\}$ (where $M$ denotes the number of classes), a clean image $\mathbf{x}\in\mathbb{R}^d$, and its ground-truth (GT) label $y$, an untargeted adversarial attack aims to generate an adversarial example (AE) $\mathbf{x}^* = \mathbf{x}+\Delta$ such that $\mathcal{F}(\mathbf{x}^*)\neq y$. The objective is to maximize the classification loss $\mathcal{L}(\mathbf{x}+\Delta,y)$ while constraining the perturbation magnitude within an allowed $\ell_p$-ball (typically $p \in \{2, \infty\}$) of radius $\epsilon$. Formally:
\begin{equation}\label{eq:lossobj}
    \max_{\Delta} \mathcal{L}(\mathbf{x}+\Delta,y), \quad \text{s.t.} \quad \|\Delta\|_p \leq \epsilon.
\end{equation}
In white-box settings, attackers optimize Eq.~(\ref{eq:lossobj}) using gradients, e.g., the Fast Gradient Sign Method (FGSM)~\cite{fgsm}: $\mathbf{x}'=\mathbf{x}+\epsilon\cdot\text{sgn}(\nabla\mathcal{L}(\mathbf{x},y))$. However, in the {hard-label black-box} setting, the attacker only accesses the top-1 predicted label, rendering direct gradient access impossible.

\subsection{Hard-Label Black-Box Attacks}
\label{subsec:hard_label_attacks}

Early research, such as the Boundary Attack~\cite{brendel2018decision}, relied on heuristic random walks along the decision boundary. While pioneering, it suffered from extremely low query efficiency due to the blindness of its search strategy. To improve convergence, subsequent works incorporated gradient or normal vector estimation to guide the attack. Canonical methods like HSJA~\cite{chen2020hopskipjumpattack} and Bounce~\cite{wan2024bounceattack} utilize Monte Carlo sampling to approximate the decision boundary normal, enabling more directed perturbation updates. Building on this, TtBA~\cite{wangttba} weights the perturbation direction with its estimated normal vector. It utilizes a robust ``two-third'' update rule and curvature-based detection to escape local optima, achieving superior efficiency in normal-vector-based attacks.

However, these gradient-estimation-based methods typically require substantial query budgets to approximate normal vectors. 
Distinct from this direction, a parallel line of optimization-based approaches bypasses explicit normal estimation. Opt~\cite{cheng2019query} pioneered this by formulating the attack as a zero-order optimization task. In line with this approach, SignOpt~\cite{cheng2020sign} further improved query efficiency by reconstructing the gradient direction solely from the aggregated signs of directional derivatives along random projections, rather than estimating magnitude. Specifically, SignOpt reformulates the attack objective in Eq.~(\ref{eq:lossobj}) as finding a unit vector $\mathbf{v}\in\mathbb{S}^{d-1}$ (where $\mathbb{S}^{d-1}:=\{\mathbf{v}\in\mathbb{R}^d:\|\mathbf{v}\|_2=1\}$) to minimize:
    \begin{equation}\label{eq:signoptobj}
    \begin{gathered}
   g(\mathbf{v})=\min\{r>0\ | \ \mathcal{F}(\mathbf{x}+r\cdot\mathbf{v})\neq y\}.
        \end{gathered}
    \end{equation}
    Here, $g(\cdot)$, defined on $\mathbb{S}^{d-1}$, represents the shortest distance $r$ from the clean image $\mathbf{x}$ to the decision boundary along direction $\mathbf{v}$. SignOpt solves Eq.~(\ref{eq:signoptobj}) by estimating the gradient $\nabla_\mathbf{v} g(\mathbf{v})$ and updating $\mathbf{v}$ using standard Stochastic Gradient Descent (SGD) as follows:\footnote{The minimization over $r$ is performed using a binary search.}

\noindent{1. Sample multiple directions $\{\mathbf{n}_i\}_{i=1}^{N}$ from a Gaussian distribution}, typically setting $N=200$ to mitigate noise. 

\noindent{2. Obtain the sign of the directional derivative} via finite differences:\footnote{Note that calculating this directional derivative only requires access to the hard-label oracle $\mathcal{F} (\mathbf{x}+r\mathbf{v})\neq y$.} $\text{sgn}\left( \nabla_{\mathbf{n}_i}g(\mathbf{v})) = \text{sgn}([g(\mathbf{v} + \tau\mathbf{n}_{i} ) - g(\mathbf{v})]/ \tau \right)$, where $\tau$ is a small constant.

\noindent{3. Approximate $ \nabla_\mathbf{v} g(\mathbf{v}) \approx \sum_{i=1}^N \text{sgn}\left( \nabla_{\mathbf{n}_i}g(\mathbf{v}) \right) \cdot \mathbf{n}_{i} $ and update $\mathbf{v}$ using SGD with a learning rate $\eta$.}

Despite these advancements, continuous gradient-recovery approaches still struggle with high-dimensional search spaces, particularly under $\ell_\infty$ constraints. This limitation motivated the development of following SOTA discrete search strategies.

\subsection{Discrete Hard-Label Attacks}
\label{subsec:discrete_attacks}

Unlike the aforementioned methods that attempt to approximate continuous gradients, a distinct category of attacks, represented by RayS~\cite{10.1145/3394486.3403225} and ADBA~\cite{Wang_Zuo_Huang_Chen_2025}, abandons explicit gradient estimation. Instead, these methods directly search for the optimal perturbation direction within a discrete space (i.e., $\{-1, 1\}^d$), leveraging the empirical observation that $\ell_\infty$-bounded adversarial
perturbation often lies on the vertices of the $\ell_\infty$ ball~\cite{moon2019parsimonious} to accelerate the search process. As our proposed framework follows and improves upon this discrete search paradigm, we review these two baselines in detail.

\noindent\textbf{RayS~\cite{10.1145/3394486.3403225}.} 
RayS converts continuous optimization into a discrete search problem over $\mathcal{H}\equiv\{-1,+1\}^d$. To this end, its optimization objective is:
     \begin{equation}\label{eq:newobjRayS}
    \begin{gathered}
  \min_{\mathbf{d}\in\mathcal{H}} g(\mathbf{d}) \quad \text{where} \ \mathcal{H}\equiv\{-1,+1\}^d.
        \end{gathered}
    \end{equation} 
Naive RayS (NRayS) heuristically optimizes Eq.~(\ref{eq:newobjRayS}) by flipping the sign of each dimension in the search space $\mathcal{H}$, while maintaining and updating the minimum perturbation magnitude $g(\mathbf{d}_{\text{best}})$ and its corresponding direction $\mathbf{d}_{\text{best}}$. It initializes the direction as an all-one vector, i.e. $\mathbf{d}_{\text{best}}=\mathbf{d}_0=\mathbf{1}$, and identifies $g(\mathbf{d}_0)$ via binary search.

To circumvent the inefficiency of dimension-wise flipping in NRayS, the authors proposed Hierarchical RayS (HRayS).
HRayS iteratively refines $\mathbf{d}_{\text{best}}$ through a hierarchical dyadic structure, effectively exploiting spatial correlations in image gradients~\cite{ilyas2019prior}. Specifically, HRayS constructs a standard dyadic search tree $\mathsf{T}_{\mathrm{dyad}}$ via the following three steps:


\noindent{\textbf{Step 1. Index partitioning.}} 
At each hierarchy level $l \in \{1, 2, \cdots,\lceil \log_2 d \rceil\}$, the index set $\Gamma=\{1,2,\cdots,d\}$ of vector $\mathbf{d}_{\text{best}}$ is sequentially partitioned into $2^l$ disjoint subsets $\{B_1^{(l)},B_2^{(l)},\cdots,B_{2^l}^{(l)}\}$ of equal length $d/{2^l}$.

\noindent{\textbf{Step 2. Sign flipping and direction update.}} 
For each subset $B_{k}^{(l)}$ ($k \in [1, 2^l]$), the signs of the entries in $\mathbf{d}_{\text{best}}$ corresponding to this subset are flipped, yielding a perturbed direction $\mathbf{d}_k^{(l)}$. The current best direction $\mathbf{d}_{\text{best}}$ is updated \emph{only if} the flipped direction yields a smaller boundary distance. That is
\begin{equation}
\label{eq:Decision1}
\begin{gathered}
\mathbf{d}_{\text{best}}=\begin{cases}  \mathbf{d}_k^{(l)}
 & \text{ if } g(\mathbf{d}_k^{(l)})< g(\mathbf{d}_{\text{best}})\\
 \mathbf{d}_{\text{best}}  & \text{Otherwise.}  \\
\end{cases}
\end{gathered}
\end{equation}
After traversing all $2^l$ subsets, $B_{ks}^{(l)}$ is defined as the union of subsets chosen for updating $\mathbf{d}_{\text{best}}$.

\noindent\textbf{Step 3. Refinement.} Steps 1 and 2 are repeated by increasing $l$ by $1$, progressively refining the direction at finer granularities, until $l= \lceil \log_2 d \rceil$, the query limit is reached, or $g(\mathbf{d}_{\text{best}})$ satisfies the perturbation budget.

\noindent\textbf{ADBA~\cite{Wang_Zuo_Huang_Chen_2025}.} 
As the current SOTA for untargeted $\ell_\infty$ attacks, ADBA follows the hierarchical structure of RayS but identifies a critical inefficiency: RayS relies on binary search to calculate precise boundary distances $g(\cdot)$ for every candidate, which is query-expensive.
ADBA improves this by comparing \emph{pairs} of search directions. 

\noindent{\textbf{Step 1. Index partitioning:}} Identical to Step 1 in HRayS.

\noindent{\textbf{Step 2. Sign flipping and direction update.}} 
For each pair of adjacent subsets $({B}_{2k-1}^{(l)}, {B}_{2k}^{(l)})$ where $k\in[1,2^{l-1}]$, the signs of entries in $\mathbf{d}_{\text{best}}$ (initialized as $\mathbf{d}_0=\bm1$) corresponding to these two groups are flipped respectively, yielding two candidate directions $\bar{\mathbf{d}}_1$ and $\bar{\mathbf{d}}_2$. A comparative function $\Lambda(\cdot, \cdot)$ determines which direction is closer to the boundary only when both candidates are adversarial:
\begin{small}
\begin{equation}\label{eq:Decision2}
\begin{gathered}
\mathbf{d}_{\text{best}}=\begin{cases}
 \mathbf{d}_{\text{best}} & \text{ if } \mathcal{F}(\mathbf{x}+r\bar{\mathbf{d}}_1 )=y\ \&  \ \mathcal{F}(\mathbf{x}+r\bar{\mathbf{d}}_2 )=y \\
  \bar{\mathbf{d}}_1 & \text{ if } \mathcal{F}(\mathbf{x}+r\bar{\mathbf{d}}_1 )\neq y\ \&  \ \mathcal{F}(\mathbf{x}+r\bar{\mathbf{d}}_2 )=y  \\
  \bar{\mathbf{d}}_2 & \text{ if } \mathcal{F}(\mathbf{x}+r\bar{\mathbf{d}}_1 )= y\ \&  \ \mathcal{F}(\mathbf{x}+r\bar{\mathbf{d}}_2 )\neq y  \\
  \Lambda (\bar{\mathbf{d}}_1,\bar{\mathbf{d}}_2) & \text{Otherwise,} 
\end{cases}
\end{gathered}
\end{equation}
\end{small}where $r$ denotes the current minimal perturbation distance (initialized to $1$ for $[0,1]$-valued images), and $\Lambda(\cdot, \cdot)$ returns the direction with the smaller decision boundary distance. 
The function $\Lambda$ iteratively checks whether perturbing $\mathbf{x}$ by a gradually reduced $r$ along both $\bar{\mathbf{d}}_1$ and $\bar{\mathbf{d}}_2$ maintains adversarial status. This process stops at $r^*$ if a direction fails to produce an AE or a step limit (e.g., 8) is reached. After determining $\mathbf{d}_{\text{best}}$ and updating $r$ to $r^*$, the algorithm returns to Step 1.

\noindent\textbf{Step 3. Refinement.} Identical to Step 3 in HRayS.

\noindent\textit{Limitation:} Although ADBA substantially improves query efficiency over RayS, it relies on a naive all-ones initialization ($\mathbf{d}_0=\bm{1}$), which limits its starting efficiency. More importantly, ADBA remains a purely empirical method that lacks theoretical guarantees regarding its convergence or gradient approximation accuracy. In contrast, our work establishes a theoretically grounded framework that combines a frequency-guided initialization with a structured, pattern-aware optimization strategy, yielding both stronger guarantees and higher practical efficiency.

\section{Theoretical Foundations: From Discrete Search to Gradient Sign Estimation}
\label{sec:demon}

To address the lack of theoretical guarantees in RayS-like hard-label attacks, we draw inspiration from their soft-label counterpart SignHunter~\cite{al2020sign}. While SignHunter is proven to approximate FGSM, its proof relies on score-based derivative estimation, which is unavailable in hard-label settings. To bridge this gap, we establish a novel theoretical framework demonstrating that RayS-based algorithms (specifically NRayS and HRayS) asymptotically approximate FGSM. As ADBA is a derivative of these core algorithms, our analysis focuses on the fundamental RayS mechanism. We first derive a probabilistic lower bound for the alignment between a random binary vector and the true gradient sign. This provides a key building block for understanding why discrete sign-search methods can recover meaningful gradient information in high dimensions.

\begin{tcolorbox}[colback=blue!5!white,colframe=black!75!black
,
boxsep=0pt,
  left=4pt,
  right=4pt,
  top=4pt,
  bottom=4pt,
  boxrule=0.7pt,
  arc=2pt
]
\begin{theorem}[Approximation Lower Bound for NRayS]
\label{theo:angle}
   
    Let $\mathbf{u}\in\mathbb{S}^{d-1}$ be an arbitrary unit vector.
    Let $\{\mathbf{d}_j\}_{j=1}^m$ be $m$ vectors sampled independently and uniformly from the binary hypercube $\mathcal{H}\equiv\{-1,+1\}^d$, and let $\hat{\mathbf{d}}_j=\mathbf{d}_j/{\sqrt{d}}$ be their normalized counterparts.
    For any precision $\zeta \in (0,1)$ and failure probability $\delta \in (0,1)$, there exists a constant $C \ge 1$ such that if the sample size satisfies $m \geq -(\ln{\delta})e^{Cd\zeta^2}$, then with probability at least $1-\delta$, there exists at least one index $j \in \{1,\dots,m\}$ satisfying:
      \begin{equation}\label{eq:theo1_1_main}
    \begin{gathered}
\mathbf{u}^\top\hat{\mathbf{d}}_j \geq \zeta.
\end{gathered}
    \end{equation}
\end{theorem}
\end{tcolorbox}

\noindent\textit{Proof.} See Appendix ~\ref{app:theo1_proof}. By instantiating Theorem~\ref{theo:angle} with the normalized gradient $\mathbf{u}=\nabla\mathcal{L}/\|\nabla\mathcal{L}\|_2$, we conclude that \emph{NRayS, given a sufficient query budget $m$, guarantees finding a direction $\mathbf{d}\in \mathcal{H}$ that aligns with the true gradient $\nabla\mathcal{L}$ with high probability, thereby approximating the FGSM update.} While the theoretical bound implies high query complexity (e.g., assuming $C=1$, ensuring $\zeta\geq0.1$ with $50\%$ probability on $3\times32\times 32$ sized CIFAR-10~ \cite{krizhevsky2009learning} images requires $m \approx 1.5^{13}$ queries), empirical findings from SignHunter suggest that strict gradient alignment is unnecessary. In practice, partial alignment---recovering the signs of the top $20-30\%$ of gradients---suffices to achieve attack success rates exceeding $70\%$ on datasets like CIFAR-10 and ImageNet \cite{deng2009imagenet}.

Theorem~\ref{theo:angle} establishes the theoretical feasibility of using discrete search to approach the true gradient sign. However, purely random search remains inefficient in high dimensions. Extending this insight, Theorem~\ref{theo:raysappro} demonstrates how HRayS leverages structured estimation to bypass this bottleneck. It shows that HRayS constructs a direction aligned with the gradient sign, where the approximation quality is governed by the subset selection error rather than blind sampling.

\begin{definition}[Subset-wise Directional Derivative Influence]
\label{def:def1}
For a selected index subset $B_{ks}^{(l)}$ at round $l$ of HRayS (following the definition in Step 2 of HRayS; see Sec.~\ref{subsec:discrete_attacks}), we define its influence as the cumulative magnitude of the gradient components within that subset:
\begin{equation}\label{eq:theo2_proof2}
    I(B_{ks}^{(l)}) := \sum_{i\in B_{ks}^{(l)}} \left| \frac{\partial \mathcal{L}(\tilde{\mathbf{x}},y)}{\partial \tilde{\mathbf{x}}[i]} \bigg|_{\tilde{\mathbf{x}}=\mathbf{x}+g(\mathbf{d}_{\text{best}})\mathbf{d}_{\text{best}}} \right|.
\end{equation}
\end{definition}
This metric quantifies the collective contribution of dimensions in $B_{ks}^{(l)}$ to increasing the loss. Based on this influence measure, we establish the following subset selection
guarantee for HRayS.

\begin{lemma}
\label{lem:lem1}
Under standard regularity conditions (Assumptions~\ref{ass1} and~\ref{ass1_2} in App.~\ref{app:theo2_proof}),
the subset $B^{(l)}_{ks}$ selected by HRayS at round $l$ satisfies:
\begin{equation}
I(B_{ks}^{(l)}) \geq \max\nolimits_{B_k^{(l)}} I(B_{k}^{(l)}) - \eta_l,
\end{equation}
where $\eta_l$ represents the selection error bound at round $l$.
\end{lemma}

\noindent\textit{Proof.} See Appendix~\ref{app:theo2_proof}. This result guarantees that each HRayS round identifies a high-impact subset, enabling the cumulative progress analysis in the following theorem.

\begin{tcolorbox}[colback=blue!5!white,colframe=black!75!black,
boxsep=0pt,
  left=4pt,
  right=4pt,
  top=4pt,
  bottom=4pt,
  boxrule=0.7pt,
  arc=2pt
]
\begin{theorem}[Gradient Approximation Guarantee for HRayS]
\label{theo:raysappro}
Assume that the gradient magnitude is uniformly bounded, i.e.,
$\|\nabla \mathcal{L}(\mathbf{x},y)\|_\infty \le G_{\max}$.
Let $\eta_l$ denote the maximum deviation between the directional derivative
influence of the subset $B_{ks}^{(l)}$ selected by HRayS at round $l$
and that of the optimal subset $\arg\max_{B_k^{(l)}} I(B_k^{(l)})$.
Let $\hat{\mathbf{d}}^*$ be the final direction produced by HRayS.
Then, the alignment between $\hat{\mathbf{d}}^*$ and the true gradient sign satisfies
\begin{equation}
\mathrm{sgn}(\nabla \mathcal{L}(\mathbf{x},y))^\top \hat{\mathbf{d}}^*
\;\ge\;
\frac{\|\nabla \mathcal{L}(\mathbf{x},y)\|_1}{G_{\max}}
\cdot
(1-\varpi),
\end{equation}
where
$\varpi =
\sum_{l=1}^{\lceil \log_2 d \rceil}
\eta_l \|\nabla \mathcal{L}(\mathbf{x},y)\|_1^{-1}$
denotes the cumulative relative error induced by suboptimal subset selections
across all hierarchical rounds.
\end{theorem}
\end{tcolorbox}

\noindent\textit{Proof.} See Appendix (App.)~\ref{app:theo2_proof}. Theorem~\ref{theo:raysappro} reveals that although HRayS does not explicitly estimate the gradient, its hierarchical decision mechanism effectively mimics the selection of dimensions with the steepest ascent. By progressively refining the coordinate set, HRayS constructs a direction $\hat{\mathbf{d}}^*$ that closely aligns with $\text{sgn}(\nabla \mathcal{L}(\mathbf{x},y))$, provided the selection error $\varpi$ remains small.

\noindent\textbf{Implications for Our Approach.}
Theorems~\ref{theo:angle} and~\ref{theo:raysappro} suggest that the efficiency of hard-label attacks hinges on two factors: (1) the search space dimensionality (governing the bound in Theorem~\ref{theo:angle}), and (2) the cumulative selection error (quantified by $\varpi$ in Theorem~\ref{theo:raysappro}).
Motivated by these theoretical insights, we propose DPAttack, which bridges these gaps through two principled design choices: Frequency-Prior Initialization (Sec.~\ref{sec:betinit}): 
    To mitigate the high dimensionality penalty, we utilize frequency-domain priors to design a zero-query initialization. Unlike random noise, this direction exhibits local spatial coherence (i.e., blocks of consistent signs), effectively reducing the search complexity from the full pixel space to a structured subspace. Theorem~\ref{theo:ddm} further guarantees that this method yields a positive expected alignment with the true gradient sign. \textbf{Pattern-Driven Optimization (Sec.~\ref{sec:pdo}):} 
    To address the cumulative error $\varpi$ (Theorem~\ref{theo:raysappro}), we identify that the fixed partitioning in RayS and ADBA often disrupts the spatial coherence established by our initialization, leading to conflicting updates (gain cancellation). We thus employ a tailored pattern-driven search that groups pixels by their initial sign coherence. By respecting initialization's intrinsic structures, this strategy is proven to dominate baselines in sign alignment (Theorem~\ref{thm:main_text}) and reduce query complexity (Theorem~\ref{theo:complexity-fixed}).

\section{Our Proposed Attack}\label{sec:method}
\subsection{Threat Model}
\label{subsec:thrmodel}

We align our threat model with standard hard-label black-box attack settings~\cite{10.1145/3394486.3403225}, defined by the following aspects:

\noindent \textbf{Adversary's Goal.} 
The adversary aims to generate an untargeted adversarial example $\mathbf{x}^*$ that misclassifies the victim model (i.e., $\mathcal{F}(\mathbf{x}^*) \neq y$) while maintaining the perturbation within an imperceptible budget (e.g., satisfying $\ell_\infty$ or $\ell_2$-norm constraints). Formally, we adopt the same optimization objective as RayS (defined in Eq.~(\ref{eq:newobjRayS})), minimizing the distance to the decision boundary within the discrete search space.

\noindent \textbf{Adversary's Knowledge and Capabilities.} 
We assume a hard-label black-box scenario. The adversary can only interact with the victim model $\mathcal{F}$ via an oracle (e.g., an online API) to obtain the top-1 predicted label for a given input. Access to the model's architecture, parameters, gradients, or continuous output probabilities (soft labels) is strictly prohibited.

\begin{figure*}[t!]\centering
  \includegraphics[width=\textwidth]{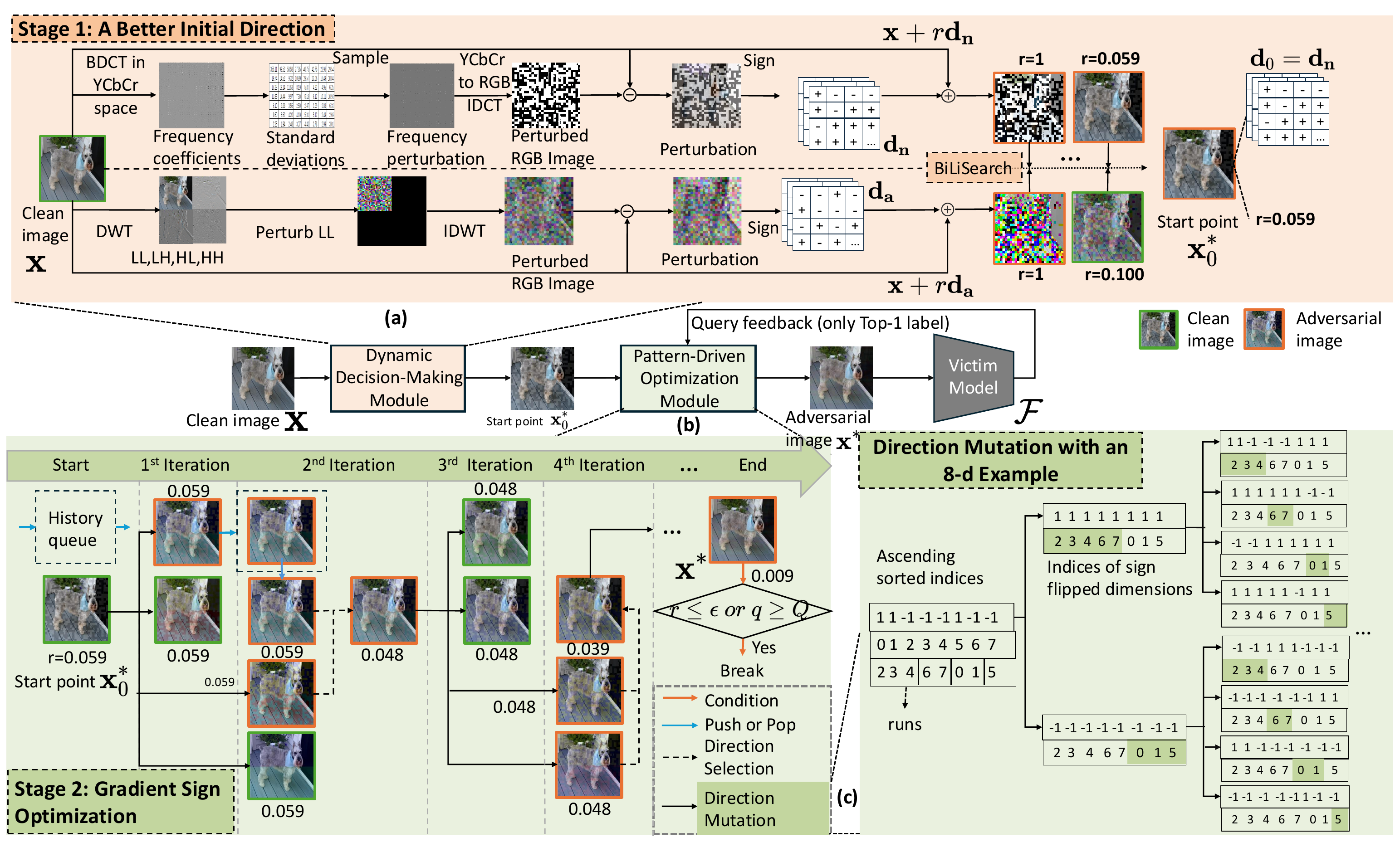}
  \caption{Framework of the proposed hard-label attack method DPAttack. (a) Stage 1: The Dynamic Decision-Making (DDM) module generates a superior initial perturbation direction $\mathbf{d}_0$. (b) Stage 2: The Pattern-Driven Optimization (PDO) module iteratively refines the direction and the perturbation magnitude $r$, yielding the final adversarial example $\mathbf{x}^*=\operatorname{clip}(\mathbf{x}+r\hat{\mathbf{d}}^*,[0,1])$.}
  \label{fig:overview}
\end{figure*}

\subsection{The Overall Attack Pipeline of DPAttack}
\label{subsec:dpaAttack}

Guided by the theoretical insights established in Sec.~\ref{sec:demon}, we propose DPAttack, a coarse-to-fine framework designed to efficiently approximate the true gradient sign in FGSM. As illustrated in Fig.~\ref{fig:overview} and Algorithm~\ref{alg:algorithmtotal}, the pipeline consists of two integrated stages: (1) The Dynamic Decision-Making (DDM) Module (Sec.~\ref{sec:betinit}), which leverages frequency priors to generate a structured and informative initialization; and (2) The Pattern-Driven Optimization (PDO) Module (Sec.~\ref{sec:pdo}), which refines this direction by exploiting spatial coherence to accelerate convergence. We detail these components below.

\subsection{Better Initialization via DDM}
\label{sec:betinit}

Our initialization approach relies on the premise that profiling model vulnerabilities provides effective prior knowledge for constructing potent initial perturbations. While previous studies~\cite{chen2022rethinking} have used Shapley values to identify frequency sensitivity, this metric is derived by masking frequency bands rather than perturbing them, making it ill-suited for capturing sensitivity to additive noise. To address this, we introduce a novel analysis method and propose the DDM module.

\subsubsection{Analysis: Block-DCT Frequency Sensitivity (BFS)}
\label{sub:bfs_analysis}

To accurately characterize model vulnerability to frequency-based noise, we propose BFS. While drawing conceptual inspiration from the frequency sensitivity analysis in~\cite{yin2019fourier}, our method significantly advances this framework by introducing two critical modifications tailored for adversarial efficiency: 
1) We employ the variation in CE loss rather than classification error rate to provide a fine-grained sensitivity metric; 2) We apply perturbations in the Block Discrete Cosine Transform (BDCT) domain~\cite{wallace1991jpeg} instead of the Fast Fourier Transform domain \cite{gonzalez2018digital}, as DNNs are typically more vulnerable to the block-wise artifacts introduced by BDCT.

\noindent{\textbf{BFS Procedure.}} 
Given an RGB image $\mathbf{x}\in\mathbb{R}^{C\times H\times W}$, we convert it to YCbCr space and divide it into non-overlapping blocks of size $w \times w$ (e.g., $w=8$). Performing DCT on these blocks yields the frequency matrix $\mathbf{I}\in\mathbb{R}^{C\times Z\times w\times w}$, where $Z=(H/w)\times(W/w)$ is the total number of blocks. 
To measure sensitivity, we define a Gaussian noise distribution $\mathcal{N}(0, \sigma^2_{\text{max}})$, where $\sigma_{\text{max}}$ is calibrated such that the perturbation exhausts the pixel-domain budget $\epsilon$. For a specific frequency coordinate $(c, i, j)$, we sample a noise vector $\hat{\mathbf{w}}_{c,i,j} \in \mathbb{R}^Z$ from this distribution and add it to the corresponding coefficients across all blocks. The perturbed image $\hat{\mathbf{x}}$ is obtained via Inverse BDCT (IBDCT) and projection:
\begin{equation}\label{eq:abft}
    \hat{\mathbf{x}}=\Pi_{\epsilon,\mathbf{x}}\big(\text{IBDCT}(\mathbf{I}[c,:,i,j]+\hat{\mathbf{w}}_{c,i,j},w)\big)
\end{equation}
where $\Pi_{\epsilon,\mathbf{x}}(\cdot)$ denotes the projection operator that maps the input onto a hypersphere centered at $\mathbf{x}$ with radius $\epsilon$.

\noindent{\textbf{Observation: The Correlation with Clean Image Statistics.}} 
We applied BFS to 1,000 ImageNet validation images across diverse architectures. Results on ResNet-50, ViT-B-32, and adversarially trained WideResNet-50 (Fig.~\ref{fig:dctfreqsens}) reveal that model sensitivity (CE loss) exhibits an oscillatory decay from low to high frequencies. Crucially, we observe that this sensitivity pattern is closely related to the frequency-domain variance of clean images. We compute the standard deviation of frequency coefficients across blocks for clean images as:
\begin{equation}\label{eq:d1_1}
    \sigma_{c,i,j}=\sqrt{\frac{1}{Z}\sum\limits_{z=0}^{Z-1}|\mathbf{I}[c,z,i,j]-\mu_{c,i,j}|^2}.
\end{equation}
As shown in Fig.~\ref{fig:logvar}, the distribution of $\sigma_{c,i,j}$ exhibits a similar decaying trend. 
To quantitatively assess this relationship, we compute the Pearson correlation between the channel-wise BFS sensitivity curves and the corresponding clean-image frequency variance profiles. As summarized in Table~\ref{tab:pearson}, we observe a strong and statistically significant positive correlation across standard architectures such as ViT-B-32 and ResNet-50.
For adversarially trained models, this positive correlation remains evident in the luminance channel (Y) (e.g., $\bar{r} \approx 0.7$), while the chrominance channels exhibit more heterogeneous or mildly negative correlations.
Importantly, this does not alter the overall trend, since the sensitivity of chrominance channels is consistently much lower than that of the luminance channel across all models, as reflected by their substantially smaller CE responses in Fig.~\ref{fig:dctfreqsens}.

While targeting model-specific sensitivity peaks in BFS curves (BFS-peak) may ideally seem optimal, identifying the exact band is almost infeasible in black-box settings where the architecture is unknown. Moreover, empirical evidence (App.~\ref{app:freqprior}) reveals that compared to clean-image statistics, BFS-peak yields only marginal gains on matched models while suffering severe degradation on mismatched ones. Therefore, we leverage the intrinsic frequency properties of clean images to establish a practical, model-agnostic prior.

\begin{algorithm}[t!]
\small
\caption{\textbf{DPAttack} Process}
\label{alg:algorithmtotal}
\textbf{Input}: Target classifier $\mathcal{F}$, clean RGB image $\mathbf{x}$, GT label $y$, perturbation constraint $\epsilon$, candidate block size set $\mathbb{W}$.\\
\textbf{Output}: Adversarial example $\mathbf{x}^*$.
\begin{algorithmic}[1]
\STATE Obtain optimal block size $w$ and initialization direction $\mathbf{d}_\mathbf{n}$ via Dynamic Blocksize Selection:
\quad $(w, \mathbf{d}_\mathbf{n}, r) \gets \operatorname{DBS}(\mathcal{F}, \mathbf{x}, y, \mathbb{W})$ \hfill \COMMENT{\textit{refer to Alg.~\ref{alg:algorithmDBS}}}
\STATE Initialize low-pass filtered direction $\phi(\hat{\mathbf{d}}_\mathbf{r})$ with $h\leftarrow w$.
\STATE Perform initial direction selection and boundary search:
\STATE \quad $(\mathbf{d}_0, r_0) \gets \operatorname{BiLiSearch}(\mathcal{F}, \mathbf{d}_\mathbf{n}, \phi(\hat{\mathbf{d}}_\mathbf{r}), \mathbf{x}, y,r)$\hfill \COMMENT{\textit{refer to Alg.~\ref{alg:bilisearch}}}
\STATE Set current best parameters: $\mathbf{d}_{\text{best}} \gets \mathbf{d}_0$ and $r \gets r_0$.

\WHILE{Query Budget $> 0$ \textbf{and} $r > \epsilon$}
    \STATE Update $\mathbf{d}_{\text{best}}$ and $r$ via Pattern-Driven Optimization (Sec.~\ref{sec:pdo}).
\ENDWHILE

\RETURN $\mathbf{x}^* \gets \operatorname{clip}(\mathbf{x} + r \cdot \mathbf{d}_{\text{best}}, [0, 1])$
\end{algorithmic}
\end{algorithm}

\begin{figure}
\centering
  \includegraphics[width=\linewidth]{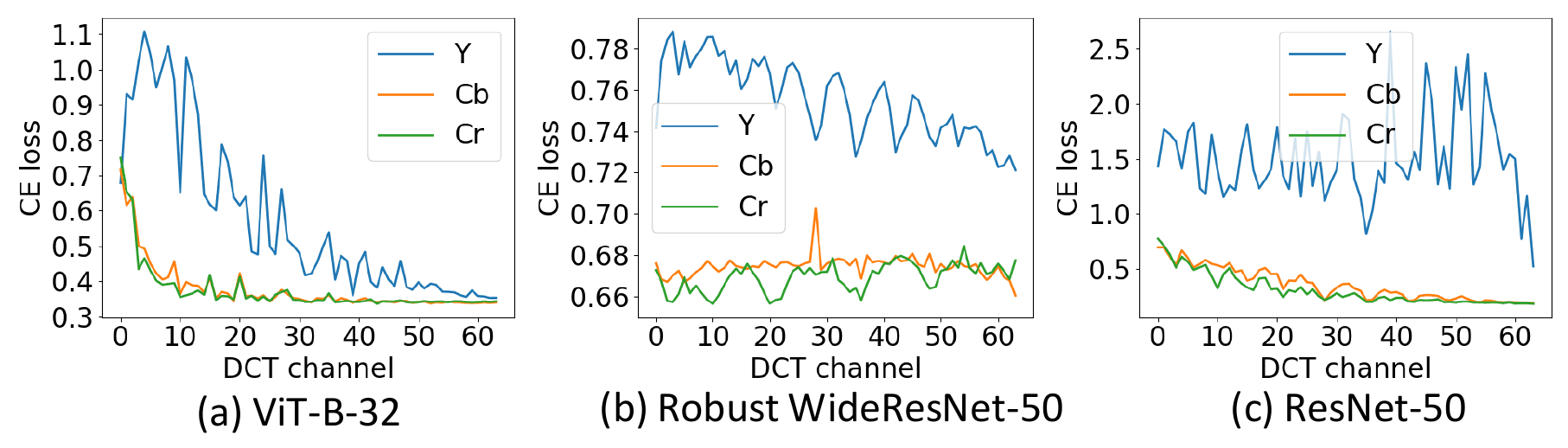}
  \caption{Classification sensitivity analysis of frequency bands via the proposed BFS. Higher CE loss indicates greater sensitivity of the corresponding BDCT frequency band.}
 
  \label{fig:dctfreqsens}
\end{figure}

\begin{figure}
\centering
  \includegraphics[width=\linewidth]{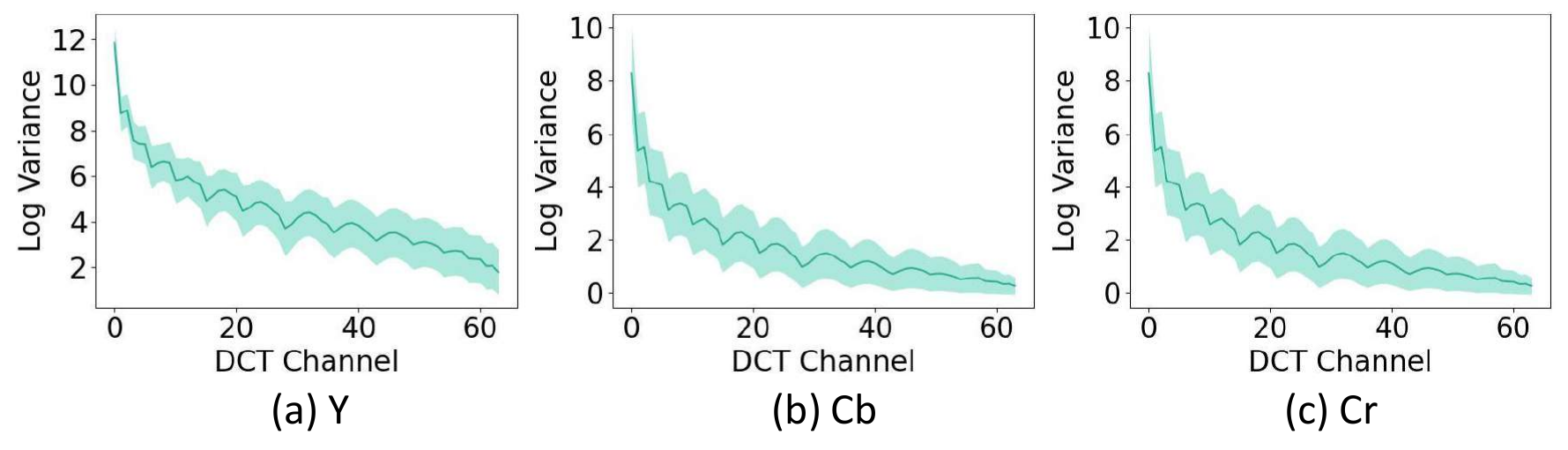}
  \caption{Clean image frequency statistics. The solid line and the shaded region represent the mean and standard deviation of the log-variance $\log(\sigma_{c,i,j}^2+1)$ (defined in Eq.~(\ref{eq:d1_1})).}
  \label{fig:logvar}
\end{figure}

\begin{table}[t!]
\footnotesize
\tabcolsep=0.15cm
\caption{Pearson correlation analysis ($\bar{r}^{p-\text{value}}$) between the proposed BFS profile and the clean image frequency variance across different models and color channels.}
\centering 
\begin{threeparttable}
\begin{tabular}{c|ccc}
    \Xhline{0.8pt}
     Channel & ViT-B-32   & WRS-50 & ResNet-50 \\
    \hline
     Y     & 0.8074$^{7.69E-16}$ & 0.7024$^{9.87E-11}$ & 0.0327$^{7.97E-01}$ \\
 Cb    & 0.9274$^{3.61E-28}$ & -0.1714$^{1.76E-01}$ & 0.9137$^{6.18E-26}$ \\
 Cr    & 0.8844$^{3.44E-22}$ & -0.4087$^{7.99E-04}$ & 0.9618$^{1.40E-36}$ \\
    \Xhline{0.8pt}
    \end{tabular}     
         \begin{tablenotes}   
        \footnotesize   
         \item[1] $\bar{r}$: the correlation coefficient; $p$-value: the statistical significance; WRS-50: Robust WideResNet-50. 
      \end{tablenotes}           
    \end{threeparttable}\label{tab:pearson}       

\end{table}

\subsubsection{Methodology: Constructing Initial Directions}\label{sec:dirs}

Based on the observation above, we propose using the clean image frequency variance $\sigma_{c,i,j}$ directly as a prior to construct an initialization direction $\mathbf{d_n}$. For comparison and diversity, we also design two alternative directions, $\mathbf{d_b}$ and $\mathbf{d_r}$, based on gradient local similarity~\cite{ilyas2019prior} and color sensitivity~\cite{andriushchenko2020square}.

\noindent{\textbf{1. Variance-Based Direction ($\mathbf{d_n}$).}} We construct the frequency perturbation matrix $\tilde{\Delta}$ by sampling noise coefficients $\mathbf{w}_{c,i,j}\sim\mathcal{N}(0,\sigma_{c,i,j}^2)$ for each frequency band, where $\sigma_{c,i,j}^2$ is computed solely from the clean image $\mathbf{x}$ at attack time. The spatial perturbation direction is then derived by:
\begin{equation}\label{eq:d1}
    \mathbf{x_n^\prime}= \operatorname{IBDCT}(\mathbf{I}+\tilde{\Delta},w), \enspace \mathbf{d_n} = \text{sgn}(\mathbf{x^\prime_n}-\mathbf{x}), \enspace \text{with } \text{sgn}(0):=1.
\end{equation}
Zero elements in $\mathbf{d_n}$ are replaced with $1$ to maintain a dense search space. Generating $\mathbf{d_n}$ requires no validation set, model query, or pre-training.

\noindent{\textbf{2. Gradient-Similarity-Based Direction ($\mathbf{d_b}$).}}
Inspired by HRayS, we verify that gradient spatial correlations persist even after sign quantization (Fig.~\ref{fig:gradcossim}(a)). To leverage this sign-level coherence, we design $\mathbf{d_b}$ as a block-wise structure consisting of $\hat{m}$ segments of length $\hat{n}$, where each segment is assigned a uniform sign:
\begin{equation}\label{eq:bardef}
\begin{gathered}
\mathbf{d_b}=\Big((-1)^{\lfloor(i-1)/\hat{n}\rfloor}\Big)_{i=1}^{\hat{m}}.
\end{gathered}
\end{equation}

\noindent{\textbf{3. Square-Color-Based Direction ($\mathbf{d_r}$).}} 
Leveraging the sensitivity of DNNs to square-shaped color perturbations~\cite{andriushchenko2020square}, we generate $\mathbf{d_r}$ as the sign of the residual between the clean image $\mathbf{x}$ and a randomly square-perturbed image $\mathbf{x}'_r$, where the square has side length $h$. That is:
\begin{equation}\label{eq:rcolor5}
\mathbf{d_r}=\text{sgn}(\mathbf{x}^\prime_r-\mathbf{x}),\quad \text{with } \text{sgn}(0):=1.
\end{equation}
Details for generating $\mathbf{x}^\prime_r$ are presented in App.~\ref{app:dr}.

\noindent{\textbf{Low-Frequency Focusing Method ($\phi$).}} 
To mitigate inter-band interference, we introduce a wrapper $\phi(\cdot)$ that confines perturbations to the low-frequency component $\mathbf{x}_{\text{LL}}$ via Discrete Wavelet Transform (DWT)~\cite{gonzalez2018digital}. We prioritize DWT over BDCT to minimize hyperparameters, reserving BDCT for generating pixel-domain block-wise artifacts.
We extract the low-pass component $\mathbf{x}_{\text{LL}}\in\mathbb{R}^{C\times\hat{H}\times\hat{W}}$ via:
\begin{equation}\label{eq:d2_1}
\mathbf{x}_{\text{LL}},\mathbf{x}_{\text{others}}=\operatorname{DWT}(\mathbf{x},dl),
\end{equation}
where $\hat{H}=H/2^{dl}$ and $\hat{W}=W/2^{dl}$. The decomposition level $dl$ is bounded by $\lfloor\log_2(\bar{w})\rfloor$, where $\bar{w}$ is the pattern size of the underlying strategy (i.e., $w$ for $\mathbf{d_n}$, $\hat{n}$ for $\mathbf{d_b}$, or $h$ for $\mathbf{d_r}$). We isolate low frequencies by setting the residuals $\mathbf{x}_{\text{others}}=\mathbf{0}$ and applying a perturbation $\hat{\mathbf{d}}_\mathbf{a}\in\{\hat{\mathbf{d}}_\mathbf{n},\hat{\mathbf{d}}_\mathbf{b},\hat{\mathbf{d}}_\mathbf{r}\}$ to $\mathbf{x}_{\text{LL}}$. Note that $\hat{\mathbf{d}}_\mathbf{a}$ is generated with pattern size scaled by $2^{-dl}$ (e.g., $w/2^{dl}$ for $\hat{\mathbf{d}}_\mathbf{n}$). The perturbed low-frequency component $\hat{\mathbf{x}}_{\text{LL}}=\mathbf{x}_{\text{LL}}+\hat{\mathbf{d}}_\mathbf{a}$ is then projected back to the pixel space via Inverse DWT:
\begin{equation}\label{eq:d2_end}
\mathbf{x}^\prime_\mathbf{a}=\text{IDWT}(\hat{\mathbf{x}}_{\text{LL}},\mathbf{0},dl).
\end{equation}
Encapsulating this process as $\phi(\hat{\mathbf{d}}_\mathbf{a})$, we derive the final direction $\mathbf{d_a}$ by taking the sign of the difference, ensuring non-zero entries:
\begin{equation}\label{eq:d2_end_}
\mathbf{d_a}=\text{sgn}(\mathbf{x}^\prime_\mathbf{a}-\mathbf{x}), \quad \text{with } \text{sgn}(0):=1.
\end{equation}

\begin{figure}
\centering
  \includegraphics[width=\linewidth]{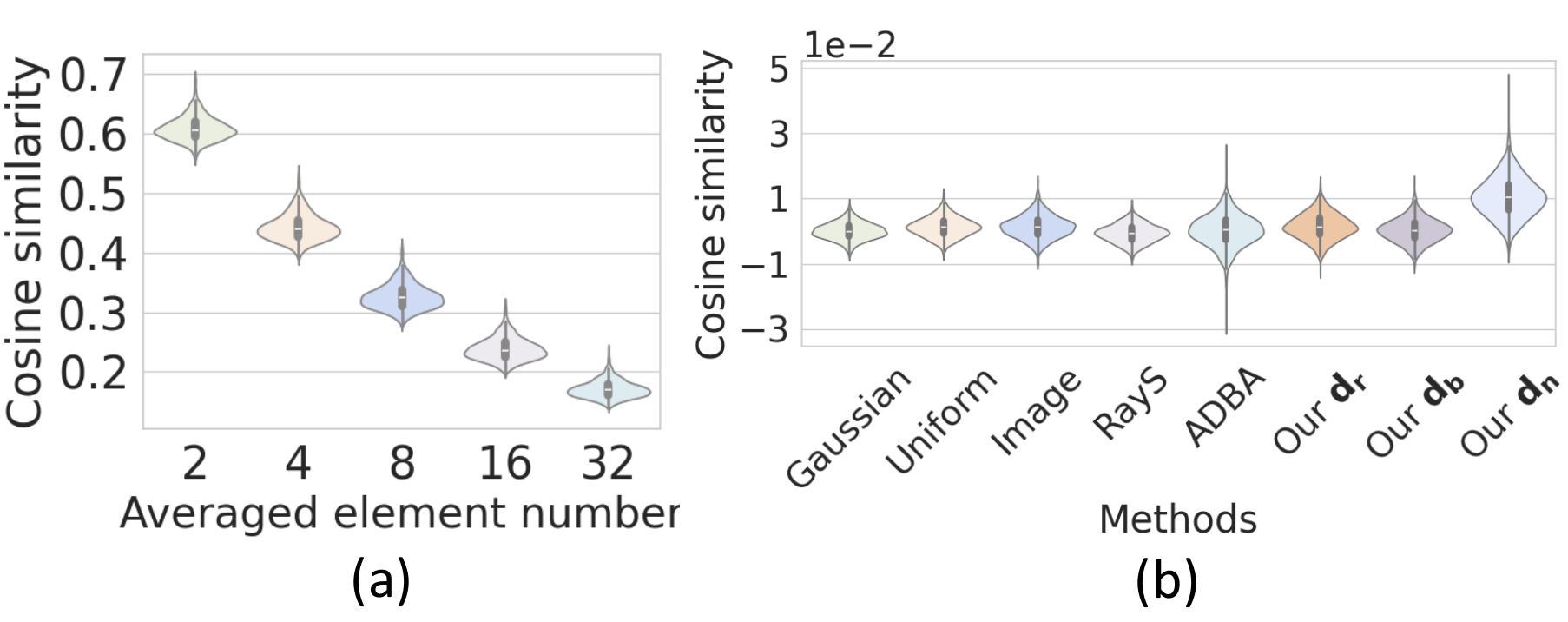}
  \caption{Cosine similarity between the true gradient sign and: (a) its block-wise averaged approximations; (b) various initialization directions.}
  \label{fig:gradcossim}
\end{figure}

\begin{table}[t!]
\tabcolsep=0.06cm
\small
  \centering
  \caption{Comparison of curvature ($\kappa$) and perturbation distances ($\ell_\infty$, $\ell_2$) across initialization methods. ``Image'' denotes a random sample from a different class.}
    \begin{tabular}{c|cc|cc|cc}
    \Xhline{0.8pt}
    Metrics & \multicolumn{2}{c|}{Curvature $\kappa$$\uparrow$}  & \multicolumn{2}{c|}{$\ell_\infty$$\downarrow$} & \multicolumn{2}{c}{$\ell_2$$\downarrow$} \\
    \hline
    Methods & Avg.   & Med.  & Avg.   & Med.& Avg.   & Med. \\
    \hline
    Gaussian &           58,736  &           49,404   &        0.363  &        0.341  &       49.4  &       46.7  \\
    Uniform &           58,273  &           48,932   &        0.315  &        0.318  &       49.7  &       48.6  \\
    Image &           57,873  &           47,068  &         0.362  &        0.364  &       55.2 &       54.2  \\
    RayS  &           58,167  &           45,454  &        0.518  &        0.533  &     147.8  &     150.1  \\
    ADBA  &           39,408  &           33,592  &        0.354  &        0.339  &     114.9  &     121.2  \\
    Our $\mathbf{d_r}$ &           58,308  &           47,349   &        \underline{0.091}  &        0.083  &       \underline{35.1}  &       32.1  \\
    Our $\mathbf{d_b}$ &     60,934 &    50,192   &        0.108  & \underline{     0.072 } &       39.3  & \underline{     27.2 } \\
  
      Our $\mathbf{d_n}$  &           \textbf{68,502}  &           \textbf{54,743}  &      \textbf{     0.078 } &        \textbf{0.069}  & \textbf{     28.9 } &  \textbf{26.1}  \\
  
    \Xhline{0.8pt}
    \end{tabular}%
  \label{tab:startpointmetrics}%
\end{table}%

\subsubsection{Validation and Dynamic Selection}

We benchmark our initialization directions against standard baselines (Gaussian, Uniform, RayS, ADBA, and random images) using ResNet-50. The resulting starting points are evaluated visually (App.~\ref{appsec:initres}) and quantitatively via three metrics: 1) Perturbation distance ($\ell_2/\ell_\infty$); 2) Local curvature $\kappa$ (indicating optimization speed);\footnote{Optimization behaviors and $\kappa$ estimation are detailed in Apps.~\ref{app:randkchange} and \ref{app:anainitd}.} and 3) Cosine similarity with the true gradient sign. As detailed in Table~\ref{tab:startpointmetrics} and Fig.~\ref{fig:gradcossim}, $\mathbf{d_n}$ consistently demonstrates superior geometric properties: it achieves the highest gradient sign alignment (Fig.~\ref{fig:gradcossim}(b)) and significantly smaller boundary distances with higher local curvature compared to all baselines. These advantages are critical for optimization; indeed, our ablation study (Sec.~\ref{sec:ablation}) confirms that replacing $\mathbf{d_n}$ with other initializations severely degrades the attack success rate. 


To theoretically justify this advantage, we model initialization as a sign-alignment problem. We show that, under a frequency-variance prior and an informative surrogate-gradient condition, sampling from a BDCT frequency distribution positively correlated with model sensitivity yields positive expected alignment with the true gradient sign up to a surrogate-transfer error $\varepsilon_{\mathrm{sur}}$ (as formalized in \textbf{Theorem~\ref{theo:ddm}}). This analysis is rooted in the Arcsine Law~\cite{1446497} (see App.~\ref{app:theo3_proof}), which connects positive value correlation to positive sign correlation for jointly Gaussian variables. In our setting, the frequency-variance prior biases $\mathbf{d_n}$ toward frequency bands that are, on average, more sensitive to the model; when this induced alignment dominates the transfer error ($\gamma>\varepsilon_{\mathrm{sur}}$), positive expected alignment follows. In contrast, isotropic random initialization distributes energy uniformly across frequencies and yields zero expected alignment in expectation. Under a local linear decision-boundary assumption, this positive alignment further implies a smaller expected boundary distance (Corollary~\ref{cor:boundary} in App.~\ref{subsec:theorem3}), explaining the empirical advantages of $\mathbf{d_n}$ in Table~\ref{tab:startpointmetrics}.

\noindent\textbf{Dynamic Block Size Selection.}
While our theoretical analysis supports the use of frequency-based
initialization, empirical results (Sec.~\ref{sec:ablation}) indicate that the
attack performance exhibits mild sensitivity to the choice of the BDCT block
size for a given victim model.
This sensitivity arises because the block size controls the granularity of the
frequency decomposition, which in turn affects both the correlation between the
initial direction and the true gradient sign, as well as the efficiency of the
subsequent optimization process.
To account for this variability without manual tuning, we introduce a
Dynamic Block Size Selection (DBS) strategy that query-efficiently identifies
a suitable block size $w$ (see App.~\ref{sec:DBS}).

\noindent{\textbf{BiLiSearch Directions.}} 
Even with optimized block sizes, relying on a single initialization is not universally optimal against unknown defenses (e.g., adversarial training). While $\mathbf{d_n}$ is generally superior, we observe that the color-based alternative $\phi(\hat{\mathbf{d}}_\mathbf{r})$ yields marginal gains on adversarially trained models (see Table~\ref{tab:ddmpdoabla}). To ensure robustness, we propose BiLiSearch (App.~\ref{app:algo}), a hybrid strategy combining binary and line search to efficiently select the direction $\mathbf{d}_0$ with the minimal boundary distance $r_0$:
\begin{equation}\label{eq:x0star}
    \mathbf{d}_0,r_0=\text{BiLiSearch}(\mathbf{d_n},\phi(\hat{\mathbf{d}}_\mathbf{r})), \enspace \mathbf{x}_0^*=\operatorname{clip}(\mathbf{x}+r_0\mathbf{d}_0,[0,1]).
\end{equation}

\begin{tcolorbox}[colback=blue!5!white,colframe=black!75!black
,
boxsep=0pt,
  left=4pt,
  right=4pt,
  top=4pt,
  bottom=4pt,
  boxrule=0.7pt,
  arc=2pt
]
\begin{theorem}[Positive Expected Alignment under Frequency-Variance Prior]
\label{theo:ddm}
Let $\mathbf{g}=\nabla_{\mathbf{x}}\mathcal{L}(\mathbf{x},y)$ and 
$\mathbf{u}=\mathrm{sgn}(\mathbf{g})\in\{-1,1\}^{d}$ denote the true gradient sign.
Let the initialization be $\mathbf{x}_0=U\mathbf{z}$, where
$U\in\mathbb{R}^{d\times d}$ is the orthonormal BDCT basis and
$\mathbf{z}$ is a zero-mean Gaussian vector with frequency-band variances
$\{v_q\}_{q=1}^{Q}$. The initialization direction is
$\mathbf{d_n}=\mathrm{sgn}(\mathbf{x}_0)$. Let $\mathbf{h}=U\tilde{\mathbf{h}}$ be a Gaussian surrogate gradient signal with
frequency-band variances $\{s_q\}_{q=1}^{Q}$. For
$\rho_i=\mathrm{Corr}((\mathbf{x}_0)_i,\mathbf{h}_i)$, let
$\gamma := d^{-1}\sum_{i=1}^{d}\frac{2}{\pi}\arcsin(\rho_i)$. If $\gamma>0$
and the surrogate-to-true-gradient transfer error is bounded by
$\varepsilon_{\mathrm{sur}}<\gamma$, then
\begin{equation}
\frac{1}{d}\mathbb{E}\!\left[\langle \mathbf{d_n},\mathbf{u}\rangle\right]
\ge
\gamma-\varepsilon_{\mathrm{sur}}
>0.
\end{equation}
In contrast, for isotropic random initialization
$\mathbf{d}_{\mathrm{rand}}$ that is coordinate-wise symmetric and independent
of $\mathbf{u}$, we have $d^{-1}\mathbb{E}\!\left[\langle
\mathbf{d}_{\mathrm{rand}},\mathbf{u}\rangle\right]=0$.

\end{theorem}
\end{tcolorbox}

\subsection{Pattern-Driven Optimization (PDO)}\label{sec:pdo}

Having established the initial direction $\mathbf{d}_0$, the next challenge is efficient direction refinement.
Unlike random noise, $\mathbf{d}_0$-predominantly sourced from the BDCT-based $\mathbf{d_n}$-naturally exhibits strong spatial coherence in the pixel space, forming contiguous blocks of identical signs (``runs''), as illustrated in Figs.~\ref{fig:overview}(a) and (c).
However, SOTA hard-label attacks such as RayS and ADBA adopt rigid, pattern-agnostic dyadic bisection, which blindly partitions the pixel space and fragments these coherent structures.
This mismatch disrupts the intrinsic geometry of $\mathbf{d}_0$ and leads to \emph{gain cancellation} (formally defined in Lemma~\ref{lem:gain} of App.~\ref{sec:app_theo4}), where conflicting gradient signals within a partition offset each other, resulting in query inefficiency.

To address this issue, we propose the \emph{Pattern-Driven Optimization (PDO)} module, which explicitly aligns the search process with the spatial patterns of $\mathbf{d}_0$.
Instead of operating in the full pixel space of dimension $d$, PDO treats each sign-consistent run as an atomic search unit, effectively reducing the optimization problem to a lower-dimensional \emph{pattern space} with dimensionality equal to the number of runs $M \ll d$.
For example, on ImageNet ($d \approx 150\text{k}$), using block sizes $w \in \{4,8,16\}$ reduces $M$ to approximately $20\text{k}$, $11\text{k}$, and $6\text{k}$, respectively.
By preserving spatial coherence and avoiding destructive partitioning, PDO enables more efficient gradient sign approximation that scales with the compressed pattern structure rather than the raw pixel dimensionality.

\subsubsection{Algorithm Implementation}\label{sec:pdoimp}
To operationalize this, PDO constructs a search tree $\mathsf{T}_{\mathrm{pat}}$ aligned with the pattern of $\mathbf{d}_0$, via three phases:

\noindent\textbf{Phase 1: Pattern-Driven Tree Construction.}
\begin{itemize}
    \item Sorting and indexing: Define a permutation $\Gamma$ sorting $\mathbf{d}_0$ such that $\mathbf{d}_0[\Gamma(1)] \leq \dots \leq \mathbf{d}_0[\Gamma(d)]$.
     \item Atomic partitioning: Partition $\Gamma$ into atomic units $\{B_1, \dots, B_M\}$ grouping contiguous indices, effectively capturing the ``runs'' in $\mathbf{d}_0$.
     \item Hierarchical grouping: Build a binary hierarchy where, at level $s \in [1, \lceil \log_2 M \rceil]$, atomic units are merged into $2^s$ groups $\{\mathcal{G}_j\}_{j=1}^{2^s}$ for coarse-to-fine search.
\end{itemize}
    

\noindent\textbf{Phase 2: Sign Flipping and Direction Update.}
We refine $\mathbf{d}_{\text{best}}$ (initialized as $\mathbf{d}_0$) by iteratively flipping signs within adjacent group pairs $(\mathcal{G}_{2l-1}, \mathcal{G}_{2l})$ to generate candidates $\bar{\mathbf{d}}_1$ and $\bar{\mathbf{d}}_2$. To maximize efficiency, we employ a history-aware update rule:
\begin{small}
\begin{equation}\label{eq:Decision2_ours}
\mathbf{d}_{\text{best}}=\begin{cases}
 \mathbf{d}_{\text{best}} & \text{ if } \mathcal{F}(\mathbf{x}+r\bar{\mathbf{d}}_1 )=y\ \&  \ \mathcal{F}(\mathbf{x}+r\bar{\mathbf{d}}_2 )=y \\
 \Lambda (\bar{\mathbf{d}}_1,\bar{\mathbf{d}}_{\text{his}}) & \text{ if } \mathcal{F}(\mathbf{x}+r\bar{\mathbf{d}}_1 )\neq y\ \&  \ \mathcal{F}(\mathbf{x}+r\bar{\mathbf{d}}_2 )=y   \\
 \Lambda (\bar{\mathbf{d}}_2,\bar{\mathbf{d}}_{\text{his}}) & \text{ if } \mathcal{F}(\mathbf{x}+r\bar{\mathbf{d}}_1 )= y\ \&  \ \mathcal{F}(\mathbf{x}+r\bar{\mathbf{d}}_2 )\neq y  \\
 \Lambda (\bar{\mathbf{d}}_1,\bar{\mathbf{d}}_2) & \text{Otherwise.}
\end{cases}
\end{equation}
\end{small}%
Here, $\Lambda$ selects the direction with the smaller boundary distance. Notably, we use a buffer $\bar{\mathbf{d}}_{\text{his}}$ to cache solitary successful directions, preserving promising candidates that baseline methods (e.g., ADBA) would discard.

\noindent\textbf{Phase 3: Refinement and Loop.}
We increment the granularity level $s$ to progressively refine the search.\footnote{A fallback that splits atomic units or re-initializes $\mathbf{d}_0$ is only triggered in rare low-dimensional cases (e.g., CIFAR-10) and was not observed on higher-resolution datasets such as ImageNet.} The complete workflow is visualized in Fig.~\ref{fig:overview}(b).

\subsubsection{Theoretical Guarantees}
Having defined the algorithm, we now utilize a signal recovery framework to justify why aligning the search tree with initialization patterns ($\mathsf{T}_{\mathrm{pat}}$) outperforms standard dyadic trees ($\mathsf{T}_{\mathrm{dyad}}$) employed by HRayS and ADBA.

\begin{tcolorbox}[colback=blue!5!white,colframe=black!75!black
,
boxsep=0pt,
  left=4pt,
  right=4pt,
  top=4pt,
  bottom=4pt,
  boxrule=0.7pt,
  arc=2pt
]
\begin{theorem}[Per-query Sign Alignment Dominance]\label{thm:main_text}
Let $\mathbf{u} = \operatorname{sgn}(\nabla_{\mathbf{x}}\mathcal{L})$ be the true gradient sign, and let
$\mathbf{d}_t\in\{\pm1\}^d$ be the perturbation direction after $t$ queries.
We measure sign alignment using the agreement
$A(\mathbf{d}_t,\mathbf{u})
:=\frac{1}{d}\sum_{i=1}^d \mathbb{I}\{\mathbf{d}_{t,i}=\mathbf{u}_i\}
$. Under the assumptions of local block coherence and a weakly positive initial correlation (Asmps.~\ref{asmp:block} and \ref{asmp:purity} in App.~\ref{sec:app_theo4}), the expected agreement $A_t$ after $t$ queries satisfies:
\begin{equation}
\mathbb{E}\big[A_t^{\mathrm{pat}}\big] \ge \mathbb{E}\big[A_t^{\mathrm{dyad}}\big], \quad \forall t \in \mathbb{N},
\end{equation}
where $\mathrm{pat}$ and $\mathrm{dyad}$ denote the pattern-driven strategy and the blind dyadic baseline, respectively.
\end{theorem}
\end{tcolorbox}
\noindent Proof in App. \ref{sec:app_theo4}. The inequality degenerates to an equality only if $\mathbf{d}_0$ lacks spatial variation (e.g., $\mathbf{d}_0=\mathbf{1}$) or aligns perfectly with the dyadic grid. We provide empirical verification of this dominance in Sec.~\ref{sec:pdores}. This per-step advantage scales into a significant improvement in asymptotic query complexity, as formalized in \textbf{Theorem~\ref{theo:complexity-fixed}} (see App.~\ref{sec:app_theo5} for the full proof). Throughout Theorem~\ref{theo:complexity-fixed}, we use the number of tree expansions as a proxy for query complexity. This is justified because, in RayS-style hard-label attacks (including our PDO framework),
each node expansion incurs at most a constant number of model queries.
Therefore, the asymptotic bounds in Theorem~\ref{theo:complexity-fixed} translate directly to
query complexity up to constant factors. Also, the complexity $T$ quantifies the cost for \textit{full} sign recovery. In practice, attacks succeed upon achieving \textit{sufficient alignment} (partial recovery); thus, actual consumption is usually lower than $T$. We attribute this complexity gap to a structural phenomenon we term \textit{Intrinsic Gradient Shattering}, a detailed empirical analysis of which is provided in the following validation.

\begin{tcolorbox}[colback=blue!5!white,colframe=black!75!black
,
boxsep=0pt,
  left=4pt,
  right=4pt,
  top=4pt,
  bottom=4pt,
  boxrule=0.7pt,
  arc=2pt
]
\begin{theorem}[Query Complexity under Block Sign-Coherence]\label{theo:complexity-fixed}
Under Asmps.~\ref{asmp:block} and~\ref{asmp:purity}, let the true gradient sign $\mathbf{u}$ consist of $K$ spatially coherent blocks $\{B_k\}_{k=1}^K$. To identify descent directions aligned with $\mathbf{u}$, the expected query complexities for \emph{dyadic} search ($T_{\mathrm{dyad}}$) and \emph{pattern-driven} search ($T_{\mathrm{pat}}$) satisfy:
\begin{equation}
    T_{\mathrm{dyad}} = \Omega\left(\sum_{k=1}^K \log_2\frac{d}{|B_k|}\right), \quad T_{\mathrm{pat}} = O\left(\sum_{k=1}^K \gamma_k\right),
\end{equation}
where $\Omega(\cdot)$ and $O(\cdot)$ denote asymptotic lower and upper bounds. $\gamma_k$ is the number of $\mathbf{d}_0$-runs intersecting block $B_k$. In particular, the pattern-driven strategy is asymptotically more efficient in expectation when $\gamma_k < \log_2(d/|B_k|)$ on average.
\end{theorem}
\end{tcolorbox}

\begin{figure}
\centering
  \includegraphics[width=\linewidth]{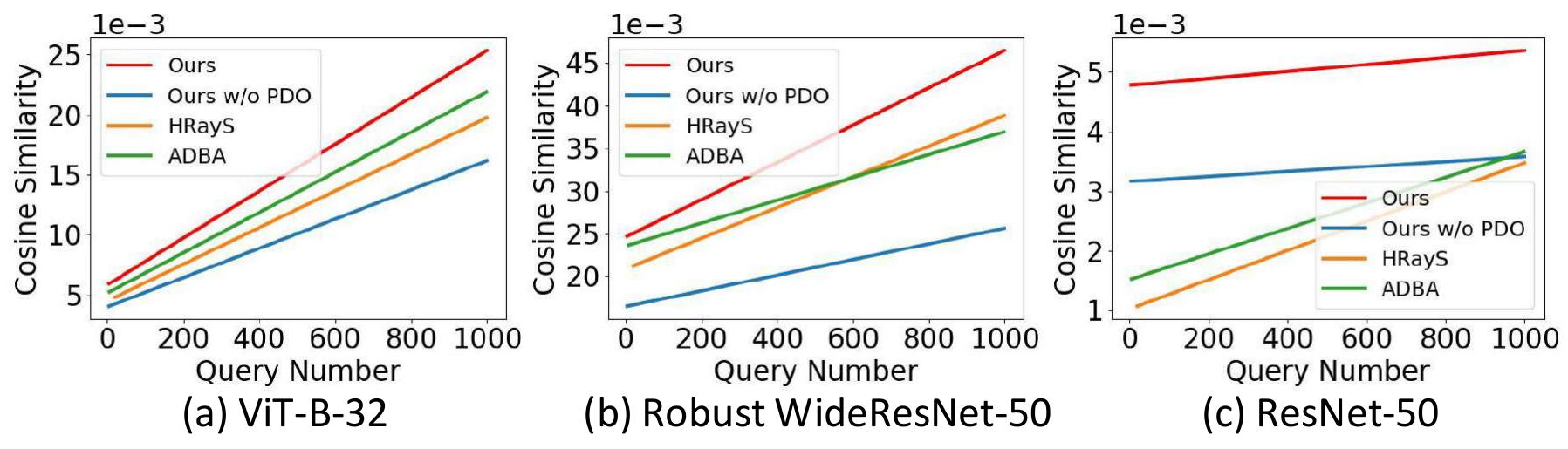}
  \caption{Evolution of cosine similarity between the true gradient sign and the perturbation direction.}\label{fig:gradsimoverlquery}
\end{figure}

\begin{figure}
\centering
  \includegraphics[width=\linewidth]{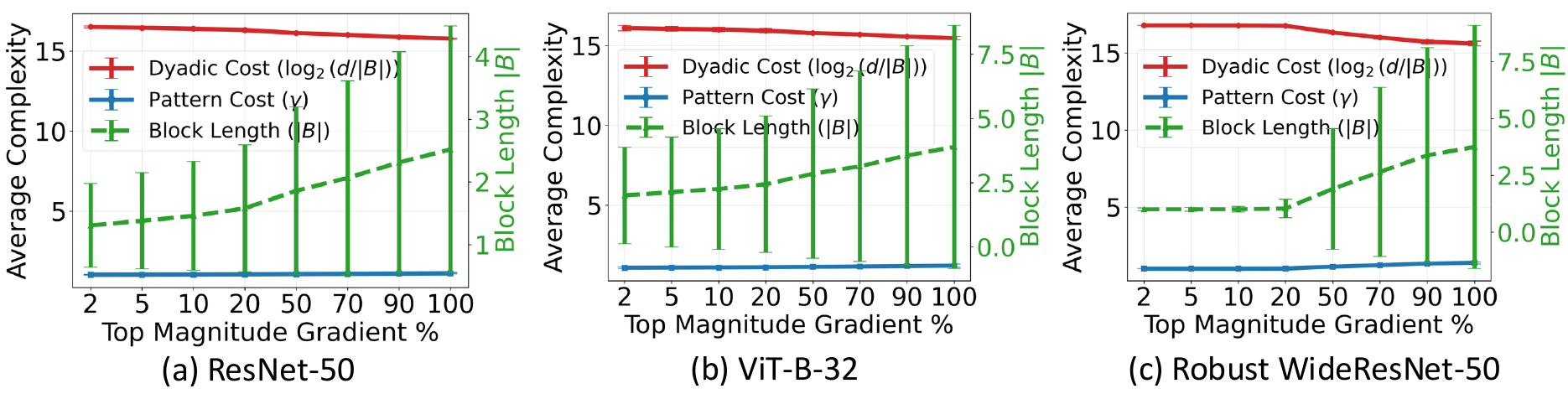}
  \caption{Empirical validation of Theorem~\ref{theo:complexity-fixed}. Results for additional models are provided in App. Fig.~\ref{fig:gradsignstru_others}.} 
  \label{fig:gradsignstru}
\end{figure}

\subsubsection{Empirical Validation}\label{sec:pdores}
To empirically verify Theorems~\ref{thm:main_text} and~\ref{theo:complexity-fixed}, we track the cosine similarity between the search direction $\mathbf{d}_{\text{best}}$ and the true gradient sign $\text{sgn}(\nabla\mathcal{L}(\mathbf{x},y))$ (averaged over 1,000 ImageNet images). The results are visualized in Fig.~\ref{fig:gradsimoverlquery}. Note that to improve clarity, curves undergo a linearly fitted process, which causes a minor visual divergence at $t=0$ between ``Ours'' and ``Ours w/o PDO'' (using standard dyadic split), even though they originate from the exact same initial direction. The observed trends strongly confirm our theoretical claims:

    \noindent\textbf{1)} The ``Ours'' curve consistently upper-bounds the baseline ``Ours w/o PDO'' across architectures. This empirical gap confirms Theorem~\ref{thm:main_text}, demonstrating that incorporating initial pattern priors generally yields a better descent direction for any given query budget $t$.

   \noindent\textbf{2)} The substantial gap between the slow ascent of ``Ours w/o PDO'' and the steep rise of ``Ours'' also confirms the complexity divergence predicted by Theorem~\ref{theo:complexity-fixed}. We attribute this to \textit{Intrinsic Gradient Shattering}—a phenomenon where top-magnitude gradient components, critical for optimization, exhibit rapid spatial oscillation. As visualized by the green lines in Fig.~\ref{fig:gradsignstru}, the average sign-consistent block size $|B|$ for these top-$2\%$ magnitude gradients collapses to $\approx1$ in CNNs and $\approx2$ in ViTs. This inherent fragmentation triggers two distinct regimes: (i) \textbf{Dyadic Collapse:} The rigid partitioning of dyadic strategies is forced into a worst-case scenario by these pixel-level oscillations. This drives the search depth toward the theoretical limit ($\approx \log_2 d$), incurring a query complexity of $\approx 16 \sim 17$ per block—matching the value of dyadic costs in Fig.~\ref{fig:gradsignstru}. (ii) \textbf{Pattern Resilience:} In contrast, PDO sidesteps this shattering by adaptively aligning with the run-length boundaries of $\mathbf{d}_0$. By treating these oscillations as coherent atomic units rather than disjoint pixels, PDO maintains a near-optimal fragmentation factor $\gamma\approx1$ (represented by the blue lines in Fig.~\ref{fig:gradsignstru}), significantly accelerating convergence. 
 
    \noindent\textbf{3)} 
    In Fig.~\ref{fig:gradsimoverlquery}, our method consistently outperforms HRayS and ADBA. Additionally, the increasing HRayS trend validates Theorem~\ref{theo:raysappro}, confirming that while less efficient than our method, the standard dyadic hierarchy still guarantees a progressive reduction in approximation error.

\begin{table*}[t!]
\footnotesize
\tabcolsep=0.03cm
  \centering
  \caption{The untargeted attack performance on ImageNet (ResNet-50) under $\ell_\infty=0.05$ constraint.}
  \scalebox{0.8}{
    \begin{tabular}{c|cccccccc|cccccccc|cccccccc}
    \Xhline{0.8pt}
    Metrics & \multicolumn{8}{c|}{ASR}      & \multicolumn{8}{c|}{Avg.Q} & \multicolumn{8}{c}{Med.Q} \\
    \hline
    Max$Q$  & HSJA  & BouR &  BouT &TtBA&HRayS  & ADBA  & \textbf{Ours$_\text{opt}$}& \textbf{Ours$_\text{dyn}$}  & HSJA  & BouR & BouT&TtBA & HRayS  & ADBA  & \textbf{Ours$_\text{opt}$}& \textbf{Ours$_\text{dyn}$}  & HSJA & BouR  & BouT&TtBA & HRayS  & ADBA  & \textbf{Ours$_\text{opt}$}& \textbf{Ours$_\text{dyn}$} \\
    \hline

    5     &      0.0   &      0.0   &        \underline{3.5}  &   0.0      &   2.0  &   0.0    & \textbf{20.5} &0.0&   0.0    &   -    &              5  &-&              5  &  -     &              5  &  -&  -   &    -   &              5  &  -  &          5  & -      &              5&-  \\
    
    10    &          0.0    &    0.0   &    \underline{9.5}  &   0.0      &   2.0  &      0.0   & \textbf{37.0} &-&   -    &   -    &              6  &          &  10  &    -   &              6  &-&  -     &    -   &              6  & -   &        10  &   -    &              5  &-\\
    
    20    &           6.5  &           6.5  &           {9.5}  &    0.0    &    2.0  &           3.5  & \underline{37.0}&\textbf{39.5} &            14  &            14  &              6  &   -    &     20  &            17  &            6&19  &            14  &            14  &              6  &  -     &     20  &            17  &              5&19  \\
    
    50    &           6.0  &           6.0  &         17.0  &     0.0    & 16.5  &  {22.5}  & \underline{51.5} &\textbf{55.0}&            13  &            13  &            20  &    -    &    42  &            32  &            19 &23 &            13  &            13  &              6  &  -     &     40  &            31  &            13&20  \\
    
    80    &         14.0  &         14.0  &         17.5  &   6.5   &    21.0  & {38.0}  & \underline{57.0}&\textbf{64.0} &            40  &            40  &           21  &    75    &    52  &            46  &            24  &29&            59  &            59  &              6  &     75     &  56  &            42  &            15&20  \\
    
    100   &         15.0  &         16.0  &         16.0  & 7.5   &    25.0  & {49.5}  & \underline{62.0} &\textbf{66.0}&            40  &            37  &            23  &   77   &      60  &            56  &            29  &31&            59  &            38  &              6  &   75    &     57  &            59  &            15 &20 \\
    
    200   &         17.5  &         17.5  &         22.0  &   10.5   &   51.0  & {71.0}  & \textbf{75.5} & \underline{75.0} &            75  &            64  &            72  &   87  &     109  &            87  &            49 &46 &            99  &            99  &            48  &    75    &  104  &            80  &            26  &22\\
    
    300   &         20.0  &         18.5  &         21.0  &    14.0  &   54.0  & \underline{79.5}  & \textbf{83.5}& \textbf{83.5} &            77  &            71  &            54  &  126   &     136  &          104  &            67 &66 &            99  &            99  &              6  &    82  &    121  &            84  &            34&22  \\
    
    400   &         16.5  &         16.0  &         23.5  &   16.0   &   69.0  &  {85.0}  & \underline{88.5}&\textbf{89.0} &            61  &            60  &            72  &   158   &    160  &          117  &            83  & 85&           99  &            99  &              7  &  144      &  143  &            85  &            38&24  \\
    
    500   &         20.5  &         18.0  &         25.5  &    19.0  &   70.5  & {87.5}  & \underline{90.0}& \textbf{92.0} &            92  &            81  &            85  &  184  &      198  &          128  &            88&97  &            99  &            99  &            88  &    146   &   174  &            88  &            38 &29 \\
    
    1k  &         20.5  &         18.5  &         22.0  &    32.5    & 89.5  &  {93.5}  & \underline{95.5}& \textbf{97.5} &          164  &          110  &          184  &    414   &   312  &          163  &          127&          131  &          150  &          149  &              8  &    404    &  226  &            97  &            44 &36 \\
    
    3k  &         23.0  &         19.5  &         23.0  &    59.5  &   97.5  & {99.0}  & \underline{99.5} & \textbf{100.0} &          352  &          209  &          196  &   1,117   &    431  &          236  &          199&          180  &          150  &          149  &          139  &    965  &    252  &          100  &            48  &            39  \\
    
    5k  &         21.0  &         19.0  &         22.5 &67.5 & \textbf{ 100.0} & \textbf{ 100.0} & \textbf{ 100.0} & \textbf{ 100.0} &          210  &          140  &          319  &   1,441 &      522  &          249  &          222 &180 &          149  &          149  &          138  &    978   &   263  &          102  &            49 &39 \\
    \Xhline{0.8pt}
    \end{tabular}%
    }
  \label{tab:stres50}%
\end{table*}%

\section{Evaluation}\label{sec:exp}

\subsection{Experimental Setup}

\noindent{\textbf{Datasets and Victim Models.}} We evaluate our method across a diverse range of tasks. For standard classification on {CIFAR-10}~\cite{krizhevsky2009learning} and {ImageNet-1K}~\cite{deng2009imagenet}, we target representative architectures (e.g., ResNet-18/50, ViT-B-32, DenseNet-121) alongside adversarially trained WideResNet variants~\cite{croce2021robustbench,robustness}. We also assess the {ViT-H-14 CLIP}~\cite{ilharco_gabriel_2021_5143773} foundation model and test robustness against context shifts on {ObjectNet}~\cite{barbu2019objectnet}. For dense prediction, we employ {MS-COCO}~\cite{lin2014microsoft} (attacking FasterRCNN~\cite{Torchvision}) and {SA-1B}~\cite{kirillov2023segment} (attacking SAM~\cite{kirillov2023segment}). Furthermore, we include {ImageNet-C}~\cite{hendrycks2019benchmarking} and {PathMNIST}~\cite{yang2023medmnist} to evaluate performance on corrupted and non-natural (medical) domains using their respective top-performing models. See App.~\ref{app:expsetting} for the full list of over 15 evaluated architectures and image sizes for each dataset.

\noindent{\textbf{Metrics.}} We evaluate attack effectiveness using Attack Success Rate (ASR) within a query budget (Max.Q). We compute ASR$=\frac{1}{N}\sum_{i=0}^{N-1}\mathbb{I}(\mathcal{F}(\mathbf{x}^*_i)\neq y_i)$ and set $N=200$ following \cite{wan2024bounceattack}. To assess efficiency, we further report the average (Avg.Q) and median (Med.Q) queries among successful attacks.

  \noindent\textbf{Comparative Methods.} We evaluate two variants of our method: (1) {Ours$_{\text{opt}}$}, which uses the fixed block size that achieves the highest average ASR across query budgets, and (2) {Ours$_{\text{dyn}}$}, which employs our dynamic block size selection strategy. We compare them against SOTA query-based hard-label attacks: {TtBA}~\cite{wangttba}, {HRayS}~\cite{10.1145/3394486.3403225}, {ADBA}~\cite{Wang_Zuo_Huang_Chen_2025}, {HSJA}~\cite{chen2020hopskipjumpattack}, and {Bounce}~\cite{wan2024bounceattack} with both random ({BouR}) and targeted ({BouT}) initializations. For $\ell_2$ comparisons, we include {SurFree}~\cite{maho2021surfree}, {Tangent}~\cite{ma2021finding}, and CGBA(-H)~\cite{reza2023cgba}. Implementation details, including the setting of decomposition level $dl$ and hyperparameter configurations, are provided in App.~\ref{app:expsetting}.

\subsection{Performance on Standard Classifiers}
\label{sec:expstdmodel}

We evaluate the efficacy of DPAttack against standard classifiers trained on CIFAR-10, ImageNet-1K, ImageNet-C and PathMNIST. Primary results are summarized in Tables~\ref{tab:stres50}, \ref{tab:cifarvgg}-\ref{tab:otherdataset}, where extensive evaluations across diverse architectures and domains are detailed. \textbf{ASR.} 
Across nearly all scenarios, our method consistently outperforms SOTA baselines. The gap is particularly pronounced under tight query budgets. For instance, on ResNet-50 with Max.Q$=50$, our$_{\text{dyn}}$ surpasses the SOTA ADBA by $32.5\%$ in ASR while consuming only approximately two-thirds of its Avg.Q and Med.Q. Note that for extremely low budgets (Max.Q$<20$), the ASR of Ours$_{\text{dyn}}$ drops due to the DBS warm-up overhead, which costs about 15/25 queries per image on standard/robust models; this can be mitigated by tuning the block size on a few images first and then using the most frequently selected block size for subsequent attacks, although per-image DBS remains preferable under larger budgets. \textbf{Query Efficiency.} 
DPAttack achieves competitive or lower query consumption compared
to prior works at most cases—significantly outperforming the runner-up ADBA—while maintaining substantially higher ASR. Although our Avg.Q on ImageNet occasionally exceeds that of BouT, this stems from our significantly higher ASR: we successfully crack ``hard'' samples that require more queries, whereas baseline methods simply fail on them (yielding no query count contribution). \textbf{Additional Results.} 
Comparisons under varying perturbation magnitudes and $\ell_2$ constraints are provided in Appendix~\ref{app:differentPerturbation500} and Appendix~\ref{sec:l2app}, respectively. These results consistently demonstrate our method's superior ASR and query efficiency over SOTA baselines; notably, Ours$_\text{dyn}$ excels in $\ell_2$ scenarios, often significantly surpassing the Ours$_\text{opt}$ version.

\begin{table*}[t!]
\footnotesize
\tabcolsep=0.03cm
  \centering
  \caption{The untargeted attack performance on ObjectNet (ViT-H-14-CLIPA-336) under $\ell_\infty=0.05$ constraint.} 
  \scalebox{0.85}{
  \begin{tabular}{c|cccccccc|cccccccc|cccccccc}
    \Xhline{0.8pt}
    Metrics & \multicolumn{8}{c|}{ASR}                      & \multicolumn{8}{c|}{Avg.Q}                    & \multicolumn{8}{c}{Med.Q} \\
   \hline
    Max$Q$  & HSJA  & BouR  & BouT & TtBA  & HRayS  & ADBA  & Ours$_\text{opt}$& Ours$_\text{dyn}$  & HSJA  & BouR  & BouT & TtBA & HRayS  & ADBA  & Ours$_\text{opt}$& Ours$_\text{dyn}$  & HSJA  & BouR  & BouT & TtBA & HRayS  & ADBA  & Ours$_\text{opt}$& Ours$_\text{dyn}$ \\
   \hline
    10    &            0.0    &   0.0    &            \underline{10.0} & 0.0 &   0.0     &  0.0      & \textbf{21.5} &0.0&   -    &  -     &               6   &-& -     &     -  &               5  &&    -    &  -     &              5  &- &   -   &     -  &              5 & \\
   
    50    &           5.5  &  5.5     &           {15.0} &0.0 &           5.0  &            7.5  & \underline{37.0}& \textbf{39.5}&             13  &   13    &             20  &    -   &      41  &                              41  &             16 & 22&            12  &  12     &              6  &   -   &      39  &            42  &            10 &20 \\
   
    100   &         20.5  &                18.0  &            {22.5}  & 9.5 &        13.0  &          20.5  & \underline{44.5} &\textbf{50.0}&             40  &            42  &             35  & 77&            72  &                              61  &             25 &33 &            61  &            58  &              6  & 76  &         81  &            58  &            15 &20 \\
   
    500   &         23.5  &   23.5    & 23.0      & 22.0    &     38.0  &          {49.5}  & \underline{65.5}& \textbf{67.5} &             95  &  91     &  72   & 207  &           178  &                            156  &             98 &92 &          102  &  98     & 88      & 190   &      157  &          118  &            34  &28\\
   
    1k  &         25.0  &    23.5   & 25.5      & 32.5  &     51.0  &          {61.0}  & \underline{77.0} &\textbf{78.5}&           136  &  129     &  137  & 363 &           326  &                            257  &           194 & 186&          151  &   149    &  6     &  313  &      220  &          162  &            55 & 49\\
   
    3k  &         26.5  &   22.5    &  26.5     &  58.5     &  73.0  &          \underline{82.5}  & \textbf{88.5}&\textbf{88.5} &           242  &   141    &   202    & 1,105 &         770  &                            636  &           371& 364 &          152  &   149    &  137     &  850   &     430  &          290  &            86 & 65\\
   
    6k  &         27.0  &    25.0   &            29.5  &   73.0 &     82.0  &          \underline{88.0}  & \textbf{92.5}&\textbf{92.5} &           637  &   130    &           197  &  1,692  &    1,171  &                            872  &           536 &523 &          152  &  149     &              8  &   1,356 &      565  &          313  &            98  &82\\
   
    15k &         33.0  &   24.5    &  27.5     &  83.0   &    91.5  &          {95.5}  & \textbf{98.5}&\underline{98.0} &        1,123  & 412      &   137    & 2,437 &      2,135  &                         1,527  &        1,039&1,022  &          154  &   148    & 7      &   1,650   &    853  &          362  &          124 &92 \\
   
    30k &         37.5  &   22.5    &  28.5     &  89.5   &    {96.5}  & \underline{99.0} & \underline{99.0}&\textbf{99.5} &        3,880  &    107   &  203   & 3,498 &        3,024  &                         2,244  &        1,366&1,399  &          155  &  148     &  138     & 1,952  &       940  &          509  &          171 &99 \\

     40k &   40.0        &24.5   & 30.0 & 89.5   & 95.5  & 99.5 &99.5 & \textbf{100.0}&        6,552  &    114   &  225   & 4,460 &        1,211  &                         2,426  &        1,367&1,555  &          153  &  149     &  8     & 2,125  &       651  &          518  &          141 &103 \\
    \Xhline{0.8pt}
    \end{tabular}%
    }
  \label{tab:attclip}%
\end{table*}%

\subsection{Performance on CLIP and APIs}
\label{sec:expapis}
\noindent{\textbf{CLIP}.} CLIP model is a versatile and prevailing foundation vision-language model for many downstream tasks and is suitable for zero-shot image classification, which achieves $77.4\%$ Top-1 classification accuracy on ObjectNet. We evaluate our method against the CLIP model on ObjectNet to test its zero-shot robustness. To strictly simulate a hard-label black-box setting, we assume the adversary can only access the index of the predicted category (the class with the highest cosine similarity), without access to the specific similarity values or gradients. As shown in Table~\ref{tab:attclip}, our method demonstrates superior ASR and query efficiency compared to SOTA baselines. Notably, under a tight query budget (Max.Q$=50$), Ours$_\text{dyn}$ outperforms ADBA by $32\%$ in ASR while halving the queries. Even with a generous budget ($>3,000$), while the ASR gap narrows ($<10\%$), our method remains significantly more efficient, reducing Med.Q by approximately $80\%$ and Avg.Q by $40\%$ compared to ADBA.

\noindent{\textbf{APIs}.} We further validate our approach on four popular commercial image labeling services: Imagga~\cite{Imagga}, Google Cloud Vision~\cite{GoogleVision}, Tencent DetectLabel Pro~\cite{Tencent}, and Baidu General Image Recognition~\cite{Baidu}. The test images are randomly sampled from the ImageNet validation set. Due to the high cost of commercial API calls, we compare exclusively against ADBA, the strongest baseline identified in previous experiments. The results, summarized in Table~\ref{tab:api}, indicate that our algorithm consistently outperforms ADBA across all APIs and perturbation limits. For instance, on the Tencent API with $\ell=0.1$, Ours$_\text{opt}$ achieves an ASR improvement of over $44\%$ while consuming significantly fewer queries. Our method drastically reduces the query complexity from the tens or hundreds required by ADBA to merely $10-20$. This significant reduction in attack cost exposes the severe vulnerability of online APIs to low-budget adversaries, highlighting the urgent need for more robust defenses.

\subsection{Performance on Defensive Methods}
\label{sec:expdef}

We evaluate our attack against both static (adversarial training) and dynamic stateful defense mechanisms (Blacklight).

\noindent\textbf{Adversarial Training.} 
We test attack effectiveness on robust WideResNet models (CIFAR-10 and ImageNet). As summarized in Table~\ref{tab:defense}, our method achieves the highest ASR with superior query efficiency. Notably, we outperform HRayS by over $6\%$ ASR on ImageNet. In contrast, ADBA shows limited improvement ($0.5\%$ on ImageNet) and even degradation ($-0.5\%$ on CIFAR-10) versus HRayS. Additional results under varying constraints are detailed in App.~\ref{app:wrs50}.

\noindent\textbf{Stateful Detection.} 
We extend our evaluation to Blacklight~\cite{li2022blacklight}, a SOTA defense that detects attacks by identifying highly correlated queries via hashing-based fingerprinting. Since Blacklight flags attacks when hash collisions within a sliding window exceed a threshold, we propose a randomized variant of DPAttack to break this correlation. 
Specifically, we inject Gaussian noise $\xi$ into the update direction: $\mathbf{x} + r(\text{clip}(\bar{\mathbf{d}}_1 + \xi, -\mathbf{1}, \mathbf{1}))$. 
This stochasticity sufficiently diversifies the query patterns, preventing the generation of identical fingerprints and thereby evading the collision-based detection mechanism. Adhering to the official settings for ImageNet, we use a detection threshold of $25$ and a window size of $50$. Any attack triggering this threshold is marked as a failure, and the query count before detection is recorded as $1^{\textnormal{st}}$D$Q$. 
As shown in Table~\ref{tab:blacklight}, this strategy is remarkably effective: our method yields a $0\%$ Detection Rate (Det$R$). In sharp contrast, competing SOTA methods are rendered ineffective, triggering the defense with a $100\%$ Det$R$ (except BouT at $95\%$). We further compare with the randomized certifiable attack CA~\cite{hong2024certifiable} (the surrogate-free binary-search variant) in App.~\ref{app:ca}; under the same settings, DPAttack reaches the same ASR/Det$R$ while requiring smaller perturbations and far fewer queries.

\begin{table}[t!]
\footnotesize
\tabcolsep=0.05cm
  \centering
  \caption{The untargeted attack performance on real-world APIs under different perturbation constraints. Abbreviations: ASR: A., Med.Q: M.Q, Avg.Q: A.Q.}
  \scalebox{0.9}{
    \begin{tabular}{c|c|ccc|ccc|ccc|ccc}
    \Xhline{0.8pt}
    \multirow{2}[1]{*}{$\ell_\infty$} & APIs  & \multicolumn{3}{c|}{Imagga} & \multicolumn{3}{c|}{Google}& \multicolumn{3}{c|}{Tencent}& \multicolumn{3}{c}{Baidu} \\
\cline{2-14}          & Methods & A.   & M.Q& A.Q  & A.   & M.Q & A.Q& A. & M.Q  & A.Q & A. & M.Q  & A.Q  \\
    \hline
    {\multirow{3}[1]{*}{$0.05$}} & ADBA  & 83    & 61& 105    & 53   &            26& 28 &33&25&26&65&25&26  \\
       & \textbf{Ours}$_\text{opt}$ & \underline{90} & \textbf{6} & \underline{38} & \underline{91}  &  \underline{22} &          \underline{23} & \underline{67}& \textbf{10}& \textbf{15}&\textbf{88}&\textbf{9}&\textbf{13} \\

       & \textbf{Ours}$_\text{dyn}$ & \textbf{100} & \underline{23} & \textbf{20} & \textbf{100}  &  \textbf{21} &          \textbf{19} & \textbf{75}& \underline{23}& \underline{20}&\textbf{90}&\underline{21}&\underline{20} \\
    \hline
    {\multirow{3}[1]{*}{$0.1$}} & ADBA  & 99 & 41   &  70    & 68  &  24 & 26 &56&33&35&80&24&27 \\
          & \textbf{Ours}$_\text{opt}$ & \textbf{100}  & \textbf{4}& \textbf{8} & \underline{94} & \underline{7}& \underline{14} & \textbf{100}& \textbf{14}& \textbf{14} &\underline{98}&\textbf{7}&\textbf{14}\\
           & \textbf{Ours}$_\text{dyn}$ & \textbf{100}  & \underline{16}& \underline{16} & \textbf{100} & \textbf{15}& \textbf{16} & \underline{97}& \underline{17}& \underline{16} &\textbf{99}&\underline{16}&\underline{16}\\
    \Xhline{0.8pt}
    \end{tabular}%
    }
  \label{tab:api}%
\end{table}%

\begin{table}[t!]
\footnotesize
\tabcolsep=0.05cm
  \centering
  \caption{The untargeted attack performance on adversarially trained WideResNet. Max.Q: 1,000.}
  \scalebox{0.9}{
    \begin{tabular}{c|c|cccccccc}
    \Xhline{0.8pt}
    Dataset & Metrics & \multicolumn{1}{c}{HSJA} & \multicolumn{1}{c}{BouR} & \multicolumn{1}{c}{BouT} & \multicolumn{1}{c}{TtBA} & \multicolumn{1}{c}{HRayS} & \multicolumn{1}{c}{ADBA} & \multicolumn{1}{c}{\textbf{Ours}$_\text{opt}$}& \multicolumn{1}{c}{\textbf{Ours}$_\text{dyn}$} \\
    \hline
    \multirow{3}[2]{*}{CIFAR-10} & ASR   & 4.5   & 3.5   & 3.5   & 1.5   & {20.5}  & 20.0    & \textbf{22.5} &\underline{21.0}\\
          & Avg.Q & 140   & 141   & 58    & 361   & 290   & 211   & 234&205 \\
          & Med.Q & 140   & 140   & 8     & 398   & 222   & 101   & 143 &83\\
    \hline
    \multirow{3}[2]{*}{ImageNet} & ASR   & 10.5      &   10.5    & 11.0      & 3.5      &  44.5     & {45.0} &     
    \textbf{50.5} & \underline{48.5} \\
          & Avg.Q &  112     &  111     &   84    &   302    &  189     &  143     &143 &191 \\
          & Med.Q &  149     &  146     & 6      &  223     & 91      &   58    & 73&90 \\
    \Xhline{0.8pt}
    \end{tabular}%
    }
  \label{tab:defense}%
\end{table}%

\begin{table}[t!]
\footnotesize
\tabcolsep=0.05cm
  \centering
  \caption{The untargeted attack performance on Blacklight against ViT-B-32. Max.Q: 100. Outside/\textcolor{gray}{inside} parentheses: DPAttack with/without Gaussian-noise injection.}
  \scalebox{1}{
    \begin{tabular}{c|ccccccccc}
    \Xhline{0.8pt}
     Metrics & \multicolumn{1}{c}{HSJA} & \multicolumn{1}{c}{BouR} & \multicolumn{1}{c}{BouT} & \multicolumn{1}{c}{TtBA} & \multicolumn{1}{c}{HRayS} & \multicolumn{1}{c}{ADBA} & \multicolumn{1}{c}{\textbf{Ours}$_\text{opt}$}& \multicolumn{1}{c}{\textbf{Ours}$_\text{dyn}$} \\
    \hline
     ASR$\uparrow$   & 0   & 0   & 5   & 0   & 0  & 0    & \underline{41}\,\textcolor{gray}{(22)} &\textbf{43}\,\textcolor{gray}{(24.5)}\\
          Avg.Q$\downarrow$ & -   & -   & 5    & -   & -   & -   & \textbf{39} \,\textcolor{gray}{(11)}&\underline{40}\,\textcolor{gray}{(17)} \\
          Med.Q$\downarrow$ & -   & -   & 5    & -   & -   & -   & \underline{31}\,\textcolor{gray}{(10)}&\textbf{26}\,\textcolor{gray}{(17)}\\
     Det$R$$\downarrow$   & 100      &   100    & 95     & 100      & 100     & 100 &     
   \textbf{0}\,\textcolor{gray}{(78)} & \textbf{0}\,\textcolor{gray}{(76.5)} \\
          $1^{\textnormal{st}}$D$Q$ $\uparrow$&  9     &  9     &   8    &   11    &  5     &  2     &-\,\textcolor{gray}{(15.7)} &-\,\textcolor{gray}{(20.3)} \\
         
    \Xhline{0.8pt}
    \end{tabular}%
    }
  \label{tab:blacklight}%
\end{table}%

\begin{table}[t!]
\footnotesize
\tabcolsep=0.08cm
  \centering
  \caption{The untargeted attack performance on dense prediction tasks. Max.Q: 50.}
    \begin{tabular}{c|ccc|ccc}
    \Xhline{0.8pt}
    \multirow{2}[1]{*}{Methods} & \multicolumn{3}{c|}{Object Detection} & \multicolumn{3}{c}{Segmentation} \\
\cline{2-7}          & ASR   & Avg.Q & Med.Q & ASR   & Avg.Q & Med.Q \\
    \hline
    ADBA  &                   17.0  &                        30  &                        29  &                      4.5  &                        36  &                        35  \\
    \textbf{Ours} $\mathbf{d_n}$ & \textbf{                  43.5 } & \textbf{                       18 } & \textbf{                          6 } & \textbf{                  36.0 } & \textbf{                       14 } & \textbf{                          6 } \\
    \textbf{Ours}$_\text{opt}$ &                   35.5  &                        {16}  &                        10  &                   \underline{35.5}  &                        17  &                        10  \\
    \textbf{Ours}$_\text{dyn}$ &                   \underline{39.5}  &                        22  &                        20  &                   33.5  &                        20  &                        18  \\
    \Xhline{0.8pt}
    \end{tabular}%
  \label{tab:odandseg}%
\end{table}%

\begin{table}[t!]
  \centering
  \footnotesize
  \tabcolsep=0.08cm
  \caption{The ablation study on the initialization methods and PDO module. Abbreviation: WideResNet-50 (WRS-50); all one vector ($\bm{1}$); dyadic search used in ADBA (ADBAS), see Step 2 of ADBA in Sec.~\ref{subsec:discrete_attacks}. Max.Q: 500.}
    \begin{tabular}{c|c|cccccccc}
    \Xhline{0.8pt}
    Model & Search & {Gauss.} & {Unif.} & {Img.} & \multicolumn{1}{c}{$\bm{1}$} & \multicolumn{1}{c}{$\mathbf{d_b}$} & \multicolumn{1}{c}{$\mathbf{d_r}$} & \multicolumn{1}{c}{$\mathbf{d_n}$} &\multicolumn{1}{c}{DDM}\\
    \hline
    \multirow{2}[1]{*}{ViT} & ADBAS & 52.5  & 62.5  & 76.5  & 75.5  & 63.5  & 75.5  & 75.5  & 77.0 \\
          & PDO   & 70.5  & 72.0    & 76.0    & 75.5  & 82.5  & 81.5  & \textbf{85.5} & \textbf{85.5} \\
    \hline
    \multirow{2}[1]{*}{WRS-50} & ADBAS & 16.5  & 18.0    & 36.5  & 45.0    & 27.5  & 35.0    & 28.5  & 37.0 \\
         & PDO   & 42.0    & 43.5  & 43.0    & 45.0    & 44.0    & 48.0    & 47.0    & \textbf{50.5} \\
    \Xhline{0.8pt}
    \end{tabular}%
  \label{tab:ddmpdoabla}%
\end{table}%

\subsection{Generalization to Dense Prediction Tasks} \label{sec:exp_odseg}

To evaluate the generalization of DPAttack across diverse vision tasks, we evaluate its performance on object detection and segmentation against the runner-up, ADBA (detailed attack objectives in App.~\ref{app:expsetting}). As summarized in Table~\ref{tab:odandseg}, our method demonstrates superior efficacy and efficiency in both tasks. Notably, for segmentation, Ours $\mathbf{d_n}$ outperforms ADBA by over $31.5\%$ in ASR. This significant margin validates that our frequency-based priors are highly effective even for complex dense prediction models. App. Table~\ref{tab:dense1000} shows that our $\mathbf{d_n}$ still sustains an ASR improvement of over $31\%$ (OD) and $24.5\%$ (Segmentation) with a larger $1,000$-query budget.

\subsection{Ablation Study}\label{sec:ablation}

\noindent\textbf{DDM.}
To evaluate our DDM initialization, we compare it against standard baselines (Gaussian, Uniform, Random Images, and All-ones vector $\mathbf{1}$) while keeping the subsequent PDO search fixed. As shown in Table~\ref{tab:ddmpdoabla}, DDM significantly outperforms traditional priors. For instance, on ViT, DDM achieves an ASR of $85.5\%$, outstripping Gaussian and Uniform initializations by $15.0\%$ and $13.5\%$, respectively. On the robust WideResNet-50, DDM still provides a substantial gain ($+8.5\%$ over Gaussian). Notably, the advantage of DDM over individual components ($\mathbf{d_b}, \mathbf{d_r}, \mathbf{d_n}$) highlights the necessity of our integrated initialization pipeline. These results confirm that DDM provides a \textit{structurally informative prior} that aligns with the target model's spectral sensitivity, thereby boosting search efficiency. Complementary studies on the individual contributions of $\mathbf{d_n}/\phi(\hat{\mathbf{d}}_\mathbf{r})$ in DDM are detailed in App.~\ref{app:ablation}. Results show that using our variance-based $\mathbf{d_n}$ alone can even yield higher ASR on standard models compared to DDM. Specifically, $\mathbf{d_n}$ is
critical for enhancing overall ASR, while $\phi(\hat{\mathbf{d}}_\mathbf{r})$ balances performance across diverse models.

\noindent\textbf{PDO.} 
The PDO module is designed to refine the initial direction. To assess its impact, we replace PDO with the basic structured dyadic search (denoted as ADBAS, the search strategy used in ADBA). Focusing on the last column in Table~\ref{tab:ddmpdoabla}, we observe that under DDM initialization, PDO outperforms ADBAS by 8.5\% on ViT and 13.5\% on WideResNet-50. This empirically validates Theorems~\ref{thm:main_text} and~\ref{theo:complexity-fixed}, proving that PDO's sign-consistent run partitioning is superior to fixed dyadic splitting. Moreover, PDO consistently enhances ASR across virtually all initialization methods. Notably, when initialized with an unstructured vector ($\mathbf{1}$), the ASRs of PDO and ADBAS are identical (e.g., 45.0\% on WRS-50). This demonstrates that without an initial spatial structure, PDO effectively degrades to a standard dyadic search, whereas its true power is unlocked when paired with DDM's structural priors.

\begin{figure}
\centering
  \includegraphics[width=\linewidth]{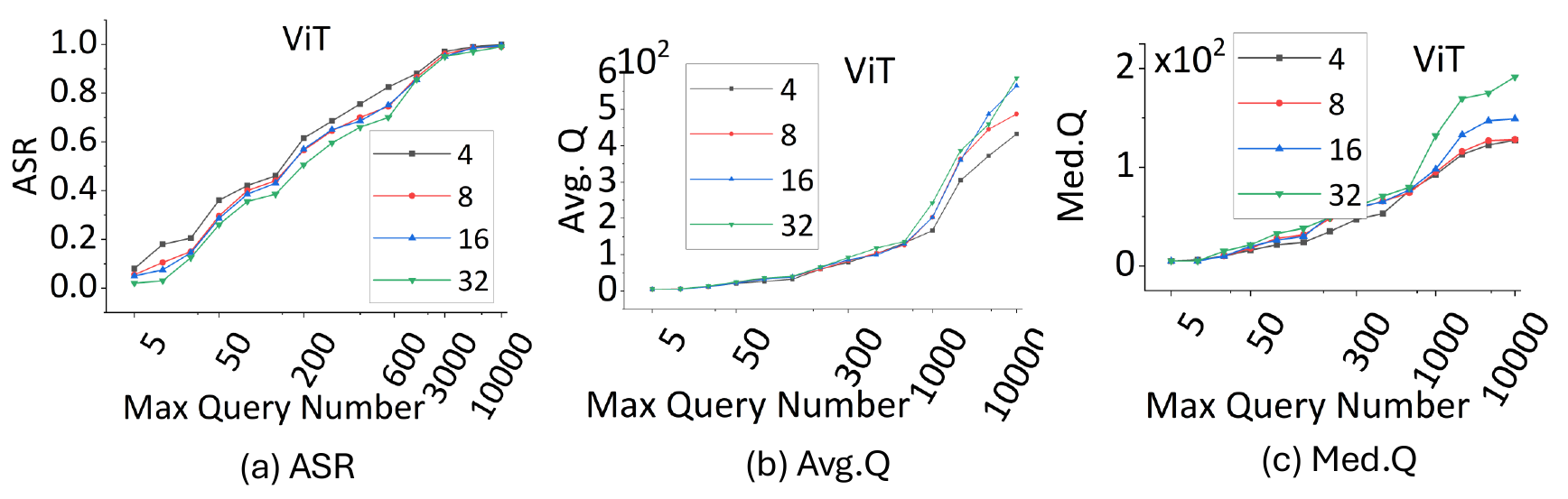}
  \caption{The influence of block size $w$ used in BDCT.}
  \label{fig:blocksizeResnet50}
\end{figure}

\noindent\textbf{Block Size $w$.}
We evaluate the sensitivity of the block size $w \in \{4, 8, 16, 32\}$ on ImageNet. As illustrated in Fig.~\ref{fig:blocksizeResnet50}, while the overall ASR remains relatively stable, the choice of $w$ leads to fluctuations in query efficiency. Furthermore, the optimal $w$ varies across architectures and settings. For instance, on ImageNet, WideResNet-50 and ViT-B-32 peak at $w=16$ and $w=4$, respectively, while ResNet-18 favors $w=2$ on CIFAR-10. This ``one-size-does-not-fit-all'' observation underscores the necessity of our dynamic block size selection algorithm, which adaptively identifies a suitable $w$ for a given target model to achieve better ASR and query efficiency.

\section{Discussion}
Targeted attacks are inherently more challenging than untargeted attacks due to their narrower adversarial regions and higher-curvature decision boundaries. Although our BDCT initialization effectively captures frequency-band sensitivity and block-wise structural priors, it does not explicitly preserve Fourier phase information, which may limit its ability to guide perturbations toward a specific target class. A promising future direction is to combine the BFS amplitude prior with target-image FFT phase information. 

Regarding perceptibility, our method achieves higher ASR while maintaining comparable PSNR, SSIM, and LPIPS values to baseline methods under the $\ell_2$ constraint. Under the $\ell_\infty$ setting, the visual quality of our perturbations is similar to other sign-flip-based methods such as RayS and ADBA despite substantially improved ASR. Compared with other baselines, our perturbations introduce only slight additional visual degradation while providing significantly stronger attack effectiveness.

\section{Conclusion}

In this work, we propose DPAttack, a highly efficient hard-label black-box attack framework that bridges spectral sensitivity with spatial structural search. By leveraging a systematic analysis of adversarial sensitivity across BDCT frequency coefficients, we develop a customized initialization strategy that provides a structurally informative starting point. To exploit this prior, we introduce a pattern-driven direction optimization module, which significantly accelerates gradient sign estimation by searching along sign-consistent runs. Furthermore, we establish a novel theoretical framework for hard-label sign-flipping attacks, providing proofs for the superiority of our initialization and search strategies. Extensive experiments across various architectures, datasets, tasks and real-world APIs demonstrate that DPAttack consistently surpasses SOTA methods in both attack success rate and query efficiency.

\section*{Acknowledgments}
We are grateful to the anonymous reviewers
for their valuable guidance and insightful comments. J. Zhou's, F. Li's, and part of J. Liu's work was supported in part by Macau Science and Technology Development Fund under 001/2024/SKL, 0119/2024/RIB2 and 0110/2025/R1B2; in part by Research Committee at University of Macau under MYRG-CRG2025-00031-FST and MYRG-GRG2025-00086-FST; in part by the Guangdong Basic and Applied Basic Research Foundation under Grant 2024A1515012536. I. Echizen's and part of J. Liu's work was supported by JSPS KAKENHI Grant JP24H00732; by JST CREST Grants JPMJCR20D3 and JPMJCR2562 including AIP challenge program; and by JST K Program Grant JPMJKP24C2 Japan.
\appendix
\section*{Ethical Considerations}

We structure our ethical analysis based on the stakeholder framework and principles outlined in the Menlo Report~\cite{bailey2012menlo} and the USENIX Security Ethics Guidelines. Our research involves the development of a query-efficient hard-label black-box attack. We have identified the following key stakeholders: (1) Commercial Machine Learning as a Service (MLaaS) providers (Google, Baidu, Tencent, Imagga); (2) Authors of academic defense models (e.g., Blacklight); (3) Human subjects in the datasets; and (4) The broader research community and society.

\noindent\textbf{Interaction with Commercial APIs (MLaaS Providers).}
To evaluate the practical efficiency of our attack and the robustness of real-world APIs, we tested our method against live APIs provided by Google, Baidu, Tencent, and Imagga.
\begin{itemize}[leftmargin=*]
    \item \textbf{Minimizing harm (beneficence):} We strictly limited the query volume to the minimum necessary to validate the attack's capability. Unlike Denial-of-Service (DoS) attacks or resource exhaustion exploits, our method focuses on finding adversarial examples with a low query budget. We maintained a conservative request rate to minimize the risk of degrading service quality for other users or incurring high computational costs.
    \item \textbf{Terms of Service (ToS) and legal compliance:} We acknowledge that probing APIs for vulnerabilities or adversarial robustness may technically violate certain standard ToS regarding reverse engineering or automated testing. However, consistent with the standard practice in security research, we deemed this limited violation necessary to understand the vulnerability of real-world systems to hard-label attacks. We utilized legitimately registered accounts and, where applicable, paid the required fees for API usage, ensuring no financial loss to the providers. We did not attempt to exfiltrate model weights, proprietary data, or compromise the underlying infrastructure.
\end{itemize}

\noindent\textbf{Responsible Disclosure.}
Regarding the academic defense model \textit{Blacklight}, we plan to share our detailed analysis with the authors upon publication. For the commercial APIs, since our attack demonstrates a general susceptibility of neural networks to adversarial examples rather than a specific implementation bug or infrastructure vulnerability (e.g., memory corruption), a standard Coordinated Vulnerability Disclosure process is less applicable. However, we are in the process of notifying the security teams of the respective vendors about the general robustness implications found in our study upon the acceptance of this work.

\noindent\textbf{Dataset Usage and Privacy (Respect for Persons).}
Our experiments utilize several public datasets.
\begin{itemize}[leftmargin=*]
    \item \textbf{Biomedical data (PathMNIST):} We treated the PathMNIST dataset with high caution given its biomedical nature. We utilized only the publicly available of the dataset released by the original authors. We did not attempt to de-anonymize any data subjects or use the data for any purpose other than evaluating algorithmic effectiveness.
    \item \textbf{Large-scale data:} For datasets containing human figures (e.g., SA-1B, MS-COCO, ImageNet), we adhered to their respective licenses. Our research focuses on pixel-level perturbations and does not generate offensive content or violate the dignity of the individuals depicted.
\end{itemize}

\noindent\textbf{Risk-Benefit Analysis and Mitigation.}
\begin{itemize}[leftmargin=*]
    \item \textbf{Potential misuse:} We recognize that publishing efficient hard-label attack code could potentially be used by malicious actors to evade content filters or integrity checks.
    \item \textbf{Justification for publication:} However, hard-label attacks represent a realistic threat model where attackers have limited access. By demonstrating that SOTA defenses, commercial APIs and foundation models are still vulnerable to query-efficient attacks, our work highlights a critical security gap. The benefits of motivating stronger defenses and exposing the limitations of current ``security through obscurity'' (reliance on label-only access) outweigh the risks.
    \item \textbf{Mitigation:} To mitigate potential harm, we will release our code with a responsible usage license.
\end{itemize}

\noindent\textbf{Team Wellbeing.}
Our research team was not exposed to psychologically disturbing content during the handling of standard vision datasets. Furthermore, all API testing was performed using dedicated research infrastructure to ensure operational security and isolate experimental traffic.


\section*{Open Science}
In alignment with the USENIX Security Open Science Policy and our commitment to reproducibility, we make all research artifacts associated with this paper publicly available. These artifacts, including data and code, can be accessed via the following repository: \url{https://doi.org/10.5281/zenodo.20322560}. The repository contains the datasets used in our evaluation, the source code for the attack algorithms and experimental analysis, as well as the corresponding model weights necessary to reproduce our results.

\bibliographystyle{plainurl}
\bibliography{reference}
\appendix
\section{Proof of Theorem \ref{theo:angle}}\label{app:theo1_proof}

\renewcommand{\thetheorem}{1}
\begin{tcolorbox}[colback=blue!5!white,colframe=black!75!black
,
boxsep=0pt,
  left=4pt,
  right=4pt,
  top=4pt,
  bottom=4pt,
  boxrule=0.7pt,
  arc=2pt
]
\begin{theorem}[Approximation Lower Bound for NRayS]
\label{theo:angle_app}
    Let $\mathbf{u}\in\mathbb{S}^{d-1}$ be an arbitrary unit vector.
    Let $\{\mathbf{d}_j\}_{j=1}^m$ be $m$ vectors sampled independently and uniformly from the binary hypercube $\mathcal{H}\equiv\{-1,+1\}^d$, and let $\hat{\mathbf{d}}_j=\mathbf{d}_j/{\sqrt{d}}$ be their normalized counterparts.
    For any precision $\zeta \in (0,1)$ and failure probability $\delta \in (0,1)$, there exists a constant $C \ge 1$ such that if the sample size satisfies $m \geq -(\ln{\delta})e^{Cd\zeta^2}$, then with probability at least $1-\delta$, there exists at least one index $j \in \{1,\dots,m\}$ satisfying:
      \begin{equation}\label{eq:theo1_1_main_app}
    \begin{gathered}
\mathbf{u}^\top\hat{\mathbf{d}}_j \geq \zeta.
\end{gathered}
    \end{equation}
\end{theorem}
\end{tcolorbox}

\begin{proof}
For any unit vector $\mathbf{u}\in\mathbb{S}^{d-1}$, consider
\begin{equation}\label{eq:theorangle_proof1}
X=\mathbf{u}^\top\hat{\mathbf{d}}
=\frac{1}{\sqrt{d}}\sum_{i=1}^d \mathbf{u}[i]\mathbf{d}[i],
\end{equation}
where $\mathbf{d}[i]\in\{\pm1\}$ are independent Rademacher random variables \cite{40597} (i.e., uniformly sampled from $\{\pm1\}$).
Conditioned on the fixed vector $\mathbf{u}$, the randomness of $X$ arises
solely from the random signs $\mathbf{d}[i]$. Thus, $X$ is a sum of independent, zero-mean random variables with respect to
the randomness of $\mathbf{d}$.
Define
\begin{equation}\label{eq:theorangle_proof2}
X_i := \frac{\mathbf{u}[i]\mathbf{d}[i]}{\sqrt{d}},
\end{equation}
so that $X=\sum_{i=1}^d X_i$, where each $X_i$ satisfies
\[
X_i \in \left[-\frac{|\mathbf{u}[i]|}{\sqrt{d}},
\frac{|\mathbf{u}[i]|}{\sqrt{d}}\right].
\]
After applying \textit{Hoeffding's inequality} \cite{40597}, $\forall~t>0$,
      \begin{equation}\label{eq:theorangle_proof3}
    \begin{gathered}
\mathbb{P}(X\geq\mathbb{E}[X]+t)\leq\exp{\Big(-\frac{2t^2}{\sum_{i=1}^d({2|\mathbf{u}[i]|}/{\sqrt{d}})^2}\Big)}\\\leq\exp{\big(-0.5dt^2\left\|\mathbf{u}\right\|_2^{-2}\big)}.
        \end{gathered}
    \end{equation}
Since the variables $\mathbf{u}[i]$ and $\mathbf{d}[i]$ are independent and $\mathbb{E}[\mathbf{d}[i]]=0$, we have:
  \begin{equation}\label{eq:theorangle_proof3_1}
    \begin{gathered}
\mathbb{E}[X]=\sum_{i=1}^d\mathbb{E}[X_i]=\frac{1}{\sqrt{d}}\sum_{i=1}^d\mathbb{E}[\mathbf{u}[i]\mathbf{d}[i]]\\=\frac{1}{\sqrt{d}}\sum_{i=1}^d\mathbb{E}[\mathbf{u}[i]]\mathbb{E}[\mathbf{d}[i]]=0.
        \end{gathered}
    \end{equation}
Therefore, in Eq.~(\ref{eq:theorangle_proof3}), by replacing $\mathbb{E}[X]=0$, $||\mathbf{u}||_2^2=1$ and let $t$ be a positive cosine similarity value $\zeta\in(0,1)$, we obtain:
  \begin{equation}\label{eq:theorangle_proof4}
    \begin{gathered}
\mathbb{P}(X\geq\zeta)\leq\exp{\big(-0.5{\zeta^2d}\big)}.
        \end{gathered}
    \end{equation}
We define the probability $p$ of the event $X\geq\zeta$ occurs as:
  \begin{equation}\label{eq:theorangle_proof6_app}
    \begin{gathered}
p=\mathbb{P}(X\geq\zeta):=\exp{\big(-\frac{C\zeta^2d}{2}\big)},
        \end{gathered}
    \end{equation} 
    where the constant $C\geq1$ that makes $\exp{\big(-{C\zeta^2d}/{2}\big)}\leq\exp{\big(-{\zeta^2d}/{2}\big)}$ and satisfies Eq.~(\ref{eq:theorangle_proof4}).
We independently sample $m$ vectors $\mathbf{d}_j$ from $\mathcal{H}$, recalling $\hat{\mathbf{d}}_j=\mathbf{d}_j/\sqrt{d}$. Each $X_j:=\mathbf{u}^\top\hat{\mathbf{d}}_j$ satisfies:
  \begin{equation}\label{eq:theorangle_proof6}
    \begin{gathered}
\mathbb{P}(X_j<\zeta)=1-p.
        \end{gathered}
    \end{equation} 
Then the probability of all $m$ failures is $(1-p)^m$.
Since we have $-p\geq-1$, by applying Bernoulli's inequality, we have:
  \begin{equation}\label{eq:theorangle_proof6}
    \begin{gathered}
(1-p)^m\leq\exp(-pm).
        \end{gathered}
    \end{equation} 
    To ensure that the probability of all $m$ samples failing is at most $\delta$ ($0<\delta<1$),
it suffices to require
\begin{equation}\label{eq:theorangle_proof7}
    \begin{gathered}
\exp(-pm)\leq\delta\Rightarrow m\geq-(\ln{\delta})\exp{(C\zeta^2d)}.
        \end{gathered}
    \end{equation} 
Under this condition, the probability that at least one sampled direction
$\mathbf{d}_j \in \{\pm1\}^d$ satisfies $X_j \ge \zeta$ is at least $1-\delta$, i.e.,
\begin{equation}\label{eq:theorangle_proof9}
    \mathbb{P}\!\left(
    \exists\, j \in \{1,\dots,m\} \text{ such that }
    \mathbf{u}^\top \hat{\mathbf{d}}_j \ge \zeta
    \right)
    \ge 1-\delta.
\end{equation}
Hence, it completes the proof.
\end{proof}

\section{Proof of Theorem \ref{theo:raysappro}}\label{app:theo2_proof}

\begin{assumption}\label{ass1} We assume that the CE loss $\mathcal{L}(\mathbf{x},y)$ is continuous and locally smooth \cite{madry2018towards,cheng2019query,cheng2020sign} to ensure that first-order Taylor approximations of
the loss are locally valid. 
 \end{assumption}

\begin{assumption}\label{ass1_2} The loss function $\mathcal{L}(\mathbf{x},y):\mathbb{R}^{d}\rightarrow\mathbb{R}$ satisfies subset-wise directional Lipschitz continuity with a constant $R>0$. This means that for any vector $\mathbf{d}\in\mathcal{H}$, any subset $B_k^{(l)}$ (shortened as $B$), and the vector $\mathbf{d}_k^{(l)}$ (abbreviated as $\mathbf{d}_B$) obtained by flipping the signs of $\mathbf{d}$'s elements whose indices are within subset $B$, the following holds:
   
         \begin{equation}\label{eq:lip2}
    \begin{gathered}
|\mathcal{L}\big(\mathbf{x}+g({\mathbf{d}})\cdot{\mathbf{d}}_B\big)-\mathcal{L}\big(\mathbf{x}+g({\mathbf{d}})\cdot{\mathbf{d}}\big)|\leq R\left\|{\mathbf{d}}_B-{\mathbf{d}}\right\|_1.
        \end{gathered}
    \end{equation} 
    
\end{assumption}

\noindent\textit{Remark.}
    As gradients exhibit local similarity \cite{ilyas2019prior}, Assumption \ref{ass1_2} holds in most cases. We provide a more detailed discussion in Appendix \ref{app:ass1_2}. Moreover, while the $\ell_2$ norm is standard for Lipschitz analysis, in this specific scenario, the $\ell_1$ norm offers a more straightforward and practically meaningful measure, especially when dealing with vectors with values in $\{\pm1\}$ and subset operations. It allows the Lipschitz constant $R$ to be more connected to the length of the flipped subset.

Assumption~\ref{ass1_2} guarantees that
loss variations induced by subset-wise sign flips are well-controlled, while Assumption~\ref{ass1} ensures that first-order Taylor approximations of
the loss are locally valid.
Together, these regularity conditions allow observed loss differences to serve
as reliable proxies for directional derivative influences, leading to the
following subset selection guarantee. We now formalize this notion by introducing the following subset-wise influence
measure.

\renewcommand{\thedefinition}{1}
\begin{definition}[Subset-wise Directional Derivative Influence]
\label{def:def1_app}
For a selected index subset $B_{ks}^{(l)}$ at round $l$ of HRayS (following the definition in Step 2 of HRayS; see Sec.~\ref{subsec:discrete_attacks}), we define its influence as the cumulative magnitude of the gradient components within that subset:
\begin{equation}\label{eq:theo2_proof2_app}
    I(B_{ks}^{(l)}) := \sum_{i\in B_{ks}^{(l)}} \left| \frac{\partial \mathcal{L}(\tilde{\mathbf{x}},y)}{\partial \tilde{\mathbf{x}}[i]} \bigg|_{\tilde{\mathbf{x}}=\mathbf{x}+g(\mathbf{d}_{\text{best}})\mathbf{d}_{\text{best}}} \right|.
\end{equation}
\end{definition}
This metric quantifies the collective contribution of the dimensions in
$B_{ks}^{(l)}$ to the local increase of the loss. A larger $I(\cdot)$ implies the subset contains high-impact dimensions. Based on this influence measure, we establish the following subset selection
guarantee for HRayS.

\renewcommand{\thelemma}{1}
\begin{lemma}
\label{lem:lem1_app}
Under Assumptions~\ref{ass1} and~\ref{ass1_2},
, the subset $B^{(l)}_{ks}$ selected by HRayS at round $l$ satisfies:
    \begin{equation}
    \label{eq:theo2_proof3}
    I(B_{ks}^{(l)}) \geq \max\nolimits_{B_k^{(l)}} I(B_{k}^{(l)}) - \eta_l,
    \end{equation}
    where $\eta_l$ represents the selection error bound at round $l$.
\end{lemma}

\begin{proof}
We prove Lemma~\ref{lem:lem1_app} by relating the loss variation induced by a subset
flip to its directional derivative influence.

Let $\tilde{\mathbf{x}}=\mathbf{x}+g(\mathbf{d})\mathbf{d}$ and
$\Upsilon=g(\mathbf{d})(\mathbf{d}_B-\mathbf{d})$.
By a first-order Taylor expansion of $\mathcal{L}$ around $\tilde{\mathbf{x}}$,
we have
\begin{equation}\label{eq:l1prf_new1}
\mathcal{L}(\tilde{\mathbf{x}}+\Upsilon)-\mathcal{L}(\tilde{\mathbf{x}})
= \nabla\mathcal{L}(\tilde{\mathbf{x}})^\top\Upsilon + o(\|\Upsilon\|).
\end{equation}
Ignoring higher-order terms, this gives
\begin{equation}\label{eq:l1prf_new2}
\mathcal{L}(\tilde{\mathbf{x}}+\Upsilon)-\mathcal{L}(\tilde{\mathbf{x}})
\approx g(\mathbf{d}) \sum\nolimits_{i\in B}
\nabla\mathcal{L}(\tilde{\mathbf{x}})[i]\,
(\mathbf{d}_B[i]-\mathbf{d}[i]).
\end{equation}
Since $\mathbf{d}_B[i]=-\mathbf{d}[i]$ for all $i\in B$,
we have $\mathbf{d}_B[i]-\mathbf{d}[i]=-2\mathbf{d}[i]$, and therefore
\begin{equation}\label{eq:l1prf_new3}
\mathcal{L}(\tilde{\mathbf{x}}+\Upsilon)-\mathcal{L}(\tilde{\mathbf{x}})
\approx
-2g(\mathbf{d}) \sum\nolimits_{i\in B}
\nabla\mathcal{L}(\tilde{\mathbf{x}})[i]\mathbf{d}[i].
\end{equation}
Taking absolute values and applying the triangle inequality yields
\begin{equation}\label{eq:l1prf_new4}\small
\big|\mathcal{L}(\tilde{\mathbf{x}}+\Upsilon)-\mathcal{L}(\tilde{\mathbf{x}})\big|
\;\le\;
2g(\mathbf{d}) \sum\nolimits_{i\in B}
\big|\nabla\mathcal{L}(\tilde{\mathbf{x}})[i]\big|
=2g(\mathbf{d})\, I(B).
\end{equation}

On the other hand, by Assumption~\ref{ass1_2} (subset-wise Lipschitz continuity),
we have
\begin{equation}\label{eq:l1prf_new5}
\big|\mathcal{L}(\tilde{\mathbf{x}}+\Upsilon)-\mathcal{L}(\tilde{\mathbf{x}})\big|
\le
R\|\mathbf{d}_B-\mathbf{d}\|_1
=2R|B|.
\end{equation}
Equations~(\ref{eq:l1prf_new4}) and~(\ref{eq:l1prf_new5}) together imply that
the loss variation induced by flipping a subset $B$ is a bounded and monotone
function of its directional derivative influence $I(B)$.

At round $l$, HRayS selects a subset $B_{ks}^{(l)}$ whose induced loss change
is within $\eta_l$ of the maximum achievable loss change over all candidate
subsets $\{B_k^{(l)}\}$. Since the loss change is monotone in $I(B)$,
this directly implies
\begin{equation}
I(B_{ks}^{(l)})
\ge
\max\nolimits_{B_k^{(l)}} I(B_k^{(l)}) - \eta_l,
\end{equation}
which completes the proof.
\end{proof}
\noindent\textit{Remark.}
Assumption~\ref{ass1_2} is not used to derive a tight bound on $I(B)$,
but rather to ensure that loss variations induced by subset-wise sign flips
are comparable across different subsets. This regularity justifies using empirical loss differences as a stable
ranking criterion in HRayS’s greedy subset selection.

Lemma~\ref{lem:lem1_app} establishes that HRayS preserves high-influence
gradient coordinates through greedy subset selection. To translate this influence-based property into a statement about the boundary
distance function $g(\mathbf{d})$, we next relate $g(\mathbf{d})$ to the change
in cross-entropy loss under local linearity.

At each round $l$, HRayS's behavior can be interpreted as evaluating whether a perturbation in a particular direction causes a significant decrease in $g({\mathbf{d}}_{best})$ within a subspace (subset $B_{ks}^{(l)}$). Given the current best direction $\mathbf{d}_{\text{best}}$, we flip the coordinates in the subset $B_{ks}^{(l)}$ to obtain the direction $\mathbf{d}_{ks}^{(l)}$ (referring to Step 2 of HRayS in Sec.~\ref{subsec:discrete_attacks}). We define:
     \begin{equation}\label{eq:theo2_proof1}
    \begin{gathered}
\Delta_{B_{ks}^{(l)}}:=g({\mathbf{d}}_{best})-g({\mathbf{d}}_{ks}^{(l)})
        \end{gathered}
    \end{equation}
HRayS's strategy prioritizes greedily searching for the subsets $B_{ks}^{(l)}$ that maximize $\Delta_{B_{ks}^{(l)}}$, and focuses finer-grained search in the next round. We next investigate the relationship between $g(\mathbf{d})$ and the loss function $\mathcal{L}(\mathbf{x}+g(\mathbf{d})\mathbf{d},y)$, then combine this connection with Eq.~(\ref{eq:theo2_proof1}) to derive the direction update mechanism of HRayS in relation to the CE loss.

\renewcommand{\thelemma}{2}
\begin{lemma}\label{lemmagdandgrad}
    Under the assumption of a locally linear decision boundary, i.e., the classifier's logit function is $f(\mathbf{x})=\mathbf{W}^\top\mathbf{x}+\mathbf{b}$, the boundary distance function $g(\mathbf{d})$ is approximately inversely proportional to the single-step change in softmax CE loss, $\Delta\mathcal{L}=\mathcal{L}(\mathbf{x}+g(\mathbf{d})\mathbf{d},y)-\mathcal{L}(\mathbf{x},y)$. That is:
         \begin{equation}\label{eq:theo2_lemma1}
    \begin{gathered}
g(\mathbf{d})\propto\frac{1}{\mathcal{L}(\mathbf{x}+g(\mathbf{d})\mathbf{d},y)-\mathcal{L}(\mathbf{x},y)} 
\end{gathered}
    \end{equation}
\end{lemma}

\begin{proof} For any class $m\in\{1,\cdots,M\}$, the logit value is $f_m(\mathbf{x})$. We assume $\mathbf{x}$ with true class $y$ is correctly classified and that the nearest class at the decision boundary is $k$ where $k\neq y$. The linear decision boundary $\partial\mathcal{B}_{yk}$ between $\mathbf{x}$ and the nearest class $k$ is defined by the equation $f_y(\mathbf{x})-f_k(\mathbf{x})=0$, which is:
             \begin{equation}\label{eq:theo2_lemma1_p1}
    \begin{gathered}
(\mathbf{w}_y-\mathbf{w}_k)^\top\mathbf{x}+(b_y-b_k)=0,
\end{gathered}
    \end{equation}
  where $\mathbf{w}_y$ represents the $y$-th column of the weight matrix $\mathbf{W}$. $g(\mathbf{d})$ is the distance from point $\mathbf{x}$ along the direction vector $\mathbf{d}$ to the decision boundary $\partial\mathcal{B}_{yk}$. The point $\mathbf{x}+g(\mathbf{d})\mathbf{d}$ lies on the boundary:
               \begin{equation}\label{eq:theo2_lemma1_p2}
    \begin{gathered}
(\mathbf{w}_y-\mathbf{w}_k)^\top(\mathbf{x}+g(\mathbf{d})\mathbf{d})+(b_y-b_k)=0.
\end{gathered}
    \end{equation}
Derived from this, we obtain the formulation of $g(\mathbf{d})$ is:
         \begin{equation}\label{eq:theo2_lemma1_p4}
    \begin{gathered}
g(\mathbf{d}) =\frac{-\big((\mathbf{w}_y-\mathbf{w}_k)^\top\mathbf{x}+(b_y-b_k)\big)}{(\mathbf{w}_y-\mathbf{w}_k)^\top\mathbf{d}}\\
=\frac{-(f_y(\mathbf{x})-f_k(\mathbf{x}))}{(\mathbf{w}_y-\mathbf{w}_k)^\top\mathbf{d}}.
\end{gathered}
    \end{equation}
The intuition of this formula is: the distance to the boundary is proportional to the current logit margin and inversely proportional to the projection of direction $\mathbf{d}$ onto the normal vector of the decision boundary $\mathbf{w}_y-\mathbf{w}_k$.

On the other hand, $\Delta\mathcal{L}$ is the change in the softmax CE loss after moving a small step $\alpha$ from $\mathbf{x}$ along the direction $\mathbf{d}$. The softmax CE loss is defined as:            \begin{equation}\label{eq:theo2_lemma1_p3}
    \begin{gathered}
\mathcal{L}(\mathbf{x},y)=-\log{\frac{e^{f_y(\mathbf{x})}}{\sum_{m=1}^M e^{f_m(\mathbf{x})}}}
=-f_y(\mathbf{x})+\log\sum_{m=1}^M e^{f_m(\mathbf{x})}.
\end{gathered}
    \end{equation}
    For a small step size $\alpha$, we can use a first-order Taylor expansion to approximate the change in the loss function:
                   \begin{equation}\label{eq:theo2_lemma1_p5}
    \begin{gathered}
\Delta\mathcal{L}=\mathcal{L}(\mathbf{x}+\alpha\mathbf{d},y)-\mathcal{L}(\mathbf{x},y)\approx\alpha\cdot\nabla\mathcal{L}(\mathbf{x},y)^\top\mathbf{d}.
\end{gathered}
    \end{equation}
    Now let's compute the gradient of the loss function $\nabla\mathcal{L}(\mathbf{x},y)$ according to the chain rule.
    \begin{equation}\label{eq:theo2_lemma1_p6}
    \begin{gathered}
\nabla\mathcal{L}(\mathbf{x},y)=\nabla_\mathbf{x}\Big(-f_y(\mathbf{x})+\log\sum\nolimits_{m=1}^M{e^{f_m(\mathbf{x})}}\Big)\\=-\mathbf{w}_y+\frac{\sum\nolimits_{m=1}^M(e^{f_m(\mathbf{x})}\mathbf{w}_m)}{\sum\nolimits_{m=1}^Me^{f_m(\mathbf{x})}}\\=\frac{\sum\nolimits_{m=1}^M(e^{f_m(\mathbf{x})}(\mathbf{w}_m-\mathbf{w}_y))}{\sum\nolimits_{m=1}^Me^{f_m(\mathbf{x})}}.
\end{gathered}
    \end{equation}
Let $p_j(\mathbf{x})=\frac{e^{f_j(\mathbf{x})}}{\sum\nolimits_{m=1}^Me^{f_m(\mathbf{x})}}$ be the softmax probability for class $j$. We have:     \begin{equation}\label{eq:theo2_lemma1_p7}
    \begin{gathered}
\nabla\mathcal{L}(\mathbf{x},y)=\sum_jp_j(\mathbf{x})(\mathbf{w}_j-\mathbf{w}_y).
\end{gathered}
    \end{equation}
Substituting this gradient back into Eq.~(\ref{eq:theo2_lemma1_p5}):
                   \begin{equation}\label{eq:theo2_lemma1_p8}
    \begin{gathered}
\Delta\mathcal{L}\approx\alpha\sum_jp_j(\mathbf{x})(\mathbf{w}_j-\mathbf{w}_y) ^\top\mathbf{d}.
\end{gathered}
    \end{equation}
To simplify the analysis, we make an assumption common in adversarial attack literature: when a point $\mathbf{x}$ is near a decision boundary, its softmax probabilities are dominated by the two most relevant classes (the true class $y$ and the nearest class $k$). Thus, $p_j(\mathbf{x})\approx0$ for $j\notin\{y,k\}$. Then,
                   \begin{equation}\label{eq:theo2_lemma1_p9}
    \begin{gathered}
\Delta\mathcal{L}\approx\alpha \big(p_y(\mathbf{x})(\mathbf{w}_y-\mathbf{w}_y) ^\top\mathbf{d}+p_k(\mathbf{x})(\mathbf{w}_k-\mathbf{w}_y) ^\top\mathbf{d}\big)\\
\approx\alpha p_k(\mathbf{x})(\mathbf{w}_k-\mathbf{w}_y) ^\top\mathbf{d}.
\end{gathered}
    \end{equation}
By considering Eqs.~(\ref{eq:theo2_lemma1_p4}) and (\ref{eq:theo2_lemma1_p9}) together, we obtain:
                   \begin{equation}\label{eq:theo2_lemma1_p10}
    \begin{gathered}
g(\mathbf{d})\approx\frac{\alpha\cdot p_k(\mathbf{x})\big(f_y(\mathbf{x})-f_k(\mathbf{x})\big)}{\Delta\mathcal{L}}.
\end{gathered}
    \end{equation}
    For a fixed point x and a small step size $\alpha$, the numerator is a positive constant that does not depend on the direction $\mathbf{d}$. Thus, we arrive at the final inverse proportionality: 
                       \begin{equation}\label{eq:theo2_lemma1_p11}
    \begin{gathered}
g(\mathbf{d})\propto\frac{1}{\Delta\mathcal{L}}.
\end{gathered}
    \end{equation}
\end{proof}

Lemma~\ref{lem:lem1_app} ensures that HRayS identifies subsets with
near-maximal directional derivative influence, while
Lemma~\ref{lemmagdandgrad} links this influence to loss variation and
boundary proximity.
Together, these results provide the analytical foundation for relating
HRayS’s hierarchical subset search to gradient-aligned direction recovery. We now state an algorithmic assumption that characterizes the approximate
greedy behavior of HRayS. This assumption is not required for the preceding
analytical lemmas, but is essential for controlling error accumulation in the
global guarantee of Theorem~\ref{theo:raysappro}.

\begin{assumption}\label{ass2}
HRayS approximately selects subsets corresponding to steep ascent at the $l$-th
round, i.e., the selection error $\eta_l$ is bounded within a small range.
\end{assumption}

\renewcommand{\thetheorem}{2}
\begin{tcolorbox}[colback=blue!5!white,colframe=black!75!black
,
boxsep=0pt,
  left=4pt,
  right=4pt,
  top=4pt,
  bottom=4pt,
  boxrule=0.7pt,
  arc=2pt
]
\begin{theorem}[Gradient Approximation Guarantee for HRayS]
\label{theo:raysappro_app}
Assume that the gradient magnitude is uniformly bounded, i.e.,
$\|\nabla \mathcal{L}(\mathbf{x},y)\|_\infty \le G_{\max}$.
Let $\eta_l$ denote the maximum deviation between the directional derivative
influence of the subset $B_{ks}^{(l)}$ selected by HRayS at round $l$
and that of the optimal subset $\arg\max_{B_k^{(l)}} I(B_k^{(l)})$.
Let $\hat{\mathbf{d}}^*$ be the final direction produced by HRayS.
Then, the alignment between $\hat{\mathbf{d}}^*$ and the true gradient sign satisfies
\begin{equation}
\mathrm{sgn}(\nabla \mathcal{L}(\mathbf{x},y))^\top \hat{\mathbf{d}}^*
\;\ge\;
\frac{\|\nabla \mathcal{L}(\mathbf{x},y)\|_1}{G_{\max}}
\cdot
(1-\varpi),
\end{equation}
where
$\varpi =
\sum_{l=1}^{\lceil \log_2 d \rceil}
\eta_l \|\nabla \mathcal{L}(\mathbf{x},y)\|_1^{-1}$
denotes the cumulative relative error induced by suboptimal subset selections
across all hierarchical rounds.
\end{theorem}
\end{tcolorbox}

\begin{proof} According to Eq.~(\ref{eq:theo2_proof1}), for an invariant $g({\mathbf{d}}_{best})$, maximizing $\Delta_{B_{ks}^{(l)}}$ is approximately corresponds to minimize $g({\mathbf{d}}_{ks}^{(l)})$ within all $2^l$ subsets. Based on Lemma \ref{lemmagdandgrad}, minimizing $g({\mathbf{d}}_{ks}^{(l)})$ is approximately equivalent to maximizing $\mathcal{L}(\mathbf{x}+g({{\mathbf{d}}_\text{best}}){{\mathbf{d}}_{ks}^{(l)}})-\mathcal{L}(\mathbf{x}+g({{\mathbf{d}}_{\text{best}}}){{\mathbf{d}}_{best}})$ within all $2^l$ subsets. Thus, Eq.~(\ref{eq:theo2_proof1}) can be approximated as:
     \begin{equation}\label{eq:theo2_proof1_2}
    \begin{gathered}
\Delta\propto\mathcal{L}(\mathbf{x}+g({{\mathbf{d}}_\text{best}}){{\mathbf{d}}^{(l)}})-\mathcal{L}(\mathbf{x}+g({{\mathbf{d}}_\text{best}}){{\mathbf{d}}_{best}}),
        \end{gathered}
    \end{equation} where $\Delta$ and ${{\mathbf{d}}^{(l)}}$ are the abbreviation of $\Delta_{B_{ks}^{(l)}}$ and ${{\mathbf{d}}_{ks}^{(l)}}$. Consequently, when HRayS approximately greedily searches for the direction $\mathbf{d}_{ks}^{(l)}$ that maximizes the CE loss value $\mathcal{L}(\mathbf{x}+g(\mathbf{d}_{\text{best}}){{\mathbf{d}}_{ks}^{(l)}})$. This provides an intuitive explanation of Theorem \ref{theo:raysappro}. We now proceed with a formal proof.

 Recalling Definition \ref{def:def1_app} and Lemma \ref{lem:lem1_app}, at each level $l$, HRayS retains the set of dimensions within each subset that exhibit the maximal sum of derivatives $I(B_{ks}^{(l)})$. This dimension set is progressively refined across rounds, ultimately yielding a selected coordinate set $S\subset\{1,\cdots,d\}$ that tends to capture dimensions with large gradient magnitudes. Consequently, we can construct a $\pm1$ sign vector $\hat{\mathbf{d}}^*$ where each element
         \begin{equation}\label{eq:theo2_proof4}
    \begin{gathered}
\hat{\mathbf{d}}^*[i]\approx\begin{cases}
    \text{sgn}(\frac{\partial \mathcal{L}(\mathbf{x},y)}{\partial\mathbf{x}[i]}) & \text{if}\ i\in S,\\
    1 \ \text{or} \ -1 & \text{otherwise},
\end{cases}
        \end{gathered}
\end{equation}
 whose alignment with the true gradient's sign vector $\text{sgn}\nabla\mathcal{L}(\mathbf{x},y)$ (shortened as $\text{sgn}\nabla\mathcal{L}(\mathbf{x})$) can be measured by their inner product.
 This inner product can be expanded as follows:
          \begin{equation}\label{eq:theo2_proof5}
    \begin{gathered}
    \text{sgn}\nabla \mathcal{L}(\mathbf{x})^\top\hat{\mathbf{d}}^*=\sum_{i\in S}\text{sgn}\Big(\frac{\partial \mathcal{L}(\mathbf{x})}{\partial\mathbf{x}[i]}\Big)\hat{\mathbf{d}}^*[i]+ \\ \sum_{i\notin S}\text{sgn}\Big(\frac{\partial \mathcal{L}(\mathbf{x})}{\partial\mathbf{x}[i]}\Big)\hat{\mathbf{d}}^*[i].
        \end{gathered}
\end{equation}Taking expectation over the randomness of the unselected coordinates,
the second term vanishes. Recalling Eqs.~(\ref{eq:theo2_proof2})
and (\ref{eq:theo2_proof3}), we have
\begin{equation}\label{eq:theo2_proof6}
\begin{aligned}
\mathrm{sgn}(\nabla \mathcal{L}(\mathbf{x},y))^\top\hat{\mathbf{d}}^*
&=
\sum_{i\in S} 1 \\
&\ge
\frac{1}{G_{\max}}
\sum_{i\in S}
\left|
\frac{\partial \mathcal{L}(\mathbf{x},y)}{\partial \mathbf{x}[i]}
\right| \\
&\ge
\frac{1}{G_{\max}}
\left(
\|\nabla \mathcal{L}(\mathbf{x},y)\|_1
-
\sum_{l=1}^{\lceil\log_2 d\rceil}\eta_l
\right).
\end{aligned}
\end{equation}
Rearranging the right-hand side of this inequality yields the following expression:
 \begin{equation}\label{eq:theo2_proof7}
\mathrm{sgn}(\nabla \mathcal{L}(\mathbf{x},y))^\top\hat{\mathbf{d}}^*
\;\ge\;
\frac{\|\nabla \mathcal{L}(\mathbf{x},y)\|_1}{G_{\max}}
\cdot
(1-\varpi),
\end{equation}
where $\varpi=\sum_{l=1}^{\lceil\log_2 d\rceil}\eta_l||\nabla \mathcal{L}(\mathbf{x},y)||_1^{-1}$. Hence, it completes the proof.
\end{proof}

\section{Discussion on Assumption \ref{ass1_2}}\label{app:ass1_2}
In this section, we justify Assumption \ref{ass1_2} by leveraging the local similarity of gradient components. For vectors with $\{\pm1\}$ entries, we have $\left\|\mathbf{d}_B-\mathbf{d}\right\|_1=2|B|$, where $|B|$ denotes the cardinality of subset $B$. Therefore, Eq.~(\ref{eq:lip2}) can be rewritten as:
         \begin{equation}\label{eq:lip3}
    \begin{gathered}
    |\mathcal{L}(\mathbf{x}+g(\mathbf{d}){{\mathbf{d}}_B})-\mathcal{L}(\mathbf{x}+g(\mathbf{d}){{\mathbf{d}}})|\leq 2R|B|.
        \end{gathered}
    \end{equation} 
Since $g(\mathbf{d})$ is fixed once $\mathbf{d}$ is determined,
we set $g(\mathbf{d})=1$ for notational simplicity. Following~\cite{ilyas2019prior},
we assume that the gradient
$\nabla \mathcal{L}(\mathbf{x},y)$ exhibits \emph{spatial local similarity},
meaning that there exists a constant $\gamma>0$ such that for any subset
$B$ consisting of consecutive indices,
\begin{equation}\label{eq:gradsim}
\left|
\frac{\partial \mathcal{L}(\tilde{\mathbf{x}},y)}{\partial \tilde{\mathbf{x}}[i]}
-
\frac{\partial \mathcal{L}(\tilde{\mathbf{x}},y)}{\partial \tilde{\mathbf{x}}[j]}
\right|
\le \gamma,
\quad \forall i,j\in B,
\end{equation}
where $\tilde{\mathbf{x}}=\mathbf{x}+\mathbf{d}$. Let
\begin{equation}
s := \frac{1}{|B|}\sum_{i\in B}
\frac{\partial \mathcal{L}(\tilde{\mathbf{x}},y)}{\partial \tilde{\mathbf{x}}[i]}
\end{equation}
denote the average gradient value over subset $B$.
By Eq.~(\ref{eq:gradsim}), each gradient component in $B$ satisfies
\begin{equation}\label{eq:grad_bound}
\left|
\frac{\partial \mathcal{L}(\tilde{\mathbf{x}},y)}{\partial \tilde{\mathbf{x}}[i]}
\right|
\le |s|+\gamma,
\quad \forall i\in B.
\end{equation}
Using the first-order Taylor expansion,
\begin{equation}\label{eq:taylor}
\mathcal{L}(\mathbf{x}+\mathbf{d}_B)
-\mathcal{L}(\mathbf{x}+\mathbf{d})
\approx
(\nabla \mathcal{L}(\tilde{\mathbf{x}},y))^\top(\mathbf{d}_B-\mathbf{d}),
\end{equation}
and noting that $\mathbf{d}_B[i]-\mathbf{d}[i]=-2\mathbf{d}[i]$ for $i\in B$,
we obtain
\begin{equation}\label{eq:taylor2}
\big|\mathcal{L}(\mathbf{x}+\mathbf{d}_B)
-\mathcal{L}(\mathbf{x}+\mathbf{d})\big|
\le
2\sum_{i\in B}
\left|
\frac{\partial \mathcal{L}(\tilde{\mathbf{x}},y)}{\partial \tilde{\mathbf{x}}[i]}
\right|.
\end{equation}
Applying Eq.~(\ref{eq:grad_bound}) yields
\begin{equation}
\big|\mathcal{L}(\mathbf{x}+\mathbf{d}_B)
-\mathcal{L}(\mathbf{x}+\mathbf{d})\big|
\le 2(|s|+\gamma)|B|.
\end{equation}
Defining $R := |s|+\gamma$ completes the justification of
Assumption~\ref{ass1_2}.

\begin{table}[t!]
\footnotesize\tabcolsep=0.02cm
  \centering
  \caption{The ablation study on $\mathbf{d_n}$, $\phi(\hat{\mathbf{d}}_\mathbf{r})$ and PDO. Ours$_\text{opt}$ serves as the baseline ASR; Increased: \textcolor{SeaGreen}{\textbf{green}} background; Decreased: \textcolor{gray}{\textbf{gray}} background.}
  \scalebox{0.9}{
    \begin{tabular}{c|c|ccc|c|c|ccc}
    \Xhline{0.8pt}
    Datasets & \multicolumn{4}{c|}{  ImageNet (ViT-B-32)} & \multicolumn{5}{c}{ CIFAR-10 (VGG-16-BN )} \\
    \hline
    Max$Q$  & \textbf{Ours$_\text{opt}$}  & 
    w/ $\mathbf{d_n}$ & w/ $\phi(\hat{\mathbf{d}}_\mathbf{r})$    & w/o P.  & Max$Q$  & \textbf{Ours$_\text{opt}$} & w/ $\mathbf{d_n}$ & w/ $\phi(\hat{\mathbf{d}}_\mathbf{r})$    & w/o P. \\
    \hline
    5     & 8.0   & 0.0   & \cellcolor{SeaGreen}{0.5}   & 0.0    & 5     & 4.0   & 0.0   & \cellcolor{lightgray}{-4.0} & 0.0   \\
    
    10    & 18.5  & 0.0   & \cellcolor{lightgray}{-4.5} & 0.0     & 10    & 14.5  & 0.0   & \cellcolor{lightgray}{-10.5} & 0.0     \\
    
    50    & 30.0  & \cellcolor{lightgray}{-0.5}  & \cellcolor{lightgray}{-5.0} & \cellcolor{lightgray}{-0.5}     & 50    & 30.0  & \cellcolor{lightgray}{-2.5}   & \cellcolor{lightgray}{-13.0} & \cellcolor{lightgray}{-4.5}    \\
    
    80    & 41.5  & \cellcolor{lightgray}{-0.5}  & \cellcolor{lightgray}{-6.0} & \cellcolor{lightgray}{-2.5}      & 80    & 45.5  & \cellcolor{lightgray}{-3.5}   & \cellcolor{lightgray}{-14.0} & \cellcolor{lightgray}{-1.5}    \\
    
    100   & 44.5  & \cellcolor{SeaGreen}{4.0}   & \cellcolor{lightgray}{-4.5} & 0.0     & 100   & 49.5  & 0.0   & \cellcolor{lightgray}{-12.5} & \cellcolor{SeaGreen}{2.5}   \\
    
    200   & 64.5  & \cellcolor{SeaGreen}{4.0}   & \cellcolor{lightgray}{-5.5} & \cellcolor{lightgray}{-6.5}  & 150   & 62.0  & \cellcolor{lightgray}{-2.0} & \cellcolor{lightgray}{-13.0} & 0.0  \\
    
    300   & 74.5  & \cellcolor{SeaGreen}{1.5}   & \cellcolor{lightgray}{-5.0} & \cellcolor{lightgray}{-11.0}  & 200   & 76.0  & \cellcolor{lightgray}{-3.0} & \cellcolor{lightgray}{-16.0} & \cellcolor{lightgray}{-4.0}    \\
    
    400   & 82.0  & \cellcolor{lightgray}{-1.0} & \cellcolor{lightgray}{-6.0} & \cellcolor{lightgray}{-10.0}  & 250   & 80.0  & \cellcolor{lightgray}{-1.5} & \cellcolor{lightgray}{-13.0} & \cellcolor{lightgray}{-2.5}  \\
    
    500   & 85.5  & 0.0 & \cellcolor{lightgray}{-7.5} & \cellcolor{lightgray}{-8.5}  & 300   & 82.5  & \cellcolor{lightgray}{-0.5} & \cellcolor{lightgray}{-7.0} & \cellcolor{lightgray}{-1.5}    \\
    
    1k    & 89.5  & \cellcolor{SeaGreen}{3.0}   & 0.0 & \cellcolor{lightgray}{-5.0}  & 400   & 87.5  & \cellcolor{lightgray}{-1.0} & \cellcolor{lightgray}{-7.0} & \cellcolor{SeaGreen}{1.0}  \\
    
    3k    & 98.0  & 0.0   & \cellcolor{lightgray}{-1.0} & \cellcolor{lightgray}{-3.5}  & 500   & 92.5  & \cellcolor{lightgray}{-3.5} & \cellcolor{lightgray}{-8.0} & \cellcolor{lightgray}{-3.0}  \\
    
    5k    &  99.5 & 0.0   & \cellcolor{lightgray}{-0.5}    & \cellcolor{lightgray}{-2.5}  & 1k    & 98.5  & \cellcolor{lightgray}{-2.5} & \cellcolor{lightgray}{-5.0} & \cellcolor{lightgray}{-2.5}  \\
    
    10k    & 100.0  & 0.0   & \cellcolor{lightgray}{-0.5} & \cellcolor{lightgray}{-0.5} & 2k    & 100.0  & \cellcolor{lightgray}{-1.5} & \cellcolor{lightgray}{-1.0} & \cellcolor{lightgray}{-1.5} \\
    \Xhline{0.8pt}
    \end{tabular}%
    }
  \label{tab:ablacitandvgg}%
\end{table}%

\begin{figure*}[t!]
  \includegraphics[width=\textwidth]{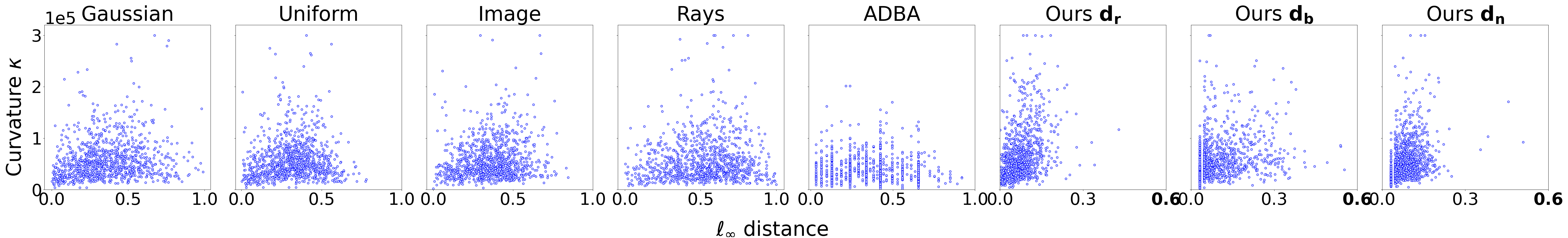}
  \caption{The relationship between boundary distance ($\ell_\infty$) and curvature $\kappa$ across different initializations.}  \label{fig:startpointscatter}
\end{figure*}

\begin{figure*}[t!]
  \includegraphics[width=\textwidth]{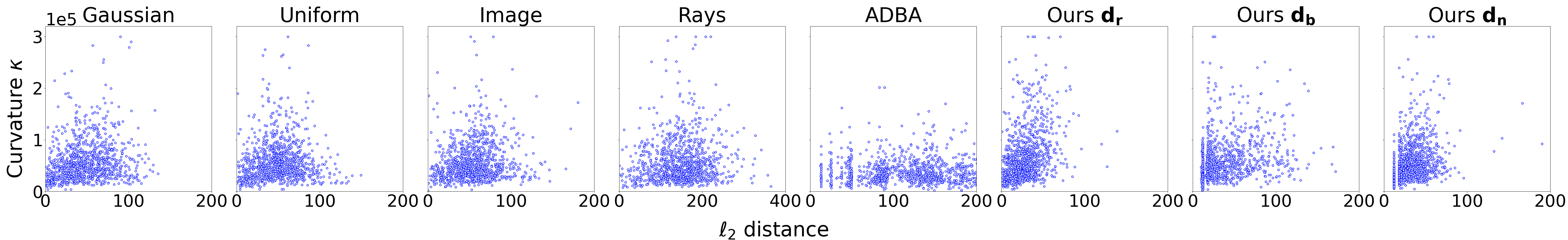}
  \caption{The relationship between boundary distance ($\ell_2$) and curvature $\kappa$ across different initializations.}
  \label{fig:startpointscatter_l2}
\end{figure*}

\section{Additional Ablation Studies}\label{app:ablation}

\subsection{Effect of Frequency Sampling Priors}\label{app:freqprior}
To validate the robustness of our universal frequency-variance curve, we compare it against two model-specific strategies: (1) \textbf{Peak-Only Strategy}, which perturbs only the single optimal DCT channel (the peak of the BFS curve) for the target model; and (2) \textbf{Model-Tailored Strategy}, which employs the exact, model-specific BFS curve to generate the perturbation direction $\mathbf{d_n}$ (see generation procedures in Sec.~\ref{sec:dirs}). Results indicate a clear trade-off between specificity and transferability. While the Peak-Only strategy boosts ASR on its matched model (e.g., ResNet-50: $+8.5\%/+5\%$ at Max.Q=$50/500$), it suffers a severe performance drop on unmatched architectures (e.g., ViT:  $-21\%$ drop). Similarly, the Model-Tailored strategy degrades performance on unseen models (ViT: $-8\%/-5\%$). This confirms that our balanced strategy effectively mitigates the risk of overfitting to specific spectral signatures, ensuring robust performance across diverse, unknown architectures.

\subsection{The contribution of Directions $\mathbf{d_n}$ and $\mathbf{d_a}$}
To further clarify the individual contributions of $\mathbf{d_n}$ and $\mathbf{d_a}=\phi(\hat{\mathbf{d}}_\mathbf{r})$—both designed to improve initial direction in DDM—we introduce two new attack variants that apply only one of them during initialization, keeping all other settings unchanged. These variants, denoted as w/ $\mathbf{d_n}$ and w/ $\mathbf{d_a}$, are evaluated in terms of ASR, as shown in Table \ref{tab:ablacitandvgg}. We perform the ablation study using Ours$_\text{opt}$ (i.e. fixed optimal block size with DDM) as the baseline. Query counts are omitted due to their comparable performance. We can observe that, using only $\mathbf{d_n}$ without $\mathbf{d_a}$ even leads to an improvement in ASR for the ViT-B-32. This may be attributed to the high compatibility between the frequency-domain sensitivity of the ViT architecture and the frequency-based sampling strategy of $\mathbf{d_n}$. In contrast, for lower-resolution datasets such as CIFAR-10, incorporating $\mathbf{d_a}$ is crucial for improving the ASR. But when using only $\mathbf{d_a}$ without $\mathbf{d_n}$, the ASR drops significantly across models. These observations indicate that $\mathbf{d_n}$ plays a critical role in enhancing the overall ASR, while $\mathbf{d_a}$ which focuses on low-frequency components, serves to balance performance across diverse black-box models, without substantially compromising the ASR on models where $\mathbf{d_n}$ alone already performs well.

\section{Initialization of Direction $\mathbf{d_r}$}\label{app:dr}
The process for initializing $\mathbf{d_r}$ is formalized as follows:
let the image height and width be $H$ and $W$, the side length of color squares be $h$. The image with randomly colored squares can be formally represented as:
\begin{equation}\label{eq:rcolor1}
 \begin{gathered}
\mathbf{x}^\prime_r = \left( \mathbf{a}_{ij} \right)_{i=1,j=1}^{H,W} \quad \text{where} \quad 
\mathbf{a}_{i,j} = \mathbf{c}_{\lfloor i/h \rfloor, \lfloor i/h \rfloor}.
\end{gathered}
\end{equation}
Spatial blocking is achieved through $\lfloor i/h\rfloor$ and $\lfloor j/h\rfloor$. The color square matrix $\mathbf{c}$ is defined as:
\begin{equation}\label{eq:rcolor2}
 \begin{gathered}
\mathbf{c} \in \mathbb{Z}^{K \times L \times 3} \quad (K = \lceil H/h \rceil, L = \lceil W/h \rceil).
\end{gathered}
\end{equation}
Each element $\mathbf{c}_{k,l}(R_{k,l},G_{k,l},B_{k,l})$ represents the RGB values of the block at position $(k,l)$. Color components satisfy
\begin{equation}\label{eq:rcolor3}
 \begin{gathered}
R_{k,l}, G_{k,l}, B_{k,l} \in \{0,1,\dots,255\} \\ \text{with} \quad \forall k,l: \mathbf{C}_{k,l} \sim U(\{0,1,\dots,255\}^3).
\end{gathered}
\end{equation}
Afterwards, we can obtain an image with randomly colored squares by performing the pixel mapping rule as follows:
\begin{equation}\label{eq:rcolor4}
 \begin{gathered}
\mathbf{a}_{i,j} = 
\begin{cases}
\mathbf{c}_{k,l}, & \begin{aligned}
&x \in [kh+1, (k+1)h] \\
&\land y \in [lh+1, (l+1)h]
\end{aligned} \\
\text{Boundary}, & \text{otherwise}
\end{cases}
\end{gathered}
\end{equation}
When $H$ and $W$ are not integer multiples of $h$, boundary handling is performed via replication padding.

\begin{figure*}[t!]
  \includegraphics[width=\textwidth]{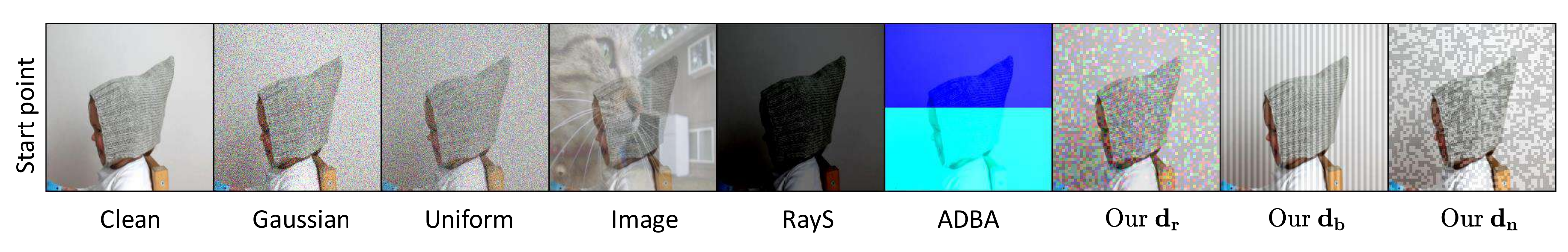}
  \caption{The illustration of AEs (dubbed start point) along different initial directions under $\ell_\infty$ contraint.}
  \label{fig:initialae}
\end{figure*}

\section{Analysis on Direction Initialization Methods}\label{appsec:initres}
We present the scatter plots comparing the $\ell_\infty$ or $\ell_2$ perturbation distances and curvature variations of start points generated by different methods (as shown in Figs.~\ref{fig:startpointscatter} and ~\ref{fig:startpointscatter_l2}). As observed, our proposed initialization directions $\mathbf{d_r}$ and $\mathbf{d_n}$ exhibit the tightest clustering towards the left side of the plot, while other methods show more dispersed and larger distances. This demonstrates that our $\mathbf{d_r}$ and $\mathbf{d_n}$ initialization directions can identify start points with greater local curvature under smaller adversarial perturbations. More examples of start points generated by different kinds of perturbation direction initialization methods are given in Fig.~\ref{fig:initialae}.

\section{Proof of Theorem~\ref{theo:ddm}}
\label{app:theo3_proof}

\subsection{Setup and Notation}

Let $U \in \mathbb{R}^{d\times d}$ denote the orthonormal BDCT basis
($U^\top U = I$), and let the BDCT coefficients be partitioned into
$C$ frequency channels. 
Let $P_q$ denote the orthogonal projector onto the subspace corresponding
to channel $q$. Let the initialization direction be generated as $\mathbf{x}_0 = U \mathbf{z}$, where $\mathbf{z} \in \mathbb{R}^d$ is a zero-mean Gaussian vector whose
covariance is block-diagonal across frequency channels: $\mathbf{z} \sim \mathcal{N}(0, \Sigma_v)$, $\Sigma_v = \sum_{q=1}^Q v_q P_q$, where $v_q \ge 0$ denotes the scalar variance assigned to frequency channel $q$. Equivalently, the components of $\mathbf{z}$ restricted to each channel
$q$ are independently sampled as
$\mathbf{z}_{|q} \sim \mathcal{N}(0, v_q I)$. 

Similarly, we introduce an auxiliary random vector in the BDCT domain,
$\tilde{\mathbf{h}} \in \mathbb{R}^d$,
which will later be endowed with specific statistical and semantic interpretations. Let $\mathbf{h} = U\tilde{\mathbf{h}}$, where
$\tilde{\mathbf{h}} \sim \mathcal{N}(0, \Sigma_s)$ with
$\Sigma_s = \sum_{q=1}^Q s_q P_q$, where $s_q \ge 0$ represents the variance proxy corresponding to the gradient energy in channel $q$. Let $\mathbf{g} = \nabla_{\mathbf{x}}\mathcal{L}(\mathbf{x}, y) \in \mathbb{R}^d$ be the true gradient in the pixel space, and $\mathbf{u} = \mathrm{sgn}(\mathbf{g}) \in \{-1, 1\}^d$ be its sign vector.
\subsection{Assumptions and Lemmas}
First, we present Lemma~\ref{lem:arcsine} to provide a quantitative link between
the Pearson correlation of two Gaussian variables and the expected
alignment of their sign vectors.
This result will serve as the key analytical tool for translating
second-order (covariance-level) alignment into first-order
(sign-level) alignment in Theorem~\ref{theo:ddm}.

\renewcommand{\thelemma}{3}
\begin{lemma}[Arcsine Law for Gaussian Signs \cite{1446497}]\label{lem:arcsine}
Let $(X, Y)$ be a pair of zero-mean jointly Gaussian random variables with correlation coefficient $\rho = \mathrm{Corr}(X, Y) \in (-1, 1)$. Then:
\begin{equation}
\mathbb{E}\big[\mathrm{sgn}(X)\,\mathrm{sgn}(Y)\big] = \frac{2}{\pi}\arcsin(\rho).
\end{equation}
By extension, for jointly Gaussian vectors $\mathbf{X}, \mathbf{Y} \in \mathbb{R}^d$ with coordinate-wise correlations $\rho_i = \mathrm{Corr}(X_i, Y_i)$, the expected normalized inner product of their signs is:
\begin{equation}
\frac{1}{d}\,\mathbb{E}\left[\langle \mathrm{sgn}(\mathbf{X}),\,\mathrm{sgn}(\mathbf{Y})\rangle\right] = \frac{1}{d}\sum_{i=1}^d \frac{2}{\pi}\arcsin(\rho_i).
\end{equation}
\end{lemma}

\begin{proof}
Using the identity $\mathrm{sgn}(t) = 2\mathbb{I}(t>0) - 1$, we have:
\begin{equation}
\mathbb{E}[\mathrm{sgn}(X)\mathrm{sgn}(Y)] = 4\mathbb{P}(X>0, Y>0) - 1.
\end{equation}
For bivariate normal variables, Sheppard's formula~\cite{Owen1956BivariateNormalTables,Rosenbaum1961TruncatedBivariateNormal} states that the orthant probability is $\mathbb{P}(X>0, Y>0) = \frac{1}{4} + \frac{1}{2\pi}\arcsin(\rho)$. Substituting this back yields the result. Linearity of expectation extends this to the vector case.
\end{proof}

To analyze the expected alignment induced by the proposed initialization,
we introduce a probabilistic model in the frequency domain.
The following assumption specifies the joint second-order structure
of the initialization and a surrogate gradient signal.

\begin{assumption}[Frequency-Variance Prior with Average Channel Correlation]
\label{asmp:prior}
We assume that, for each channel $q$, the corresponding components of
$\mathbf{z}$ and $\tilde{\mathbf{h}}$ admit a Pearson correlation
coefficient $\rho_q \in [-1,1]$.
While $\rho_q$ may vary in sign across channels, we assume a weak
aggregate bias:
\[
\mathbb{E}_q[\rho_q] > 0.
\]
\end{assumption}

Under the probabilistic model in Asmp.~\ref{asmp:prior}, the following lemma characterizes
how channel-wise correlations in the BDCT domain are transferred
to pixel-wise correlations.

\renewcommand{\thelemma}{4}
\begin{lemma}[Correlation Transfer]\label{lem:transfer}

For any pixel coordinate $i$, the Pearson correlation between
$(\mathbf{x}_0)_i$ and $(\mathbf{h})_i$ is given by
\begin{equation}
\rho_i
=
\frac{
\sum_{q=1}^Q \beta_{i,q}\,\rho_q \sqrt{v_q s_q}
}{
\sqrt{\sum_{q=1}^Q \beta_{i,q} v_q}\;
\sqrt{\sum_{q=1}^Q \beta_{i,q} s_q}
},
\end{equation}
where
$\beta_{i,q} := \|(U^\top \mathbf{e}_i)_{|q}\|_2^2$
denotes the energy contribution of frequency channel $q$ to pixel $i$.
\end{lemma}

\begin{proof}
Both $(\mathbf{x}_0)_i$ and $(\mathbf{h})_i$ are zero-mean Gaussian random variables,
as linear transformations of jointly Gaussian vectors.
Their covariance is
\begin{footnotesize}
\begin{equation}
\mathrm{Cov}\big((\mathbf{x}_0)_i,(\mathbf{h})_i\big)
=
\sum\nolimits_{q=1}^Q
\sum\nolimits_{k \in \mathcal{I}^q}
U_{ik}^2\,
\mathrm{Cov}(z_k,\tilde h_k).
\end{equation}
\end{footnotesize}
By assumption, for $k \in \mathcal{I}^q$,
$\mathrm{Cov}(z_k,\tilde h_k) = \rho_q \sqrt{v_q s_q}$,
which yields
\[
\mathrm{Cov}\big((\mathbf{x}_0)_i,(\mathbf{h})_i\big)
=
\sum_{q=1}^Q \beta_{i,q}\,\rho_q \sqrt{v_q s_q}.
\]
The variances satisfy
\[
\mathrm{Var}((\mathbf{x}_0)_i) = \sum_{q=1}^Q \beta_{i,q} v_q,
\qquad
\mathrm{Var}((\mathbf{h})_i) = \sum_{q=1}^Q \beta_{i,q} s_q.
\]
Normalizing the covariance by the standard deviations gives the stated expression.
\end{proof}

Lemma~\ref{lem:transfer} establishes correlation transfer between the
initialization and an auxiliary Gaussian vector.
To relate this auxiliary variable $\mathbf{h}$ to the true gradient of the loss $\nabla_{\mathbf{x}}\mathcal{L}$,
we introduce the following modeling assumption.

\begin{assumption}[Channel-Sensitivity Surrogate]
\label{asmp:sensitivity} Let the variance $s_q$ used for sampling $\mathbf{h}$ be obtained by our BFS analysis method, we assume that this $\mathbf{h}$ serves as a surrogate of the true gradient
$\mathbf{g} = \nabla_{\mathbf{x}}\mathcal{L}$,
in the sense that their sign vectors admit positive expected alignment:
\[
\frac{1}{d}\,\mathbb{E}\big[\langle \mathrm{sgn}(\mathbf{h}), \mathrm{sgn}(\mathbf{g}) \rangle\big] \ge 0.
\]
Accordingly, we denote the surrogate gradient $\mathbf{g}^\sharp := \mathbf{h}$.
\end{assumption}


\noindent\textit{Justification of Assumption~\ref{asmp:sensitivity}.}
We provide a justification showing that the surrogate
$\mathbf{g}^\sharp$ constructed via BFS induces non-negative expected sign
alignment with the true gradient $\mathbf{g}$.

\noindent\textit{Step 1: BFS as a monotone proxy for channel-wise gradient energy.}
For a frequency channel $q$, BFS evaluates
\[
\mathrm{BFS}_q
=
\mathbb{E}_{\boldsymbol{\delta}_q}\!\left[
\mathcal{L}(\mathbf{x}+\boldsymbol{\delta}_q,y)
\right],
\]
where $\boldsymbol{\delta}_q$ is a zero-mean random perturbation supported on
channel $q$.
By first-order Taylor expansion,
\[
\mathcal{L}(\mathbf{x}+\boldsymbol{\delta}_q,y)
=
\mathcal{L}(\mathbf{x},y)
+
\langle \nabla_{\mathbf{x}}\mathcal{L}, \boldsymbol{\delta}_q\rangle
+
o(\|\boldsymbol{\delta}_q\|).
\]
While the linear term vanishes in expectation, the magnitude of loss variation
induced by $\boldsymbol{\delta}_q$ is governed by the projected gradient energy
$\|(U^\top \mathbf{g})_{|q}\|_2$.
Consequently, $\mathrm{BFS}_q$ serves as a monotone proxy (up to a constant offset)
for the channel-wise gradient energy in expectation.

\noindent\textit{Step 2: Construction of an energy-aligned Gaussian surrogate.}
Let $\mathbf{s}=(s_q)$ be defined proportional to $\mathrm{BFS}_q$ and construct
$\tilde{\mathbf{g}} \sim \mathcal{N}(0,\Sigma_s)$ with
$\Sigma_s = \sum_q s_q P_q$.
Define the surrogate gradient $\mathbf{g}^\sharp = U\tilde{\mathbf{g}}$.
By construction, channels with larger true gradient energy
$\|(U^\top \mathbf{g})_{|q}\|_2^2$ are assigned larger variance $s_q$ in the
surrogate, implying second-order alignment between $\mathbf{g}^\sharp$ and
$\mathbf{g}$ in the BDCT domain.

\medskip
\noindent\textit{Step 3: From second-order alignment to expected sign alignment.}
Each pixel coordinate aggregates contributions from multiple frequency channels
under the orthonormal BDCT basis.
As a result, the pixel-wise covariance
$\mathrm{Cov}((\mathbf{g}^\sharp)_i,\mathbf{g}_i)$ is a non-negative weighted sum
of channel-wise energy alignments.
This induces a non-negative average Pearson correlation:
\[
\mathbb{E}_i\!\left[
\mathrm{Corr}\big((\mathbf{g}^\sharp)_i,\mathbf{g}_i\big)
\right] \ge 0.
\]
For jointly Gaussian variables, the Arcsine Law implies that expected sign
agreement is a monotone function of the Pearson correlation.
Applying this coordinate-wise yields
\[
\frac{1}{d}\,
\mathbb{E}\big[\langle \mathrm{sgn}(\mathbf{g}^\sharp),
\mathrm{sgn}(\mathbf{g})\rangle\big]
\ge 0,
\]
which establishes Assumption~\ref{asmp:sensitivity}.


\subsection{Proof of Theorem~\ref{theo:ddm}}\label{subsec:theorem3}
Combining the correlation transfer property (Lemma~\ref{lem:transfer}),
the sign-correlation relationship for Gaussian variables
(Lemma~\ref{lem:arcsine}), and the surrogate-gradient alignment
assumption (Assumption~\ref{asmp:sensitivity}),
we establish a positive expected alignment between the
proposed initialization and the true gradient sign.

\renewcommand{\thetheorem}{3}
\begin{tcolorbox}[colback=blue!5!white,colframe=black!75!black
,
boxsep=0pt,
  left=4pt,
  right=4pt,
  top=4pt,
  bottom=4pt,
  boxrule=0.7pt,
  arc=2pt
]
\begin{theorem}[Positive Expected Alignment under Frequency-Variance Prior]
\label{theo:ddm_app}
Let $\mathbf{x}_0 = U\mathbf{z}$ be the initialization generated under
Assumption~\ref{asmp:prior} with sign vector
$\mathbf{d_n} = \mathrm{sgn}(\mathbf{x}_0)$, and let
$\mathbf{u} = \mathrm{sgn}(\nabla_{\mathbf{x}}\mathcal{L})$ denote the true
gradient sign. Let $\mathbf{h}=U\tilde{\mathbf{h}}$ be the Gaussian surrogate
gradient signal, and for
$\rho_i = \mathrm{Corr}((\mathbf{x}_0)_i,(\mathbf{h})_i)$ define
$\gamma := d^{-1}\sum_{i=1}^d \frac{2}{\pi}\arcsin(\rho_i)$. If $\gamma>0$ and Assumption~\ref{asmp:sensitivity} holds with
$\varepsilon_{\mathrm{sur}}<\gamma$, then
\[
{d}^{-1}\mathbb{E}\big[\langle \mathbf{d_n}, \mathbf{u} \rangle\big]
\ge
\gamma-\varepsilon_{\mathrm{sur}}
>0.
\]
In contrast, for isotropic random initializations
$\mathbf{d}_{\mathrm{rand}}$ that are coordinate-wise symmetric and independent
of $\mathbf{u}$, we have
$d^{-1}\mathbb{E}\big[\langle \mathbf{d}_{\mathrm{rand}}, \mathbf{u} \rangle\big] = 0$.
\end{theorem}
\end{tcolorbox}

\begin{proof}
    
We now leverage the above preliminaries to prove Theorem~\ref{theo:ddm}.

\textbf{1. Expected Alignment of $\mathbf{d_n}$.}
Let $\mathbf{x}_0 = U\mathbf{z}$ be the pre-sign initialization and
$\mathbf{g}^\sharp = U\tilde{\mathbf{g}}$ be the gradient proxy defined under
Assumptions~\ref{asmp:prior} and~\ref{asmp:sensitivity}.
For each pixel coordinate $i$, denote
\[
\rho_i := \mathrm{Corr}\big((\mathbf{x}_0)_i,(\mathbf{g}^\sharp)_i\big),
\]
which admits the explicit form given in Lemma~\ref{lem:transfer}.

Under Assumption~\ref{asmp:prior}, the channel-wise correlations
$\{\rho_q\}$ have a positive expectation, which induces a positive bias in the
pixel-level correlations when aggregated across coordinates.
As a result, aggregating across pixel coordinates yields a positive bias in expectation,
namely,
\[
\frac{1}{d}\sum_{i=1}^d \mathbb{E}[\rho_i] > 0.
\]


Applying the Arcsine Law (Lemma~\ref{lem:arcsine}) coordinate-wise and using the
linearity of expectation, the expected sign alignment between
$\mathbf{d_n} = \mathrm{sgn}(\mathbf{x}_0)$ and
$\mathbf{u}^\sharp = \mathrm{sgn}(\mathbf{g}^\sharp)$ satisfies
\begin{equation}
\frac{1}{d}\mathbb{E}\big[\langle \mathbf{d_n}, \mathbf{u}^\sharp \rangle\big]
=
\frac{1}{d}\sum_{i=1}^d
\mathbb{E}\!\left[\frac{2}{\pi}\arcsin(\rho_i)\right]
:= \gamma
> 0,
\end{equation}
where the inequality follows from the monotonicity of $\arcsin(\cdot)$ on
$(-1,1)$ and the positive average correlation.

By Assumption~\ref{asmp:sensitivity}, $\mathbf{u}^\sharp$ serves as a proxy
for the true gradient sign $\mathbf{u}$ in expectation. We quantify the possible
surrogate-to-true-gradient mismatch by an error term
$\varepsilon_{\mathrm{sur}}\ge 0$, defined such that
\begin{equation}
\frac{1}{d}\mathbb{E}\big[\langle \mathbf{d_n}, \mathbf{u} \rangle\big]
\ge
\frac{1}{d}\mathbb{E}\big[\langle \mathbf{d_n}, \mathbf{u}^\sharp \rangle\big]
-
\varepsilon_{\mathrm{sur}}.
\end{equation}
Therefore, in the informative-surrogate regime
$\varepsilon_{\mathrm{sur}} < \gamma$, we obtain
\begin{equation}
\frac{1}{d}\mathbb{E}\big[\langle \mathbf{d_n}, \mathbf{u} \rangle\big]
\ge
\gamma-\varepsilon_{\mathrm{sur}}
=
\frac{1}{d}\sum_{i=1}^d
\mathbb{E}\!\left[\frac{2}{\pi}\arcsin(\rho_i)\right]
-\varepsilon_{\mathrm{sur}}
> 0.
\end{equation}

\textbf{2. Comparison with Random Baselines.}
We now consider several common random initialization strategies
$\mathbf{d}_{\mathrm{rand}}$ that are independent of $\mathbf{u}$:
\begin{itemize}[leftmargin=0.4cm]
    \item \textbf{Isotropic Noise (Gaussian or Uniform):}
    Let $\mathbf{d}_{\mathrm{rand}} = \mathrm{sgn}(\mathbf{z})$ with
    $\mathbf{z} \sim \mathcal{N}(0,I)$ or $\mathbf{z} \sim \mathcal{U}(-1,1)^d$.
    By symmetry and independence, each coordinate agrees with $\mathbf{u}$ with
    probability $1/2$, yielding
    $\mathbb{E}[\langle \mathbf{d}_{\mathrm{rand}},\mathbf{u}\rangle]=0$.

    \item \textbf{Constant Directions:}
    For $\mathbf{d}_{\mathrm{rand}}=\pm\mathbf{1}$, the expected alignment
    $\mathbb{E}[\sum_i u_i]$ vanishes under mild symmetry assumptions on the gradient. This symmetry is empirically substantiated by the zero-mean profile observed in Fig.~\ref{fig:distruegrad}, rendering the alignment negligible in practice.

    \item \textbf{Unrelated Clean Images:}
    When $\mathbf{d}_{\mathrm{rand}}$ is derived from an unrelated clean image,
    any residual correlation with $\mathbf{u}$ is bounded by a small constant
    $\varepsilon_{\mathrm{img}} \approx 0$ and lacks the systematic frequency
    alignment induced by Assumption~\ref{asmp:prior}.
\end{itemize}


Therefore, the proposed initialization $\mathbf{d_n}$ achieves positive expected
normalized alignment with the true gradient sign, whereas coordinate-wise
symmetric isotropic random baselines yield zero expected alignment. This
completes the proof.
\end{proof}

\begin{corollary}[Boundary Distance Reduction]\label{cor:boundary}
Under a local linear decision boundary model, the distance to the boundary $r(\mathbf{d})$ along direction $\mathbf{d}$ is monotonically decreasing with the gradient alignment $\langle \mathbf{d}, \mathbf{u} \rangle$. Consequently, Theorem~\ref{theo:ddm} implies that our initialization $\mathbf{d_n}$ yields a smaller expected initial distance compared to random baselines.
\end{corollary}

\begin{proof}
We consider a local linear approximation of the decision boundary near the clean input $\mathbf{x}$. Let the decision boundary be represented by the hyperplane:
\begin{equation}
    \mathcal{B} = \{\mathbf{z} \in \mathbb{R}^d \mid \mathbf{w}^\top \mathbf{z} + b = 0\},
\end{equation}
where $\mathbf{w}$ is the normal vector to the boundary (parallel to the true gradient $\nabla_{\mathbf{x}}\mathcal{L}$) and $b$ is the bias. The normalized gradient direction is $\mathbf{u} = \mathbf{w} / \|\mathbf{w}\|_2$.

Let $R$ be the orthogonal (shortest) distance from $\mathbf{x}$ to the boundary $\mathcal{B}$. For any search direction $\mathbf{d}$ (where $\|\mathbf{d}\|_2=1$), the distance $r(\mathbf{d})$ required to reach the boundary along $\mathbf{d}$ is determined by the intersection:
\begin{equation}
    \mathbf{w}^\top (\mathbf{x} + r(\mathbf{d})\mathbf{d}) + b = 0.
\end{equation}
Solving for $r(\mathbf{d})$, and noting that the orthogonal distance is $R = | \mathbf{w}^\top \mathbf{x} + b | / \|\mathbf{w}\|_2$, we obtain the geometric relationship:
\begin{equation}\label{eq:dist_geom}
    r(\mathbf{d}) = \frac{R}{\langle \mathbf{d}, \mathbf{u} \rangle},
\end{equation}
where $\langle \mathbf{d}, \mathbf{u} \rangle$ is the cosine similarity (alignment) between the search direction and the gradient.

From Eq.~(\ref{eq:dist_geom}), it is evident that for a fixed $R$, $r(\mathbf{d})$ is inversely proportional to the alignment $\langle \mathbf{d}, \mathbf{u} \rangle$. Thus, $r(\mathbf{d})$ is monotonically decreasing with respect to $\langle \mathbf{d}, \mathbf{u} \rangle$ for $\langle \mathbf{d}, \mathbf{u} \rangle > 0$.

Combining this with \textbf{Theorem~\ref{theo:ddm}}:
\begin{itemize}
    \item For our method: $\mathbb{E}[\langle \mathbf{d_n}, \mathbf{u} \rangle] > 0$.
    \item For random baselines: $\mathbb{E}[\langle \mathbf{d}_{\text{rand}}, \mathbf{u} \rangle] = 0$ (implying the search direction is orthogonal to the gradient on average, leading to a significantly larger expected boundary distance
under the linear approximation).
\end{itemize}
Therefore, the expected boundary distance for our initialization is smaller in expectation
under the local linear model: $\mathbb{E}[r(\mathbf{d_n})] < \mathbb{E}[r(\mathbf{d}_{\text{rand}})]$.
\end{proof}

\section{Evolution of Intermediate AEs}\label{app:randkchange}
We analyze the variations in perturbation magnitude $r$ and curvature $\kappa$ of intermediate AEs as the query increases. The experimental results of $1,000$ randomly sampled ImageNet validation images are drawn in Fig. \ref{fig:r}. It is easy to find that our method reduces $r$ much faster than ADBA all the time (consistently smaller mean values and standard deviation) on average, making our method find AEs under a given perturbation budget with fewer queries, leading to a higher ASR under the same query limit. As observed in Fig.~\ref{fig:r}(b), AEs crafted by our DPAttack locate at a higher average curvature region during the query process, leading to much faster reduction in $r$ especially at the early attack stage.

\begin{figure}[t!]
\centering
  \includegraphics[width=0.35\textwidth]{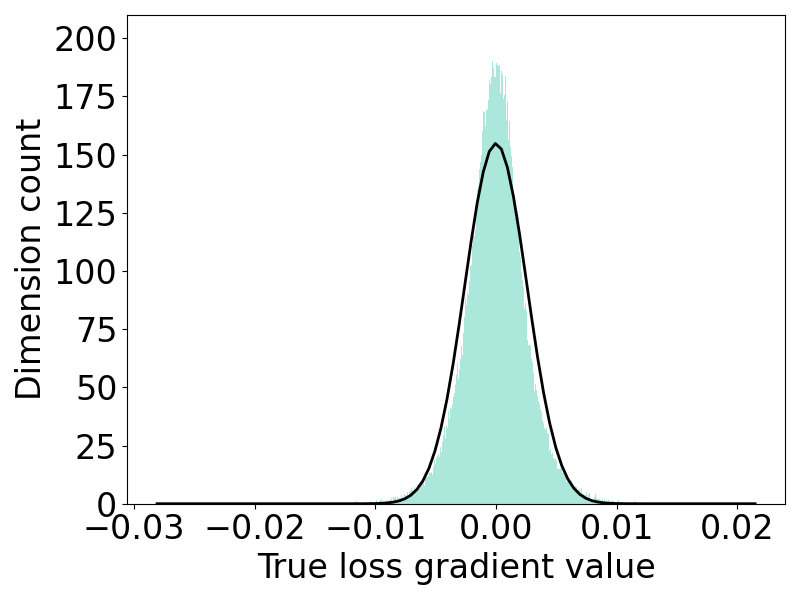}
  \caption{Distribution of the normalized CE loss gradient, $\nabla\mathcal{L}(\mathbf{x},y)/\|\nabla\mathcal{L}(\mathbf{x},y)\|_2$, averaged over 1,000 randomly sampled images from the ImageNet validation set. The distribution closely approximates a Gaussian.}
  \label{fig:distruegrad}
\end{figure}

\begin{figure}
\centering
  \includegraphics[width=\linewidth]{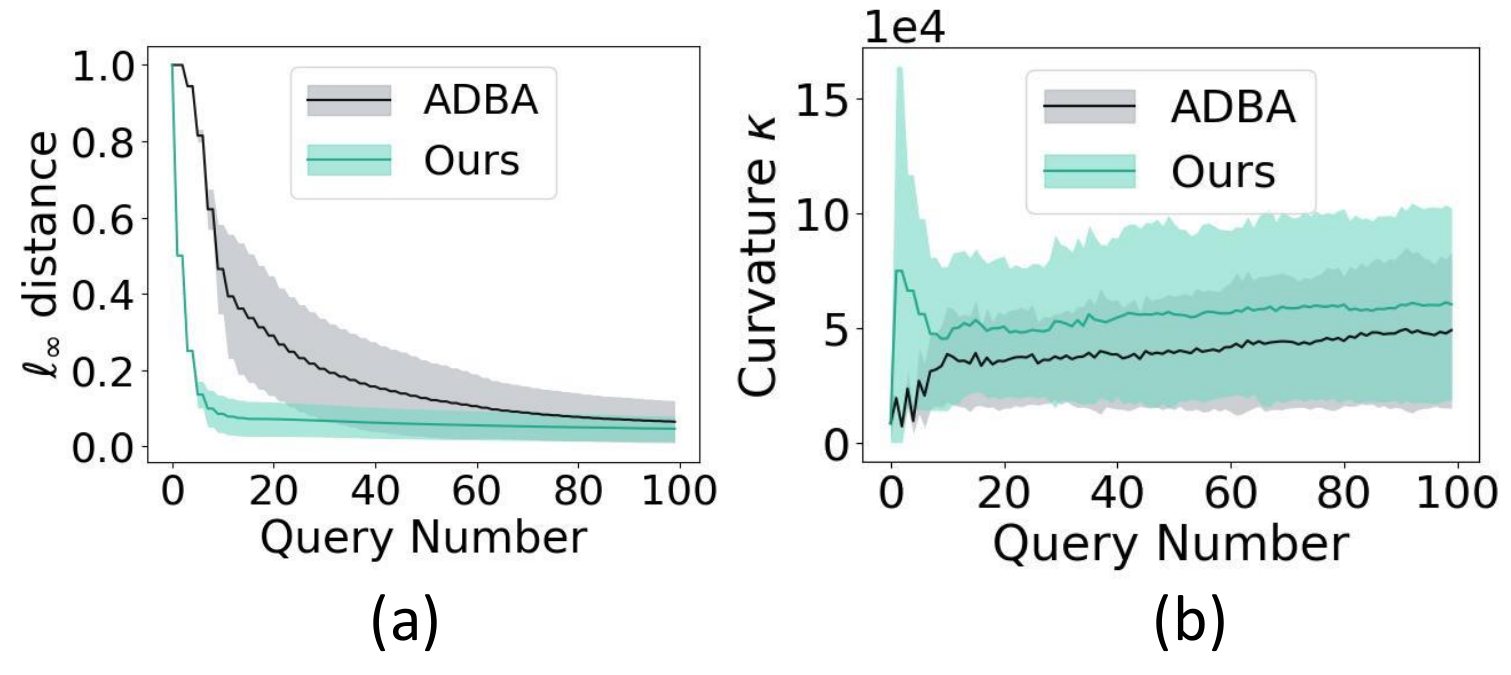}
  \caption{The comparative plot illustrates the variation in (a) the perturbation magnitude $r$ and (b) the curvature $\kappa$ as the number of queries increases.}
  \label{fig:r}
\end{figure}

\begin{algorithm}[t!]
\small
\caption{Dynamic Blocksize Selection (\textbf{DBS}) Process}
\label{alg:algorithmDBS}
\textbf{Input}: Original image $\mathbf{x}$, ground-truth label $y$, victim model $\mathcal{F}$, candidate block size set $\mathbb{W}$.\\
\textbf{Output}: Selected block size $w$, initial direction $\mathbf{d}^w_\mathbf{n}$, and current minimal decision boundary distance $r$.
\begin{algorithmic}[1]
\STATE Initialize candidate directions $\mathbf{D} \gets \{\mathbf{d}^w_\mathbf{n} \mid w \in \mathbb{W}\}$ where $\mathbf{d}^w_\mathbf{n}$ is the initialization direction for block size $w$.
\STATE Initialize current candidate set $\mathcal{W}_{curr} \gets \mathbb{W}$.
\STATE Initialize binary search bounds $l \gets 0, h \gets 1$.
\STATE Query count $Q \gets 0$, iteration $k \gets 0$.

\WHILE{$Q < Q_\text{max}$ \textbf{and} $k < k_\text{max}$}
    \STATE $r \gets (l + h) / 2$
    \STATE Initialize successful set $\mathcal{S} \gets \emptyset$.
    \FOR{each block size $w \in \mathcal{W}_{curr}$}
        \STATE Generate adversarial candidate: \\$\mathbf{x}' \gets \operatorname{clip}(\mathbf{x} + r \cdot \mathbf{d}^w_\mathbf{n}, [0, 1])$.
        \STATE Query victim model: $y'\gets \mathcal{F}(\mathbf{x}')$.
        \STATE $Q \gets Q + 1$.
        \IF{$y' \neq y$} 
            \STATE $\mathcal{S} \gets \mathcal{S} \cup \{w\}$
        \ENDIF
    \ENDFOR
    \IF{$\mathcal{S} = \emptyset$}
        \STATE $l \gets r$ \COMMENT{All candidates failed, increase magnitude}
    \ELSE
        \STATE $h \gets r$ \COMMENT{At least one succeeded, decrease magnitude}
        \STATE $\mathcal{W}_{curr} \gets \mathcal{S}$ \COMMENT{Prune unsuccessful block sizes}
    \ENDIF
    \STATE $k \gets k + 1$
\ENDWHILE

\STATE Select optimal block size $w^*$ uniformly at random from $\mathcal{W}_{curr}$.
\RETURN $w^*, \mathbf{d}^{w^*}_\mathbf{n}$, h.
\end{algorithmic}
\end{algorithm}

\section{The Curvature Estimation Method}\label{app:anainitd}
The curvature at a point on the decision boundary in high-dimensional space quantifies how sharply the boundary bends locally. Due to the computational complexity of calculating the curvature directly, in practice, the curvature can be approximated using the dominant eigenvalue $\lambda_{max}$ of the Hessian matrix $\mathbf{H}=[\frac{\partial^2\mathcal{L}(\mathbf{x},y)}{\partial\mathbf{x}[i]\partial\mathbf{x}[j]}]_{i,j=1}^{d}$ of the loss function $\mathcal{L}$ with respect to the input $\mathbf{x}$. The estimated curvature $\kappa$ is obtained via:
\begin{equation}\label{eq:r}
    \begin{gathered}
        \kappa\approx\lambda_{max}(\mathbf{H}).
    \end{gathered}
\end{equation}
In our implementation, we make use of the library PyHessian \cite{yao2020pyhessian}.

\section{The DBS Algorithm}\label{sec:DBS}

Since the block size in BDCT plays a critical role in the efficacy of the generated perturbation direction $\mathbf{d_n}$, we propose the Dynamic Blocksize Selection (DBS) algorithm. DBS efficiently identifies the optimal direction that minimizes the initial distance to the decision boundary for each specific image and victim model, providing a superior starting point for subsequent optimization. The key innovation of this algorithm is its parallel evaluation strategy: rather than performing a costly binary search for each candidate to find its exact boundary distance, it assesses all candidates simultaneously at a fixed perturbation magnitude (Lines 6--15 in Algorithm~\ref{alg:algorithmDBS}). This design drastically reduces the query overhead. To prevent excessive query consumption in edge cases where candidates exhibit nearly identical boundary distances, we stipulate a maximum comparison threshold $k_\text{max}$. If multiple candidates remain indistinguishable after $k_\text{max}$ iterations, a final direction is selected via uniform random sampling from the remaining candidate set. The complete procedure is detailed in Algorithm~\ref{alg:algorithmDBS}.

\section{The BiLiSearch Algorithm}\label{app:algo}
Conventional methods typically estimate decision boundary distances via independent binary searches for each direction. To achieve a standard precision of $10^{-3}$, such approaches require approximately $20$ queries per direction~\cite{10.1145/3394486.3403225}, which is prohibitively expensive under strict query budgets. We propose BiliSearch (Algorithm~\ref{alg:bilisearch}) to minimize this overhead through two key strategies: simultaneous evaluation and adaptive search logic.

The core of BiliSearch is a simultaneous comparison strategy that refines perturbation distances from coarse to fine granularity. Instead of independent estimations, BiliSearch evaluates two candidate directions at the same magnitude in each iteration. This allows for early termination: if one direction fails to yield an adversarial example at a given distance, the comparison can be pruned immediately without further refinement. Furthermore, BiLiSearch optimizes the search step based on the current state of the perturbation. Since the optimal distance $r$ is already significantly minimized during the DBS phase, relying solely on aggressive bisection often leads to frequent failures and redundant queries. To address this, we incorporate a line search mechanism that performs conservative, incremental reductions on $r$. This approach ensures a higher success rate and more efficient query utilization compared to standard bisection. As a result, BiliSearch reduces the average number of queries per comparison from $20$ to approximately $13$.

\section{Proof of Theorem~\ref{thm:main_text}}\label{sec:app_theo4}

\subsection{Setup and Notation}
Let $\mathbf{u} \in \{\pm 1\}^d$ denote the sign of the true gradient (i.e., $\mathbf{u} = \operatorname{sgn}(\nabla_{\mathbf{x}}\mathcal{L}(\mathbf{x}))$), and let $\mathbf{d} \in \{\pm 1\}^d$ denote the current perturbation direction.
For any index subset $S \subseteq [d]$, we define the \textit{signed correlation} $C_{\mathbf{d}}(S)$ and the \textit{normalized gain} $\Delta(S \mid \mathbf{d})$ as:
\begin{equation}\label{eq:theo4_def}
    \begin{gathered}
    C_{\mathbf{d}}(S) := \sum_{i \in S} \mathbf{u}_i \mathbf{d}_i, \quad
    \Delta(S \mid \mathbf{d}) := \frac{|C_{\mathbf{d}}(S)|}{d}.
    \end{gathered}
\end{equation}
The agreement metric (Hamming similarity) is defined as:
\begin{equation}\label{eq:agreement_def}
    A(\mathbf{d}, \mathbf{u}) := \frac{1}{d}\sum_{i=1}^d \mathbb{I}\{\mathbf{d}_i = \mathbf{u}_i\} = \frac{1}{2}\left(1 + \frac{1}{d}\langle \mathbf{d}, \mathbf{u} \rangle\right).
\end{equation}
Here, $\mathbb{I}$ is the indicator function. Flipping the signs of $\mathbf{d}$ on subset $S$ (i.e. $\mathbf{d}_i\leftarrow-\mathbf{d}_i,i\in S$) increases the agreement by exactly $\Delta(S \mid \mathbf{d})$ if and only if $C_{\mathbf{d}}(S) < 0$. Thus, identifying subsets with large $|C_{\mathbf{d}}(S)|$ is key to rapid optimization. A large magnitude implies high informative value: if positive, the current alignment is preserved; if negative, flipping the signs yields an immediate and significant improvement.

\textbf{Tree Definitions.} Let $\mathbf{d}_0$ be the initialization direction.
\begin{itemize}
    \item $\mathsf{T}_{\mathrm{pat}}$ ({Pattern-driven tree}): The leaves are the \textit{runs} (maximal constant-sign contiguous segments) of $\mathbf{d}_0$. Internal nodes are unions of adjacent runs. Splits occur \emph{only} at run boundaries.
    \item $\mathsf{T}_{\mathrm{dyad}}$ ({Dyadic tree}): A standard binary tree where splits occur at spatial midpoints, independent of $\mathbf{d}_0$.
\end{itemize}

\begin{algorithm}[t!]
\small
\caption{\textbf{BiLiSearch} Process}
\label{alg:bilisearch}
\textbf{Input}: Target classifier $\mathcal{F}$, clean RGB image $\mathbf{x}$, initiated directions $\mathbf{d}_\mathbf{n}, \mathbf{d}_\mathbf{a}$, GT label $y$, current minimal decision boundary distance $r$. \\
\textbf{Output}: Initialization  direction $\mathbf{d}_0$ and magnitude $r_0$.

\begin{algorithmic}[1]
\STATE Query budget $Q \gets 16$, line search steps $t \gets 5$.
\STATE Initialize bounds: $h_1, h_2 \gets r$; $l_1, l_2 \gets 0$.
\STATE Initialize magnitudes: $r_1 \gets r/2$, $r_2 \gets r/2$.
\STATE Initialize update modes: $m_1 \gets \text{Bisect}$, $m_2 \gets \text{Bisect}$.

\WHILE{$Q > 0$}
    \STATE Generate candidates: 
    \STATE \quad $\mathbf{x}_1 \gets \operatorname{clip}(\mathbf{x} + r_1 \cdot \mathbf{d}_\mathbf{n}, [0, 1])$
    \STATE \quad $\mathbf{x}_2 \gets \operatorname{clip}(\mathbf{x} + r_2 \cdot \mathbf{d}_\mathbf{a}, [0, 1])$
    \STATE Query model: $y'_1 \gets \mathcal{F}(\mathbf{x}_1)$, $y'_2 \gets \mathcal{F}(\mathbf{x}_2)$.
    
    \STATE \textit{// Update for direction } $\mathbf{d}_\mathbf{n}$
    \IF{$y'_1 \neq y$ \textbf{and} $m_1 = \text{Bisect}$}
        \STATE $l_1 \gets r_1$
        \STATE $r_1 \gets (h_1 + l_1)/2$ \COMMENT{Bisection update}
    \ELSE
        \STATE $m_1 \gets \text{Line}$
        \STATE $r_1 \gets h_1 - (h_1 - l_1)/t$ \COMMENT{Line update}
    \ENDIF

    \STATE \textit{// Update for direction } $\mathbf{d}_\mathbf{a}$
    \IF{$y'_2 \neq y$ \textbf{and} $m_2 = \text{Bisect}$}
        \STATE $l_2 \gets r_2$
        \STATE $r_2 \gets (h_2 + l_2)/2$ \COMMENT{Bisection update}
    \ELSE
        \STATE $m_2 \gets \text{Line}$
        \STATE $r_2 \gets h_2 - (h_2 - l_2)/t$ \COMMENT{Line update}
    \ENDIF
    \IF{($m_1 = \text{Line}$ \textbf{and} $y'_1 = y$) \textbf{or} ($m_2 = \text{Line}$ \textbf{and} $y'_2 = y$)}
        \STATE \textbf{break}
    \ENDIF

    \STATE $Q \gets Q - 2$
\ENDWHILE

\STATE \textit{// Select the direction with smaller perturbation}
\IF{$r_1 < r_2$}
    \RETURN $\mathbf{d}_0 \gets \mathbf{d}_\mathbf{n}, r_0 \gets r_1$
\ELSE
    \RETURN $\mathbf{d}_0 \gets \mathbf{d}_\mathbf{a}, r_0 \gets r_2$
\ENDIF
\end{algorithmic}
\end{algorithm}

\subsection{Assumptions and Lemmas}
\begin{assumption}[Block Sign-Coherence of True Gradient]\label{asmp:block}
The true gradient sign $\mathbf{u}$ exhibits a block structure. There exists a partition of $[d]$ into disjoint contiguous blocks $\mathcal{B}^* = \{B_1, \dots, B_K\}$ such that for each $B_k$, the sign is constant: $\mathbf{u}_i = \sigma_k \in \{\pm 1\}$ for all $i \in B_k$.
\end{assumption} 

\noindent\textit{Remark.} As evidenced by Fig.~\ref{fig:gradcossim}(a), the true gradient sign $\mathbf{u}$ exhibits distinct block-wise coherence over contiguous neighborhoods. Therefore, we regard this assumption as a data-supported condition rather than a theoretical idealization.

\begin{assumption}[Weak Informative Initialization]\label{asmp:purity}
The initialization $\mathbf{d}_0$ possesses a weak but positive expected alignment with the true gradient sign $\mathbf{u}$ within spatially coherent regions. Specifically, there exists a bias constant $\delta \in (0, 1/2]$ such that for any run $R$ of $\mathbf{d}_0$ fully contained within a coherent block of $\mathbf{u}$ (i.e., a valid run), the expected element-wise agreement satisfies:
\begin{equation}\label{eq:theo4_purity}
    \mathbb{E}\left[ \frac{1}{|R|} \sum_{i \in R} \mathbb{I}\{\mathbf{d}_{0,i} = \mathbf{u}_i\} \right] = \frac{1}{2} + \delta.
\end{equation}
This condition implies that $\mathbf{d}_0$ contains a statistically significant signal component, distinguishing it from pure random guessing (where $\delta=0$).
\end{assumption}

\noindent\textit{Remark.} 
This assumption abstracts the requirement for a ``warm start'' into two distinct properties:
(i) \textbf{Structural Coherence:} $\mathbf{d}_0$ possesses finite spatial correlation to form meaningful runs (ensuring the search space is composed of atomic units rather than fragmented noise); and
(ii) \textbf{Statistical Bias:} The initialization distribution be positively skewed towards $\mathbf{u}$.
Crucially, Eq.~\eqref{eq:theo4_purity} defines an expectation over the random sampling of $\mathbf{d}_0$. It does \textit{not} require every individual run in a single realization to be correctly aligned; it suffices that valid runs carry a positive predictive signal \textit{on average}.
Our proposed pattern-injected initialization naturally satisfies these conditions: the existence of the positive bias $\delta > 0$ is theoretically supported by Theorem~\ref{theo:ddm} (Frequency Sensitivity Alignment) and empirically evidenced by the positive shift in initial cosine similarity in Fig.~\ref{fig:gradcossim}(b).

We first establish a fundamental property of the signed correlation under partitioning,
which will be used to compare the effective gains offered by different tree structures.

\renewcommand{\thelemma}{5}
\begin{lemma}[Gain Cancellation Principle]\label{lem:gain}
For any contiguous set $S$, let $\{P_j\}$ be any partition of $S$ (i.e., $S = \bigcup_j P_j$, $P_j \cap P_k = \emptyset$). By the Triangle Inequality:
\begin{equation}\label{eq:triangle}
    |C_{\mathbf{d}_0}(S)| = \left| \sum_j C_{\mathbf{d}_0}(P_j) \right| \le \sum_j |C_{\mathbf{d}_0}(P_j)|.
\end{equation}
The inequality is strict (cancellation occurs) if the signs of $C_{\mathbf{d}_0}(P_j)$ are mixed. Specifically, if $S$ straddles a boundary where the correlation sign changes, the magnitude of the total correlation on $S$ is less than the sum of the magnitudes of its parts.
\end{lemma}

\begin{proof} Triangle Inequality.\end{proof}

Leveraging the gain cancellation principle above, we now compare the quality
of candidate nodes made available by the two tree constructions.

\renewcommand{\thelemma}{6}
\begin{lemma}[Availability Dominance]\label{lem:depth}
Let $\mathcal{C}_{\mathrm{pat}}$ be the set of nodes (runs) available at the leaves of $\mathsf{T}_{\mathrm{pat}}$. For any dyadic node $D \in \mathsf{T}_{\mathrm{dyad}}$, the gain is no larger than the aggregate gain of its pattern-aligned components.
Specifically,
\begin{equation}
    |C_{\mathbf{d}_0}(D)| \le \sum_{R \in \text{Runs}(D)} |C_{\mathbf{d}_0}(D \cap R)|,
\end{equation}
where $\text{Runs}(D)$ is the set of runs of $\mathbf{d}_0$ that intersect $D$.
\end{lemma}

\begin{proof}
Let $D$ be an arbitrary node in the dyadic tree. The interval $D$ can be partitioned into disjoint segments based on the runs of $\mathbf{d}_0$: $D = \bigcup_j (D \cap R_j)$. The signed correlation on $D$ is the sum of correlations on these segments.
By Lemma \ref{lem:gain} (taking $P_j = D \cap R_j$):
\[
|C_{\mathbf{d}_0}(D)| \le \sum_j |C_{\mathbf{d}_0}(D \cap R_j)|.
\]
In $\mathsf{T}_{\mathrm{pat}}$, the leaves are the full runs $R_j$. The algorithm selects nodes based on normalized gain $\Delta(S \mid \mathbf{d})$. If $D$ straddles a boundary between runs of opposite signs (which happens frequently in $\mathsf{T}_{\mathrm{dyad}}$ because it ignores $\mathbf{d}_0$'s structure), terms in the sum cancel out, strictly reducing the gain $|C_{\mathbf{d}_0}(D)|$. In contrast, the nodes $R_j$ in $\mathsf{T}_{\mathrm{pat}}$ are defined by the constant-sign regions of $\mathbf{d}_0$. By construction, they suffer zero cancellation from initialization sign changes.
While a run $R_j$ may still span a boundary of the true gradient $\mathbf{u}$, it is structurally superior to $D$, as the latter is susceptible to gain cancellation arising from \textit{both} $\mathbf{d}_0$ and $\mathbf{u}$. Thus, the search space of $\mathsf{T}_{\mathrm{pat}}$ offers candidates with higher non-cancelled gain potential than arbitrary dyadic intervals.
\end{proof}

\begin{table*}[t!]
\footnotesize
\tabcolsep=0.03cm
  \centering
 \caption{The untargeted attack performance on CIFAR-10 (VGG-16-BN) under $\ell_\infty=0.05$ constraint.} 
 \scalebox{0.8}{
 \begin{tabular}{c|cccccccc|cccccccc|cccccccc}
    \Xhline{0.8pt}
    Metrics & \multicolumn{8}{c|}{ASR}                      & \multicolumn{8}{c|}{Avg.Q}                    & \multicolumn{8}{c}{Med.Q} \\
    \hline
    Max$Q$  & HSJA  & BouR  & BouT &TtBA & HRayS  & ADBA  & \textbf{Ours$_\text{opt}$}& \textbf{Ours$_\text{dyn}$}  & HSJA  & BouR  & BouT &TtBA & HRayS  & ADBA  & \textbf{Ours$_\text{opt}$}& \textbf{Ours$_\text{dyn}$}  & HSJA  & BouR  & BouT &TtBA & HRayS  & ADBA  & \textbf{Ours$_\text{opt}$}& \textbf{Ours$_\text{dyn}$} \\
    \hline
    5     & 0.0   & 0.0   & 0.5 & 0.0 & \underline{1.5}   & 0.0   & \textbf{4.0}&0.0 &    -   &  -     &               5  &  - &            5  &     -  &               5  &-&     -  &     -  &               5 & -&               5  &    -   &              5 &- \\
   
    10    & 0.0   & 0.0   & \underline{2.5}  &0.0 & 1.5   & 0.0   & \textbf{14.5}  & \underline{2.5}&   -    &    -   &               7  &   -  &       10  &    -   &               6  & 9&   -   &    -   &               7 &- &             10  &   -    &              6&6  \\
   
    20    & 2.5   & 2.5   & 2.5 & 0.0 & {3.5}   & 2.0   & \underline{14.5} & \textbf{18.5} &             14  &             14  &               7  &  - &          19  &             17  &               6  & 14&            14  &             14  &               7  &-&             20  &            17  &              6 &14 \\
   
    50    & \underline{24.5}  & \underline{24.5}  & 14.0  &0.0& 8.0   & 13.0  & \textbf{30.0} & \textbf{30.0} &             44  &             44  &             33 & -&             39  &             35  &             25 &24 &             48  &             48  &             40  &    -   &      47  &            37  &            24  &18\\
   
    80    & 32.0  & {33.0}  & 25.0 & 4.0& 11.0  & 24.5  & \textbf{45.5}& \underline{41.0} &             45  &             47  &             49  &   71  &        49  &             47  &             39 &             35  &             49  &             50  &             41  & 71  &          49  &            49  &            37 &29 \\
   
    100   & 32.5  & {37.5}  & 22.5 &5.5 & 18.0  & 31.5  & \underline{49.5}& \textbf{50.5} &             45  &             47  &             44  &  73    &       69  &             57  &             43 &45 &             49  &             50  &             41  &   &          68  &            55  &            40&39  \\
   
    150   & 41.5  & 40.5  & 16.0  &8.5& 26.5  & {51.0}  & \underline{62.0}& \textbf{64.0}  &             82  &             81  &             64  &91 &            93  &             83  &             61 &61 &             89  &             89  &             79  &    72      &   98  &            85  &            52&56  \\
   
    200   & 44.0  & 41.5  & 27.0 & 9.0& 37.0  & {66.0}  & \textbf{76.0} & \underline{73.0} &             81  &             84  &           115  & 106  &        121  &           104  &             81 &74 &             89  &             89  &             82  &  107&         122  &          104  &            66 &65 \\
   
    250   & 45.0  & 45.0  & 27.0 & 15.5& 45.5  & {75.0}  & \underline{80.0} &\textbf{80.5}&             92  &             85  &             99  & 147  &        143  &           118  &             88 &             88  &             89  &             90  &             80  &141 &           150  &          117  &            69  &74\\
   
    300   & 45.5  & 45.0  & 28.0  &21.0 &57.0  &{79.0}  & \underline{82.5} &\textbf{85.5}&             89  &             85  &           106  & 185 &         173  &           125  &           94&           99  &             90  &             90  &             80  &214 &          179  &          120  &            72&82  \\
   
    400   & 42.0  & 42.0  & 29.0 & 32.5& 71.5  & {86.0}  & \underline{87.5}  & \textbf{90.5} &             87  &             84  &           120  &  259  &       212  &           144  &           109 &112 &             89  &             89  &             82  &298 &          209  &          127  &            79&87  \\
   
    500   & 44.0  & 43.0  & 29.5& 40.0 & 76.5  & {91.0}  & \textbf{92.5}& \underline{91.5}  &             83  &             86  &           115& 293 &           227  &           161  &           127 &116 &             89  &             90  &             81  &299&           217  &          132  &            83&87  \\
   
    1k  & 47.0  & 49.5  & 35.0 & 64.0& 91.0  & \underline{98.0} & \textbf{98.5}& \textbf{98.5}  &           137  &           133  &           239 &491 &           305  &           205  &           163&           161  &           140  &           139  &           133& 487 &           265  &          147  &          96  &          92 \\
   
    2k  & 48.5  & 48.0  & 37.0& 81.5 & 98.0  & \underline{99.5}  & \textbf{100.0}& {99.0} &           201  &           130  &           320 &657 &           387  &           220  &           178   &           166 &           140  &           139  &           132  &591&           288  &          147  &          107 &94 \\
   
    3k  & 52.0  & 47.5  & 40.0 &92.5 & \underline{99.5}  & \textbf{100.0}  & \textbf{100.0} & \textbf{100.0} &           180  &           136  &           346 & 873&           393  &           229  &           171 &188 &           139  &           139  &           134 & 698&           290  &          147  &          104 &98 \\
   
    \Xhline{0.8pt}
    \end{tabular}%
    }
  \label{tab:cifarvgg}%
\end{table*}%

\subsection{Proof of Theorem 4}

With the above assumptions ensuring block-level signal coherence
(Assumptions~\ref{asmp:block}–\ref{asmp:purity})
and Lemmas~\ref{lem:gain}–\ref{lem:depth} characterizing the structural
advantage of the pattern-driven candidate pool, we now establish
the per-query alignment dominance.

\renewcommand{\thetheorem}{4}
\begin{tcolorbox}[colback=blue!5!white,colframe=black!75!black
,
boxsep=0pt,
  left=4pt,
  right=4pt,
  top=4pt,
  bottom=4pt,
  boxrule=0.7pt,
  arc=2pt
]
\begin{theorem}[Per-query Sign Alignment Dominance]\label{thm:main_text_app}
Let $\mathbf{u} = \operatorname{sgn}(\nabla_{\mathbf{x}}\mathcal{L})$ be the true gradient sign, and let
$\mathbf{d}_t\in\{\pm1\}^d$ be the perturbation direction after $t$ queries.
We measure sign alignment using the agreement
$A(\mathbf{d}_t,\mathbf{u})
:=\frac{1}{d}\sum_{i=1}^d \mathbb{I}\{\mathbf{d}_{t,i}=\mathbf{u}_i\}
$. Under the assumptions of local block coherence (Asmp.~\ref{asmp:block}) and a weakly positive initial correlation (Asmp.~\ref{asmp:purity}), the expected agreement $A_t$ after $t$ queries satisfies:
\begin{equation}
\mathbb{E}\big[A_t^{\mathrm{pat}}\big] \ge \mathbb{E}\big[A_t^{\mathrm{dyad}}\big], \quad \forall t \in \mathbb{N},
\end{equation}
where $\mathrm{pat}$ and $\mathrm{dyad}$ denote the pattern-driven strategy and the blind dyadic baseline, respectively.
\end{theorem}
\end{tcolorbox}

\begin{proof}
The proof relies on the greedy nature of the selection policy and the superior quality of the candidate pool available in the pattern tree.

\begin{enumerate}
    \item \textbf{Decomposition:} The agreement at step $t$ is the cumulative sum of gains: $A_t = A_0 + \sum_{j=1}^t G_j$, where $G_j$ is the agreement improvement at the $j$-th query. Since $A_0$ is identical for both, it suffices to show $\mathbb{E}[G_j^{\mathrm{pat}}] \ge \mathbb{E}[G_j^{\mathrm{dyad}}]$.

    \item \textbf{Candidate Selection Policy:} Both methods use the same greedy rule. They simply look at all currently available nodes and pick the one that offers the largest possible gain $\Delta(S \mid \mathbf{d})$. 
    In practice, this is approximated by selecting the candidate that yields the maximum reduction in the distance to the decision boundary.
    Therefore, the improvement at step $j$ is determined by the best candidate found in the tree at that moment.

    \item \textbf{Structural Advantage:} Consider the composition of the candidate nodes.
    \begin{itemize}[leftmargin=0.1cm]
        \item \textbf{Pattern Tree:} The candidates are the \textit{runs} of $\mathbf{d}_0$. By definition, a run is a maximal segment of constant sign. Thus, pattern nodes are ``atomic'' w.r.t. $\mathbf{d}_0$ and suffer \textit{zero internal cancellation} from initialization sign changes.
        \item \textbf{Dyadic Tree:} The candidates are arbitrary spatial intervals. By Lemma \ref{lem:depth}, a dyadic node $D$ is structurally the summation of the pattern runs it covers. If $D$ covers runs of opposing signs (e.g., a positive region adjacent to a negative region), these components mathematically cancel each other out, reducing the magnitude of the total correlation: $|C(D)| < \sum |C(R_i)|$.
    \end{itemize}

\item \textbf{Availability of High-Gain Candidates (Max-Pooling Argument):} 
This difference in candidate exposure is driven by the handling of sign conflicts.
Since the algorithm is greedy, it selects the candidate with the largest instantaneous gain.

In $\mathsf{T}_{\mathrm{dyad}}$, a node often covers a region spanning multiple underlying signal runs.
Consider a dyadic node covering two adjacent runs with gains $g_1$ and $g_2$. 
Crucially, by the definition of a run, adjacent runs must have \textit{opposite signs} (otherwise they would merge into a single run).
Thus, the effective gain available to the dyadic strategy is subject to subtraction: $G_{\mathrm{dyad}} = |g_1 - g_2|$.

In contrast, $\mathsf{T}_{\mathrm{pat}}$ explicitly decomposes the search space into sign-consistent atomic runs.
The same region is exposed as separate candidates $g_1$ and $g_2$, allowing the greedy rule to select $G_{\mathrm{pat}} = \max(|g_1|, |g_2|)$.

Mathematically, for any two non-negative magnitudes, the maximum is always greater than or equal to their absolute difference:
\[
\max(|g_1|,|g_2|) \ge \big| |g_1| - |g_2| \big|.
\]
Consequently, whenever sign cancellation occurs in a dyadic node, $\mathsf{T}_{\mathrm{pat}}$ offers a candidate with larger gain. In regions with no sign changes, both trees offer equivalent candidates.
Therefore, the maximum gain attainable at each step under $\mathsf{T}_{\mathrm{pat}}$ is dominant:
\[ \mathbb{E}[G_j^{\mathrm{pat}}] \ge \mathbb{E}[G_j^{\mathrm{dyad}}]. \]



\item \textbf{Probabilistic Dominance of Boundaries:}
    While the Max-Pooling argument (Step 4) ensures that Pattern candidates minimize internal cancellation w.r.t. $\mathbf{d}_0$, success ultimately depends on alignment with the true gradient $\mathbf{u}$.
    Assumption~\ref{asmp:purity} ensures that the pattern boundaries are statistically informative about the structure of $\mathbf{u}$.

    \begin{itemize}[leftmargin=0.1cm]
        \item \textbf{Dyadic (Blind):} The fixed grid partitions are independent of the image content, meaning a dyadic node likely straddles the true boundaries of $\mathbf{u}$, leading to partially incorrect updates even if the node itself has high energy.
        \item \textbf{Pattern (Correlated):} The splits follow the boundaries of $\mathbf{d}_0$, which are spatially correlated with $\mathbf{u}$ (Assumption \ref{asmp:purity}). Thus, a Pattern node corresponds to a coherent region in the ground truth with high probability.
    \end{itemize}

    \item \textbf{Synthesis of Per-Step Dominance:} 
    The superiority of the proposed method arises from the conjunction of these two factors: 
    (1) {Structural Purity}: Pattern nodes avoid internal gain cancellation (Step 4); and 
    (2) {Semantic Alignment}: Pattern boundaries inherently respect the signal structure of $\mathbf{u}$ (Step 5).
    
    Combining the gain dominance from Max-Pooling with the higher probability of valid boundary alignment, we conclude that the expected improvement at each step is larger for the pattern-driven strategy:
    \[ \mathbb{E}[G_j^{\mathrm{pat}}] \ge \mathbb{E}[G_j^{\mathrm{dyad}}]. \]
    \item \textbf{Final Aggregation:} Summing the expected gains over $t$ steps yields the theorem statement.
\end{enumerate}
\end{proof}

\begin{table*}[t!]
\footnotesize
\tabcolsep=0.03cm
\centering
  \caption{The untargeted attack performance on CIFAR-10 (ResNet-18) under $\ell_\infty=0.05$ constraint.}
  \scalebox{0.8}{
    \begin{tabular}{c|cccccccc|cccccccc|cccccccc}
    \Xhline{0.8pt}
    Metrics & \multicolumn{8}{c|}{ASR}                      & \multicolumn{8}{c|}{Avg.Q}                    & \multicolumn{8}{c}{Med.Q} \\
    \hline
    Max$Q$  & \multicolumn{1}{c}{HSJA} & \multicolumn{1}{c}{BouR} & \multicolumn{1}{c}{BouT} & \multicolumn{1}{c}{TtBA}& \multicolumn{1}{c}{HRayS} & {ADBA} & \textbf{Ours$_\text{opt}$}& \textbf{Ours$_\text{dyn}$} & HSJA  & BouR  & BouT &TtBA & HRayS  & ADBA  & \textbf{Ours$_\text{opt}$}& \textbf{Ours$_\text{dyn}$}  & HSJA  & BouR  & BouT &TtBA & HRayS  & ADBA  & \textbf{Ours$_\text{opt}$}& \textbf{Ours$_\text{dyn}$} \\
    \hline
    5     &            0.0    &            0.0    &            0.0    & 0.0   &         \underline{0.5}  &             0.0    & \textbf{4.5} & 0.0 &   -    &     -  & -    & -  & 5     &   -    & 5&   -     &  -     &   -    &    -  &  -& 5     &    -   & 5  &    -  \\
    
    10    &            0.0    &            0.0   &          0.0   & 0.0&   {0.5}  &            0.0   & \textbf{16.5} &   \underline{2.0}&    -   &    -   &   -  &-   & 10    &   -    & 6  & 8   &  -     &    -   &   -    &- & 10    &    -   & 6& 8  \\
    
    20    &        {4.5}  &           {4.5}  &           1.5  & 0.0      &       1.0  &            0.0    & \underline{17.0} & \textbf{18.5} & 13    & 13    & 7      &-& 19    &    -   & 7& 14     & 13    & 13    & 5    &-  & 19    &    -   & 6& 14  \\
    
    50    &  {27.0}  &{27.0}  &         17.5  &        0.0   &   5.0  &          10.0  & \textbf{36.0}& \underline{31.0} & 41    & 41    & 37   & - & 44    & 38    & 25   & 24  & 48    & 48    & 40    &- & 48    & 39    & 23& 18  \\
    
    80    & {34.0}  &         31.5  &         23.5  &       14.5    & 11.5  &          22.5  & \underline{45.0}& \textbf{48.0} & 47    & 46    & 47   &71  & 59    & 53    & 33& 37    & 50    & 49    & 41    & 70 & 62    & 53    & 29& 30  \\
    
    100   &         37.0  & \underline{38.0}  &         27.5  &     13.5    &   13.5  &          32.5  & \underline{49.5}& \textbf{51.0} & 48    & 48    & 51   &70  & 74    & 64    & 38& 40    & 49    & 49    & 41    &69 & 79    & 66    & 31& 35  \\
    
    150   &         43.5  &         45.5  &         19.5  &     19.5   &    22.5  &   \underline{56.0}  & \textbf{61.5}& \textbf{61.5} & 81    & 82    & 76  &  93 & 97    & 88    & 56 & 55    & 89    & 89    & 80   &72  & 110   & 88    & 41& 50  \\
    
    200   &         41.0  &         46.0  &         26.0  &       19.5   &  33.0  &   {64.5}  & \underline{70.0}& \textbf{72.5} & 84    & 83    & 101    &95& 127   & 100   & 70  & 73   & 89    & 90    & 80    &72 & 122   & 100   & 48 & 56  \\
    
    250   &         42.5  &         48.0  &         26.5  &      31.5   &   44.5  &  {74.5}  & \textbf{79.5}& \underline{78.5} & 91    & 86    & 100  & 134 & 157   & 116   & 89& 84    & 89    & 89    & 80    &137 & 147   & 107   & 60 & 61  \\
    
    300   &         44.0  &         47.5  &         29.5  &       35.5  &   53.5  &  {78.5}  & \textbf{85.5}& \underline{83.0} & 86    & 90    & 119  & 153 & 179   & 126   & 102& 95   & 89    & 89    & 81   &139  & 191   & 113   & 72  & 63 \\
    
    400   &         45.5  &         41.5  &         32.5  &       47.0  &   63.5  & {87.0}  & \textbf{89.0} & \underline{88.5} & 89    & 87    & 110   &209 & 207   & 146   & 111 & 109  & 89    & 89    & 81     &214& 207   & 116   & 77  & 69\\
    
    500   &         46.5  &         47.0  &         30.0  &       57.0   &  75.0  &  \underline{93.0}  & \textbf{94.0} & {92.0} & 94    & 90    & 141  & 247 & 246   & 166   & 124  & 122  & 89    & 89    & 84  & 218  & 235   & 123   & 84 & 74 \\
    
    1k  &         49.5  &         53.0  &         36.0  &       77.5  &   93.5  &  \underline{97.0}  & \textbf{98.0}& \underline{97.0} & 147   & 145   & 222  & 379 & 332   & 191   & 152& 150   & 138   & 139  &   131 & 302 & 279   & 128   & 95   & 83\\
    
     4k  &  53.5        &   53.0     & 41.5     &  96.0      &\underline{99.5}    & \textbf{  100.0} & \textbf{ 100.0}& \textbf{ 100.0}& 393   & 140   & 340    &975& 411   & 214   & 183 & 192 & 140   & 139   & 134  & 774 & 289   & 130   & 106& 97 \\
    \Xhline{0.8pt}
    \end{tabular}%
    }
  \label{tab:cifar10res18}%
\end{table*}%

\section{Proof of Theorem~\ref{theo:complexity-fixed}}\label{sec:app_theo5}

\renewcommand{\thetheorem}{5}
\begin{tcolorbox}[colback=blue!5!white,colframe=black!75!black
,
boxsep=0pt,
  left=4pt,
  right=4pt,
  top=4pt,
  bottom=4pt,
  boxrule=0.7pt,
  arc=2pt
]
\begin{theorem}[Query Complexity under Block Sign-Coherence]\label{theo:complexity-fixed_app}
Under Assumptions.~\ref{asmp:block} and~\ref{asmp:purity}, let the true gradient sign $\mathbf{u}$ consist of $K$ spatially coherent blocks $\{B_k\}_{k=1}^K$. To identify descent directions aligned with $\mathbf{u}$, the expected query complexities for \emph{dyadic} search ($T_{\mathrm{dyad}}$) and \emph{pattern-driven} search ($T_{\mathrm{pat}}$) satisfy:
\begin{equation}
    T_{\mathrm{dyad}} = \Omega\left(\sum_{k=1}^K \log_2\frac{d}{|B_k|}\right), \quad T_{\mathrm{pat}} = O\left(\sum_{k=1}^K \gamma_k\right),
\end{equation}
where $\Omega(\cdot)$ and $O(\cdot)$ denote asymptotic lower and upper bounds. $\gamma_k$ is the number of $\mathbf{d}_0$-runs intersecting block $B_k$. In particular, the pattern-driven strategy is asymptotically more efficient in expectation when $\gamma_k < \log_2(d/|B_k|)$ on average.
\end{theorem}
\end{tcolorbox}

\begin{proof}
The proof compares how efficiently the two strategies expose non-cancelled
descent gain under block sign-coherence.

\paragraph{Setup.}
We refer to the signed correlation $C_{\mathbf{d}}(S)$ over a subset $S$ as its signed descent gain. Under Assumption~\ref{asmp:purity}, for any run $R \subseteq B_k$ fully contained
in a true block, the expected gain contribution is positive.

\paragraph{Dyadic Strategy.}
The dyadic strategy recursively partitions coordinates by index.
If a true block $B_k$ is fragmented into multiple dyadic cells at a given depth,
cells are likely to contain mixed signs of $\mathbf{u}$, which leads to partial cancellation of signed gain. In expectation, the magnitude of the gain is significantly reduced compared to a fully contained block. In the worst case, or in expectation over random block alignments, isolating a region fully contained in $B_k$ requires the tree to reach depth
$\Omega(\log_2(d/|B_k|))$
, yielding
\[
T_{\mathrm{dyad}} = \Omega\!\left(\sum_{k=1}^K \log_2(d/|B_k|)\right).
\]

\paragraph{Pattern-Driven Strategy.}
The pattern-driven strategy partitions coordinates along sign-consistent runs
of the initial direction $\mathbf{d}_0$.
For any run $R$ such that $R \cap B_k \neq \emptyset$, there exists a contiguous
sub-run $R' \subseteq R \cap B_k$ that is fully contained in the true block $B_k$.
By Assumption~\ref{asmp:purity}, such sub-runs admit positive expected
signed descent gain.

Importantly, the pattern-driven strategy explicitly enumerates runs that
intersect with $B_k$, and the number of such candidates is denoted by $\gamma_k$. Each candidate can be evaluated with $O(1)$ queries, and the existence of a
fully contained sub-run ensures that informative descent directions are exposed
without requiring further partitioning.
Therefore,
\[
T_{\mathrm{pat}} = O\!\left(\sum_{k=1}^K \gamma_k\right).
\]

\paragraph{Conclusion.}
The comparison highlights the efficiency of structural priors.
While the dyadic strategy requires deep partitioning proportional to the log-inverse of the block density ($\log(d/|B_k|)$) to resolve the signal, the pattern-driven strategy exploits the existing coarse alignment, requiring only $O(\gamma_k)$ queries per block.
Since $\gamma_k$ is a small constant (representing the number of initialization segments touching the block) and $\log(d/|B_k|)$ grows as the true gradient becomes higher-frequency, the pattern-driven strategy is superior.
\end{proof}
\noindent\textit{Remark.}
To substantiate the practical efficiency implied by Theorem~\ref{theo:complexity-fixed_app}, we quantify the effective dimensionality of the search space on real data. Empirically, we observe that the number of sign-consistent runs induced by the
initial direction $\mathbf{d}_0$ is significantly smaller than the number of
true gradient blocks across ImageNet images.
For example, over more than $50{,}000$ ground-truth gradient blocks, the number
of runs is approximately $20\text{k}$, $11\text{k}$, and $6\text{k}$ for
block sizes $w=\{4,8,16\}$, respectively.
This significant reduction suggests that individual runs typically span multiple gradient blocks and,
consequently, that the intersection between a run and a block often contains
a non-trivial contiguous sub-run fully contained within the block. For empirical verification of this theorem, please refer to Fig.~\ref{fig:gradsignstru} in the main text (for representative models) and Fig.~\ref{fig:gradsignstru_others} in this appendix (for diverse architectures). A detailed discussion is provided in Sec.~\ref{sec:pdores} of the main text.

\begin{table}[t!]
\small\tabcolsep=0.06cm
  \centering
  \caption{The untargeted attack performance on dense prediction tasks. The query limit is 1,000.}
    \begin{tabular}{c|c|c|c|c|c|c}
    \Xhline{0.8pt}
    \multirow{2}[1]{*}{Methods} & \multicolumn{3}{c|}{Object Detection} & \multicolumn{3}{c}{Segmentation} \\
\cline{2-7}          & ASR   & Avg.Q & Med.Q & ASR   & Avg.Q & Med.Q \\
    \hline
    ADBA  &                   39.5  &                     \underline{146}  &                        73  &                   43.5  &                     364  &                     298  \\
    Ours $\mathbf{d_n}$ & \textbf{                  70.5 } & \textbf{                       65 } & \underline{                       27 } & \textbf{                  68.0 } & {                    227 } & {                    109 } \\
    Ours$_\text{opt}$ &                   65.0  &                     167  &                       {  36 } &                   63.5  &                     \underline{161}  &                       \textbf{36}  \\
    Ours$_\text{dyn}$ &                  \underline{69.5}  &                     160  &                        \textbf{ 26}  &                   \underline{64.0}  &                     \textbf{ 160}  &                        \underline{41}  \\
    \Xhline{0.8pt}
    \end{tabular}%
  \label{tab:dense1000}%
\end{table}%

\begin{figure}[t!]
    \centering

    \begin{subfigure}{0.32\linewidth}
        \centering
        \includegraphics[width=\linewidth]{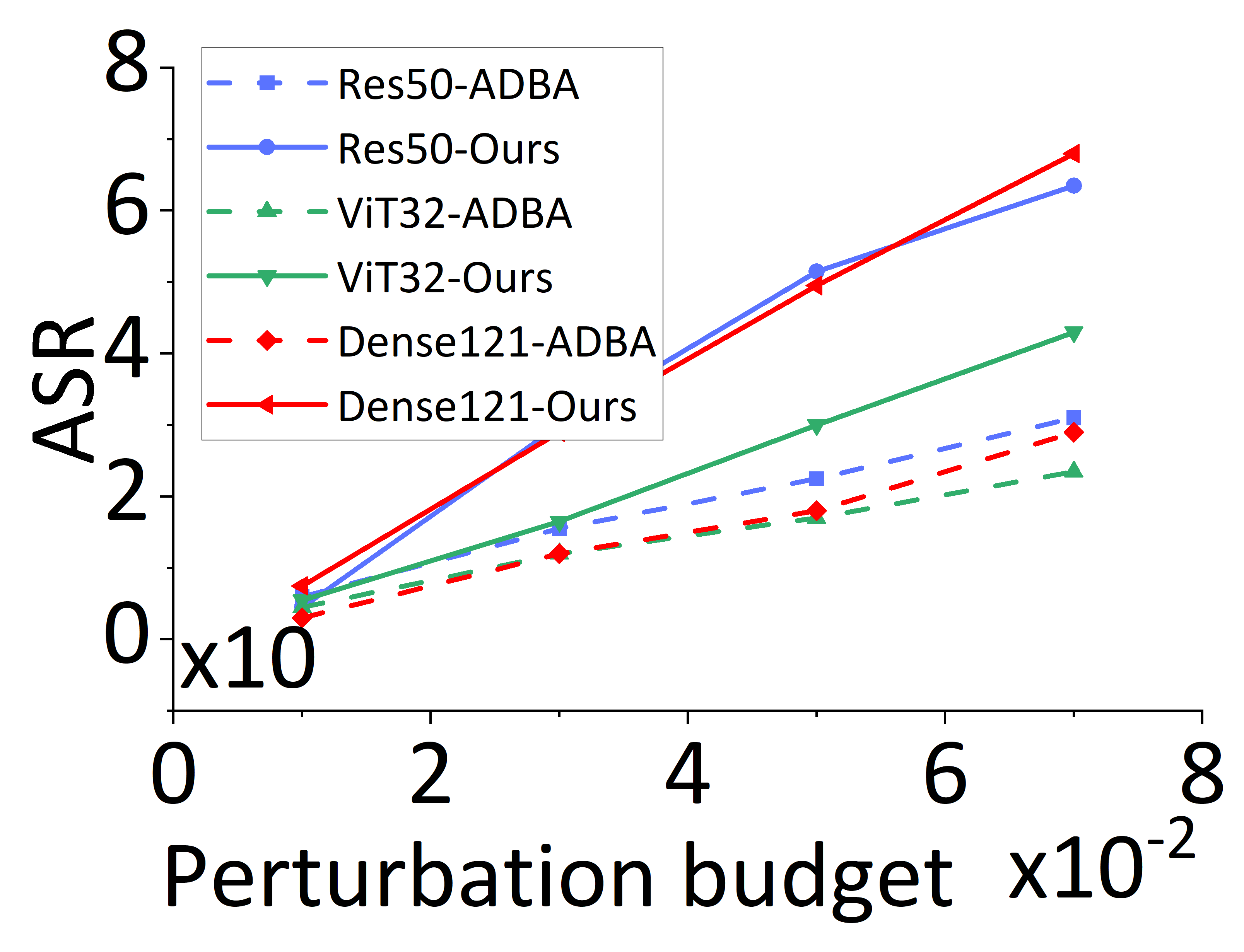}
        \caption{ASR}
        \label{fig:asr50}
    \end{subfigure}
    \hfill
    \begin{subfigure}{0.32\linewidth}
        \centering
        \includegraphics[width=\linewidth]{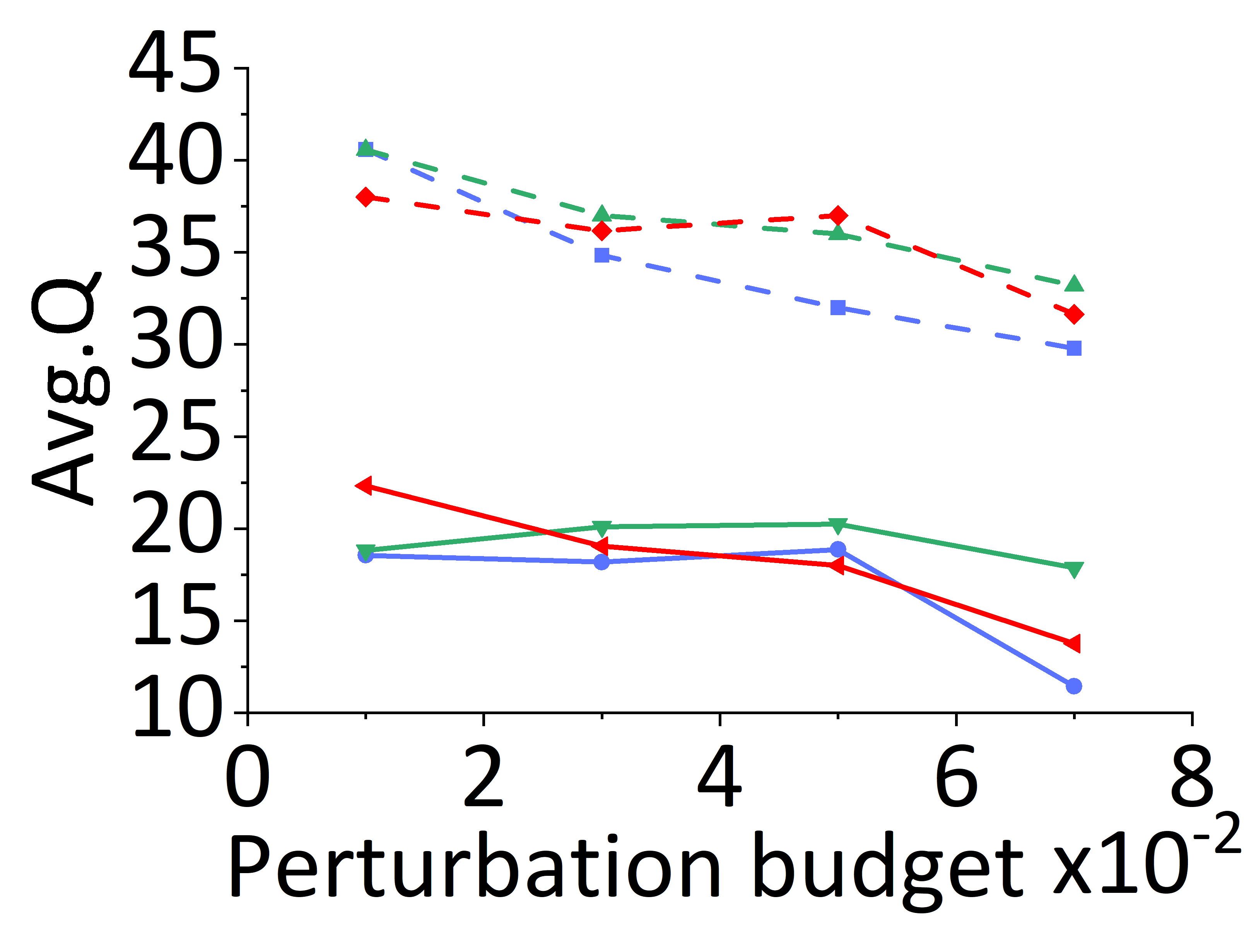}
        \caption{Avg.Q}
        \label{fig:avgq50}
    \end{subfigure}
    \hfill
    \begin{subfigure}{0.32\linewidth}
        \centering
        \includegraphics[width=\linewidth]{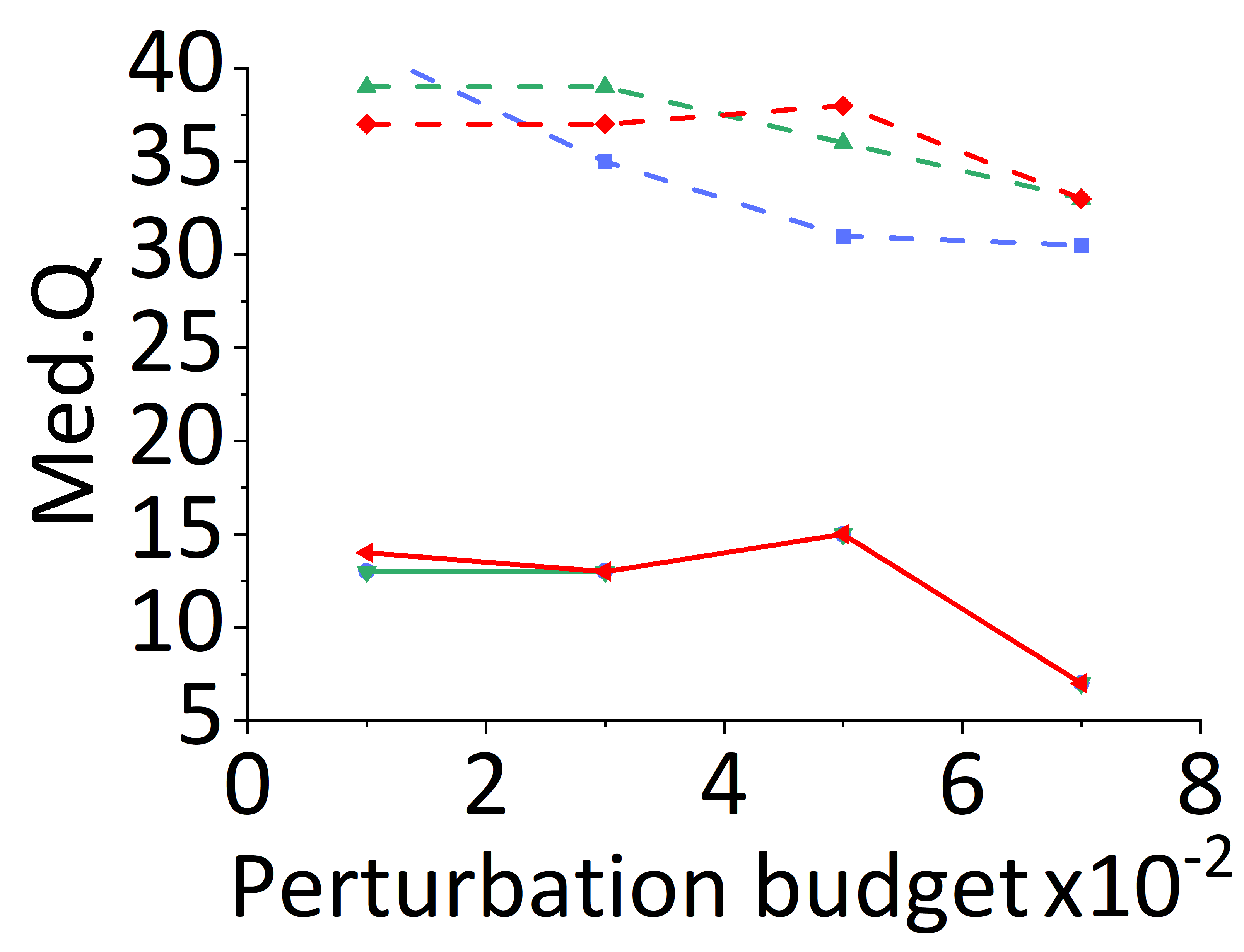}
        \caption{Med.Q}
        \label{fig:medq50}
    \end{subfigure}

    \vspace{0.3em}

    \begin{subfigure}{0.32\linewidth}
        \centering
        \includegraphics[width=\linewidth]{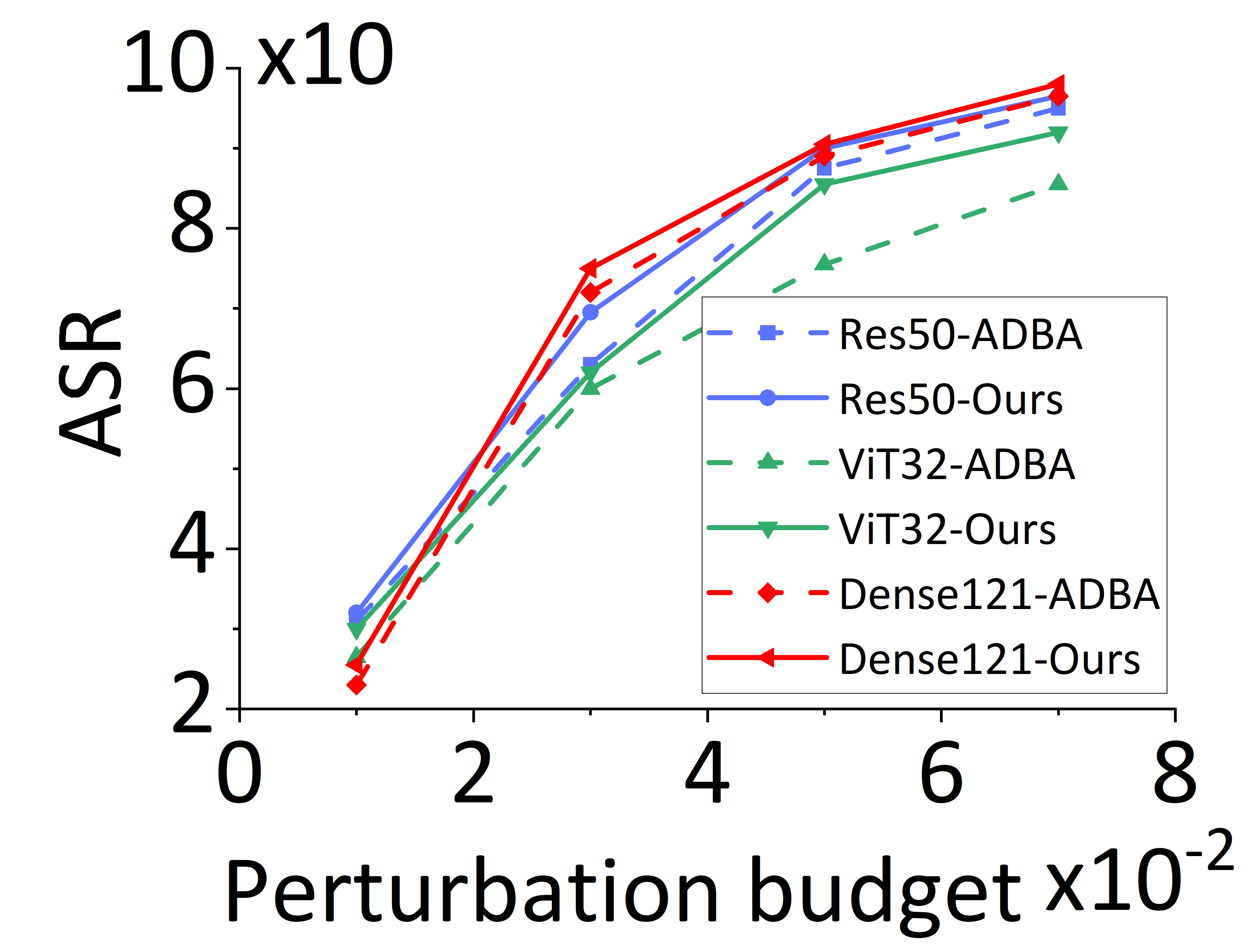}
        \caption{ASR}
        \label{fig:asr500}
    \end{subfigure}
    \hfill
    \begin{subfigure}{0.32\linewidth}
        \centering
        \includegraphics[width=\linewidth]{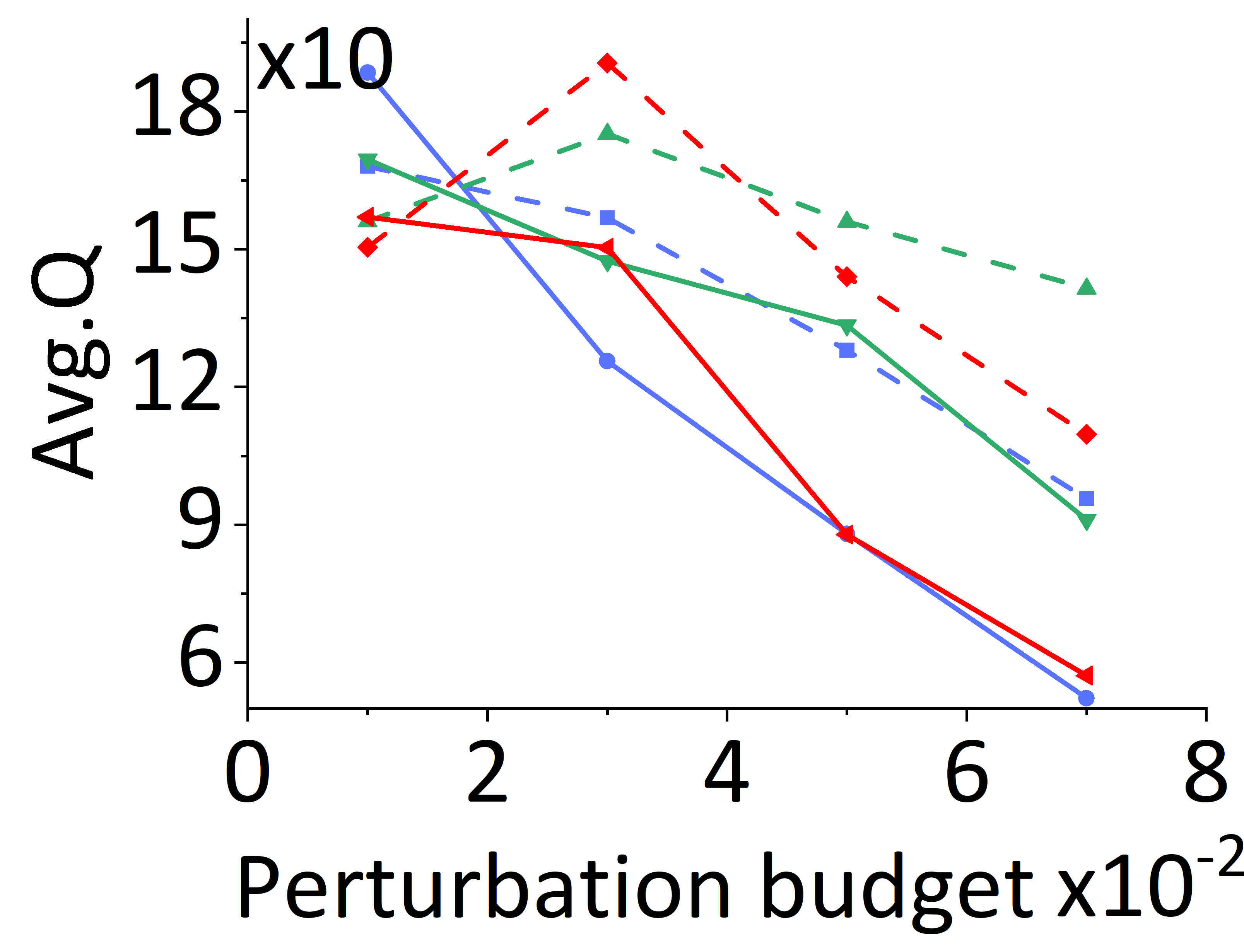}
        \caption{Avg.Q}
        \label{fig:avgq500}
    \end{subfigure}
    \hfill
    \begin{subfigure}{0.32\linewidth}
        \centering
        \includegraphics[width=\linewidth]{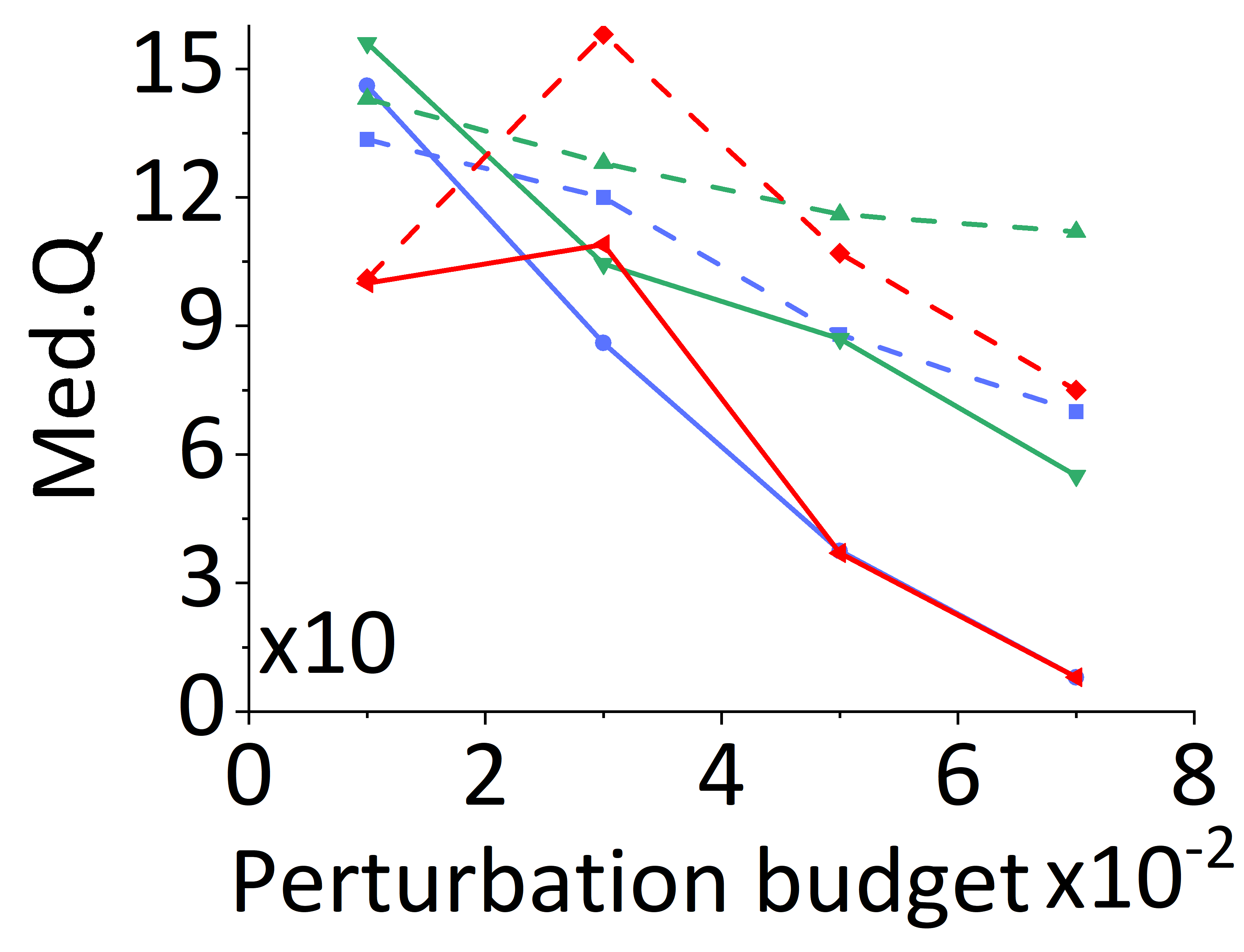}
        \caption{Med.Q}
        \label{fig:medq500}
    \end{subfigure}

    \caption{Attack performance on ImageNet under different perturbation budgets. The first row shows Max.Q=50, and the second row shows Max.Q=500.}
    \label{fig:differentPerturbation}
\end{figure}

\begin{table}[t!]
 \small\tabcolsep=0.03cm
  \centering
  \caption{The attack performance on ImageNet (adversarially trained WideResNet-50) under different perturbations.}
  \scalebox{0.75}{
    \begin{tabular}{c|ccc|ccc|ccc}
    \Xhline{0.8pt}
    Metrics & \multicolumn{3}{c|}{ASR} & \multicolumn{3}{c|}{Avg.Q} & \multicolumn{3}{c}{Med.Q} \\
    \hline
    $\epsilon$ ($\ell_\infty$)     & ADBA  & \textbf{Ours$_\text{opt}$}& \textbf{Ours$_\text{dyn}$}  & ADBA  & \textbf{Ours$_\text{opt}$}& \textbf{Ours$_\text{dyn}$}  & ADBA & \textbf{Ours$_\text{opt}$}& \textbf{Ours$_\text{dyn}$} \\
    \hline
    0.03  & 26.0    & \textbf{28.0} & \underline{27.0}  & 158   & 153  & 162& 56    & 92& 110\\
    0.05  & 45.0   & \textbf{50.5}& \underline{48.5} & 143   & 143 & 191 & 58    & 73& 90\\
    0.08  & \underline{73.0}    & \textbf{73.5} &\underline{73.0} & 96    & 116 & 121 & 41    & 40&48 \\
    0.10   & 85.5  & \textbf{87.0} & \underline{86.5}  & 74    & 89  &82  & 35    & 26 &33\\
    \Xhline{0.8pt}
    \end{tabular}%
    }
  \label{tab:wrs50}%
\end{table}%

\begin{table*}[t!]
\footnotesize
\tabcolsep=0.03cm
  \centering
  \caption{The untargeted attack performance on ImageNet (ViT-B-32) under $\ell_\infty=0.05$ constraint.}
  \scalebox{0.8}{
    \begin{tabular}{c|cccccccc|cccccccc|cccccccc}
   \Xhline{0.8pt}
    Metrics & \multicolumn{8}{c|}{ASR}      & \multicolumn{8}{c|}{Avg.Q} & \multicolumn{8}{c}{Med.Q} \\
    \hline
    Max$Q$  & HSJA  & BouR &  BouT &TtBA&HRayS  & ADBA   & \textbf{Ours$_\text{opt}$}& \textbf{Ours$_\text{dyn}$}  & HSJA  & BouR & BouT&TtBA & HRayS  & ADBA   & \textbf{Ours$_\text{opt}$}& \textbf{Ours$_\text{dyn}$}  & HSJA & BouR  & BouT &TtBA& HRayS  & ADBA   & \textbf{Ours$_\text{opt}$}& \textbf{Ours$_\text{dyn}$} \\
    \hline
    5     &   0.0   &    0.0    &  \underline{3.5}   &  0.0   &        0.5  & 0.0 & \textbf{8.0} &   -  &   -   &    -  &               5   & -&               5  &   -   &               5  &   - &   -   &   -   &               5  &  & \textbf{5} &    -   &               5 &   -  \\
   
    10    &   0.0   &  0.0    & \underline{9.0}    & 0.0 &   0.5  &  0.0  & \textbf{18.5} &  -&  -     & -      &               6  &    -  &         10  &    -   &               6  &  -&   -    &   -    &               6  &   -        &    10  &  -     &               6  &  -\\
   
    20    &   2.0  &       2.0  &   {9.0}  &    0.0    &     0.5  &  2.0 & \textbf{19.0} &\underline{17.0}&             14  &             14  &               6  & -          &    20  &             18  &             12 & 18&             14  &             14  &               6  &   -    &        20  &             18  &             10 & 19\\
   
    50    &           2.5  &           2.5  &           12.5  &   0.0   &     10.5  &  \underline{17.0}  & \textbf{30.0}& \textbf{30.0} &             14  &             14  &               8  &  -   &          42  &             36  &             20&25  &             14  &             14  &               6  & -  &            40  &             36  &             15&20  \\
   
    80    &           9.5  &           9.5  &           13.0  &   6.5   &     17.5  & {30.0}  & \underline{41.5}& \textbf{43.5} &             48  &             48  &               9  &   75   &         57  &             49  &             32 &             38  &             60  &             60  &               6  &  75  &           57  &             48  &             26&             27  \\
   
    100   &         10.0  &           8.0  &           15.5  &   9.0  &      22.0  &{35.5}  & \underline{44.5}& \textbf{48.0} &             45  &             43  &             25  &   80   &         66  &             55  &             36&45  &             60  &             60  &               6  &  80   &          72  &             52  &             29  &35\\
   
    200   &         11.0  &         12.0  &         15.5  &   15.0     &   36.5  &{52.5}  & \underline{64.5}& \textbf{65.0} &             77  &             82  &             46  &    115    &     103  &             86  &             71 &             71  &           100  &           100  &               6  &   120    &        98  &             74  &             54 &55 \\
   
    300   &         14.0  &         11.5  &         13.0  &     19.5  &    44.0  & {63.5}  & \underline{74.5}& \textbf{75.5} &             85  &             81  &             35  &  143   &        129  &           113  &           97 &           96 &           100  &             99  &               6  &  146   &        117  &             88  &             65 &71 \\
   
    400   &         12.5  &         11.5  &         14.0  &    21.0   &    51.0  &{70.5}  & \textbf{82.0}& \underline{78.5} &             83  &             81  &             85  &  167     &      163  &           136  &           121&           106  &             99  &             99  &               7  &  149    &       136  &           100  &             80&81  \\
   
    500   &         11.5  &         10.0  &         13.0  &    25.5   &    56.5  & {75.5}  & \textbf{85.5}& \underline{84.5} &           125  &             70  &             71  &    201    &     193  &           156  &           133 &130 &           100  &           100  &               7  &  152  &         153  &           116  &             87&93  \\
   
    1k  &         14.0  &         14.5  &         15.5  &     41.0   &   71.0  &  {85.5}  & \underline{89.5}& \textbf{92.5} &           160  &           124  &           119  &    435      &   299  &           217  &           160&176  &           150  &           150  &             13  &  415      &     194  &           146  &           104 &99 \\
   
    3k  &         15.0  &         14.0  &         14.5  &    67.0    &   92.0  & {97.5}  & \underline{98.0}& \textbf{98.5} &           184  &           132  &           151  &     927   &     602  &           405  &           295 &272 &           150  &           150  &           139  &  730    &       366  &           178  &           119  &103\\
   
    5k  &         16.0  &         12.0  &         17.0  &   76.0   &     95.5  & \underline{99.0}  & \textbf{99.5} & \textbf{99.5} &           340  &           128  &           158  &    1,363     &    700  &           453  &           341 &305 &           151  &           150  &               7  &    972  &       392  &           181  &           119 &105 \\
   
    10k &         17.5  &         13.0  &         14.0  &   83.0    &    99.0  & \textbf{100.0} & \textbf{100.0}& \textbf{100.0}&        1,813  &           141  &           153  &    1,801      &   961  &           502  &           346& 349 &           151  &           150  &             72  &    974   &      402  &           186  &           121& 106\\
    \Xhline{0.8pt}
    \end{tabular}%
    }
  \label{tab:vit}%
\end{table*}%

\begin{table*}[t!]
\footnotesize
\tabcolsep=0.03cm
  \centering
  \caption{The untargeted attack performance on ImageNet (DenseNet-121) under $\ell_\infty=0.05$ constraint. }
  \scalebox{0.8}{
    \begin{tabular}{c|cccccccc|cccccccc|cccccccc}
    \Xhline{0.8pt}
    Metrics & \multicolumn{8}{c|}{ASR}      & \multicolumn{8}{c|}{Avg.Q} & \multicolumn{8}{c}{Med.Q} \\
    \hline
    Max$Q$  & HSJA  & BouR &  BouT&TtBA &HRayS  & ADBA & \textbf{Ours$_\text{opt}$}& \textbf{Ours$_\text{dyn}$}  & HSJA  & BouR & BouT&TtBA & HRayS  & ADBA  & \textbf{Ours$_\text{opt}$}& \textbf{Ours$_\text{dyn}$}  & HSJA & BouR  & BouT &TtBA& HRayS  & ADBA  & \textbf{Ours$_\text{opt}$}& \textbf{Ours$_\text{dyn}$} \\
    \hline
   
    5     &    0.0   &         0.0    &   \underline{4.5}  & 0.0  &         2.0  & 0.0 & \textbf{20.0}&   -  &   -    &  -     &               5  &  -  &            5  &  -     &               5 &   -  &     -  &  -     &               5  &  -  &            5  &     -  &               5  &   - \\
   
    10    &       0.0   &      0.0  &   \underline{9.5}  &   0.0  &        2.0  &    0.0    & \textbf{32.5}&   -  &     -  &     -  &               6  &  -  &          10  &   -    &               6&   -   &    -   &  -     &               6  & - &           10  &     -  &               5 &   -  \\
   
    50    &           4.5  &           4.5  &         12.0  &    0.0   &   14.0  &  \underline{18.0}  & \textbf{49.5}&\textbf{49.5} &             14  &             14  &             13  &      -   &     39  &             37  &             18& 23 &             15  &             15  &               6  &       -    &   39  &             38  &             15& 20 \\
   
    100   &         13.5  &         11.5  &         16.0  & 8.5   &      22.0  &   {42.5}  & \underline{62.0} & \textbf{63.5}&             38  &             44  &             40  &    76   &       58  &             59  &             29  &33&             59  &             59  &             48  &     75    &     54  &             56  &             16 &20 \\
   
    200   &         15.5  &         16.0  &         16.0  & 11.5    &     37.5  &  {67.5}  & \textbf{78.5}& \textbf{79.0} &             70  &             71  &             55  &     95   &    105  &             91  &             53  &55&             99  &             99  &               6  &       78    &   96  &             80  &             28 &22 \\
   
    300   &         17.0  &         14.5  &         20.5  &  13.5  &      52.0  &   {79.5}  & \underline{84.5}& \textbf{85.5} &             88  &             71  &             49  &  109    &      148  &           115  &             66  &  69&           99  &             99  &               6  &     78     &  152  &             92  &             34&28  \\
   
    400   &         15.0  &         13.5  &         16.5  & 18.0   &      62.0  &  {85.5}  & \underline{87.5} & \textbf{89.0} &             98  &             83  &             87  & 158     &      184  &           131  &             76  & 81&            99  &             99  &             10  &     146   &    180  &           102  &             36&36  \\
   
    500   &         15.0  &         15.5  &         19.5  &  21.5  &      68.5  & {89.0} & \underline{90.5}& \textbf{92.0}&             95  &             73  &             77  &     213   &    210  &           144  &             88& 92 &             99  &             99  &               8  &    151    &    196  &           107  &             37 & 44\\
   
    700   &         17.0  &         14.0  &         19.5  &  27.5  &      77.5  &  {92.5}  & \underline{95.0} & \textbf{96.0} &           115  &           103  &             90  &  301      &    253  &           160  &           113 &           113  &           151  &           149  &               8  &     311     &  222  &           114  &             45 &46 \\
   
    1k  &         17.5  &         14.5  &         17.5  &  37.5  &      86.5  & \underline{96.5}  & \textbf{98.5} & \textbf{98.5} &           170  &           114  &           128  &    441    &    316  &           187  &           137 &           131  &           149  &           149  &           138  &       412   &  251  &           118  &             48  &48\\
   
    2k  &         18.5  &         15.0  &         15.5  &   55.0    &   96.5  & \underline{99.0}  & \textbf{100.0}& \textbf{100.0} &           334  &           110  &           115  &    813     &   423  &           218  &           155 &145 &           150  &           149  &               7  &   725    &     278  &           120  &             52&51  \\
   
    \Xhline{0.8pt}
    \end{tabular}%
    }
  \label{tab:dense121imgnet}%
\end{table*}%

\begin{table*}[t!]
\tabcolsep=0.05cm
  \centering
  \caption{The untargeted attack performance on ImageNet across diverse architectures under $\ell_\infty=0.05$ constraint.}
  \scalebox{0.7}{
    \begin{tabular}{c|ccc|ccc|ccc|ccc|ccc|ccc}
    \Xhline{0.8pt}
    Models & \multicolumn{3}{c|}{EfficientNet-B0} & \multicolumn{3}{c|}{ConvNeXt-base} & \multicolumn{3}{c|}{Inception-v3} & \multicolumn{3}{c}{SwinV2-T} & \multicolumn{3}{c}{DeiT-B} & \multicolumn{3}{c}{RegNet-X-8GF}\\
    \hline
    Max$Q$  & \multicolumn{1}{c}{ADBA} & \textbf{Ours$_\text{opt}$}& \textbf{Ours$_\text{dyn}$} & \multicolumn{1}{c}{ADBA} & \textbf{Ours$_\text{opt}$}& \textbf{Ours$_\text{dyn}$} & ADBA  & \textbf{Ours$_\text{opt}$}& \textbf{Ours$_\text{dyn}$}& ADBA & \textbf{Ours$_\text{opt}$}& \textbf{Ours$_\text{dyn}$} & \multicolumn{1}{c}{ADBA} & \textbf{Ours$_\text{opt}$}& \textbf{Ours$_\text{dyn}$}&{ADBA} & \textbf{Ours$_\text{opt}$}& \textbf{Ours$_\text{dyn}$} \\
    \hline
    50    &               24.5  &               \underline{36.5}  & \textbf{              44.5 } &                  4.5  &               \underline{17.0}  & \textbf{              18.5 } &               14.5  & \underline{44.0 } & \textbf{              48.5 } &               10.5  &   \underline{16.5}  & {\textbf{              18.5 }}&4.5&\underline{8.0}&\textbf{18.5}&20.5&35.5 &\textbf{65.0} \\
\cline{1-1}    500   &               \underline{85.5}  & \textbf{              87.0 } &               85.0  &               \underline{72.5}  & \textbf{              75.5 } &               68.0  &               85.0  &               \underline{85.5}  & \textbf{              87.0 } &               52.0  & \textbf{              67.5 } & \textbf{              67.5 } &49.5&\underline{55.0}&\textbf{62.0}&88.0&\textbf{92.0} &\underline{91.0}\\
    \Xhline{0.8pt}
    \end{tabular}%
  }
  \label{tab:morelinf}%
\end{table*}%

\begin{table*}[t!]
\small
\tabcolsep=0.02cm
  \centering
 \caption{The untargeted attack performance comparison on ImageNet under $\ell_2=5$ constraint.}  
 \scalebox{0.7}{
    \begin{tabular}{c|c|cccccccc|cccccccc|cccccccc}
    \Xhline{0.8pt}
  \multirow{3}[1]{*}{Victims}& Metrics & \multicolumn{8}{c|}{ASR}                      & \multicolumn{8}{c|}{Avg.Q}                    & \multicolumn{8}{c}{Med.Q} \\
 &  Max$Q$  & HSJA  & BouR  & BouT  & SurFree & Tangent & CGBA & \textbf{Ours$_\text{opt}$}& \textbf{Ours$_\text{dyn}$}  & HSJA  & BouR  & BouT  & SurFree & Tangent & CGBA &\textbf{Ours$_\text{opt}$}& \textbf{Ours$_\text{dyn}$}  & HSJA  & BouR  & BouT  & SurFree & Tangent &CGBA & \textbf{Ours$_\text{opt}$}& \textbf{Ours$_\text{dyn}$} \\
   \hline
   \multirow{4}[1]{*}{ResNet-50}& 100   &          5.5  &             5.5  &          5.5  &           8.5  &          8.0 &9.0 & \underline{          14.5 } & \textbf{          18.0 } &          39  &            31  &            40  &            43  &            62  &73&            39  &44 &            50  &            19  &            37  &            37  &            44  &    73&        36  & 39\\
   
    &500   &          8.0  &             6.5  &          6.5  &         18.0  &         18.0&14.0  & \textbf{34.5}& \underline{32.0} &          104  &            80  &           133  &           168  &           199&148  &          171  &160 &            89  &            79  &            76  &           110  &           150  &81&          125 &80  \\
   
    &1k    &          9.5  & {            7.5 } &          8.0  &         26.0  & {        25.0 } &18.0& \textbf{          44.5 }&  \textbf{          44.5 } &           162  & {         124 } &           163  &           351  &           415&308  &          295 & 313 &           140  & {         136 } &           126  &           280  &           310&137  &          202 & 243 \\
   
   & 5k    &         10.0  &             9.0  &          9.0  &         52.0  &         61.5&34.0  &\textbf{67.5} &  \underline{63.5}&          884  &          525  &           423  &        1,410&1,347  &        1,590  &1,088   & 953    &           141  &          138  &           128  &        1,062  &        1,304 &958 & 572 &477 \\

    \hline
      \multirow{4}[1]{*}{Dense-121}&100   &  5.0        &              5.0  & 4.0     &           7.5  &           7.0&8.5  & \textbf{ 15.5 }& \underline{ 15.0 } &      42         &             49  &              19  &              46  &              59&  74& 32&32  &50  &             47  &              9  &              39  &              45 &74 &             24&22   \\
   
   & 500   &    8.0       & 6.0  &         6.5  &         18.0  &         13.0  &14.0 &\underline{ 25.0 } & \textbf{ 32.5 }&  113            &           139  &            129  &            306  &            490&181  & 132& 157 &       90         &           88  &78  &            119  &            150  & 97&          59&123   \\
   
   & 1k  &    5.5    & 6.5  & 8.0  &         22.5  &         22.0 &19.5 & \underline{ 33.5 }& \textbf{ 44.0 } &       131       &           113  &            249  &            306  &            490 &364 & 264&304  &    140          &  137  &131  &            229  &            485  &291&           143 &237  \\
  & 5k   &     9.5     & 8.5  &       7.0 &        53.0  &         66.0  &37.0&  \textbf{67.5} & \underline{66.5}    &  568    &        509  &        367 &         1,727  &        1,889&1,469  &    1,323&  1,045  &   142     &     138  &         128  &   1,334  &   1,541  &1,110& 1,052 &567 \\  
    \hline
      \multirow{4}[1]{*}{ViT-B-32}&100   &           3.5  &              6.0  & 5     &           5.5  &           6.5&7.5  & \underline{ 15.5 }& \textbf{ 18.5 } &              21  &             43  &              59  &              53  &              64 &73 & 42&45  &18  &             51  &              39  &              48  &              70  &73&             40& 41  \\
   
   & 500   &         10.0  & 13.5  &         11.0  &         13.5  &         16.5  &20.0& \textbf{ 30.5 }& \textbf{ 30.5 } &            157  &           236  &            158  &            171  &            265 &196 & 138& 118 &              91  &           182  &79  &            106  &            219  & 214&          99 & 76 \\
   
   & 1k  &         11.5  & 20.5  & 20.5  &         19.5  &         22.0  & 27.0&\underline{ 36.5 }& \textbf{ 38.5 } &            263  &           486  &            504  &            335  &            493&362  & 231& 247 &            141  &           500  &            490  &            287  &            482 &247 & 138 & 119 \\
  & 5k   &      29.5     & 54.0  &       57.5  &        47.0  &         55.5  & 52.5& \underline{59.5}& \textbf{62.0}     &  1,768    &        1,907  &        1,985  &         1,847  &        1,708 &1,331 &    1,039&1,174    &   1,296     &     1,832  &         1,823  &   1,422  &   1,199 &871 & 439& 550 \\   
  \Xhline{0.8pt}
    \end{tabular}%
    }
  \label{tab:l2resvit}%
\end{table*}%

\begin{figure*}
\centering
  \includegraphics[width=\linewidth]{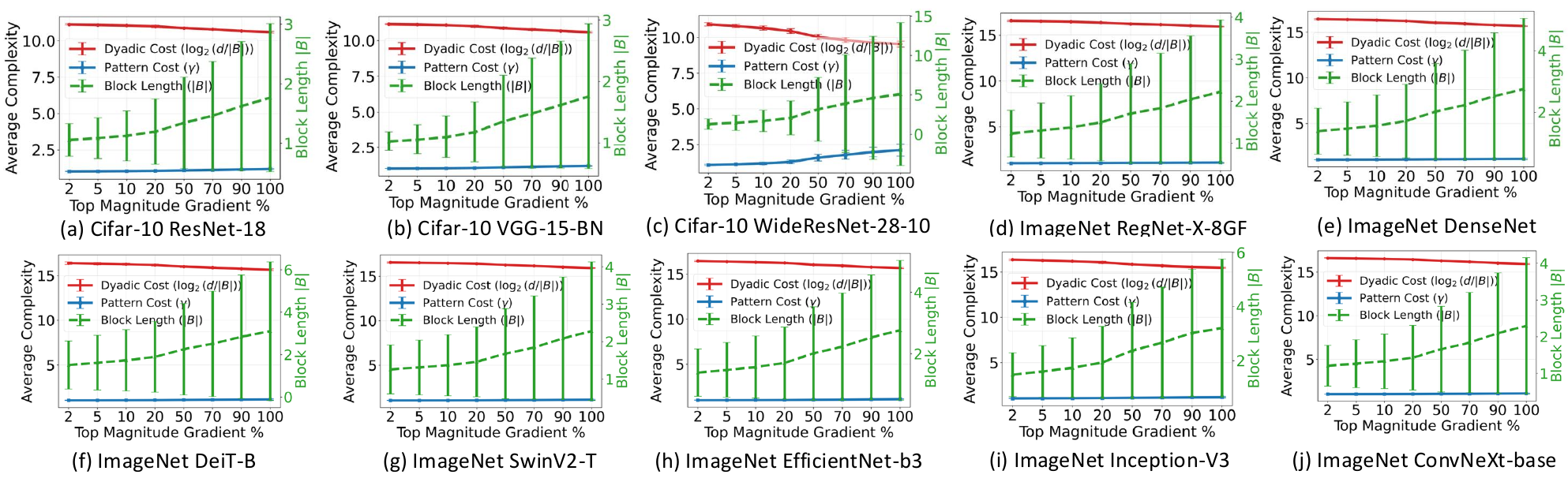}
  \caption{Extensive validation of Theorem~\ref{theo:complexity-fixed} on diverse datasets and architectures. The results consistently confirm that salient gradients exhibit rapid sign oscillations ($|B|\approx1$), validating the universality of the query complexity gap derived in Theorem~\ref{theo:complexity-fixed}.} 

  \label{fig:gradsignstru_others}
\end{figure*}

\begin{table}[t!]
\small\tabcolsep=0.02cm
  \centering
  \caption{The untargeted attack performance on ImageNet-C and PathMNIST, respectively. Max.Q$=50$.}
  \scalebox{0.7}{
    \begin{tabular}{c|c|ccc|ccc|ccc}
    \Xhline{0.8pt}
    \multirow{2}[1]{*}{Datasets} & \multirow{2}[1]{*}{Models} & \multicolumn{3}{c|}{ASR} & \multicolumn{3}{c|}{Avg.Q} & \multicolumn{3}{c}{Med.Q} \\
\cline{3-11}          &       & \multicolumn{1}{c}{ADBA}  & \textbf{Ours$_\text{opt}$}& \textbf{Ours$_\text{dyn}$} & \multicolumn{1}{c}{ADBA}  & \textbf{Ours$_\text{opt}$}& \textbf{Ours$_\text{dyn}$}& \multicolumn{1}{c}{ADBA}  & \textbf{Ours$_\text{opt}$}& \textbf{Ours$_\text{dyn}$}\\
    \hline
    \multirow{2}[1]{*}{ImgNC} & NMix &              33.5  & \textbf{             50.0 } &             \underline{48.5}  &                  30  & \textbf{                 13 } &   \underline{22}  & 29    & \textbf{6} &  \underline{20} \\
\cline{2-2}          & HMany &              43.5  & \textbf{             59.0 } &               \underline{56.5}  &                  31  & \textbf{                 16 } &                   \underline{22}  & 29    & \textbf{6} &  \underline{20} \\
\cline{1-2}    PMNIST & MLP &              21.0  & \textbf{             35.0 } &               \underline{29.5}  &                  37  & \textbf{                 31 } &                   \underline{35}  & 37    & \textbf{31} &  \underline{36} \\
    \Xhline{0.8pt}
    \end{tabular}%
    }
  \label{tab:otherdataset}%
\end{table}%

\section{Detailed Experimental Settings}\label{app:expsetting}

\noindent{\textbf{Victim Models.}} For CIFAR-10, we select standard ResNet-18 and VGG-16-BN models~\cite{Torchvision}, alongside an adversarially trained WideResNet~\cite{croce2021robustbench}. For ImageNet-1K, we target a diverse array of high-performing architectures from the Torchvision library~\cite{Torchvision}, including ResNet-50, ViT-B-32, DenseNet-121, EfficientNet-B0, ConvNeXt-Base, Inception-v3, SwinV2-T, DeiT-B, and RegNet-X-8GF. Additionally, we evaluate an adversarially trained WideResNet-50  (trained with $\ell_\infty=8/255$ \cite{robustness}). Regarding foundation models~\cite{radford2021learning}, we employ the ViT-H-14-CLIPA-336~\cite{ilharco_gabriel_2021_5143773}, chosen for its superior zero-shot performance on ObjectNet. For downstream tasks, we evaluate FasterRCNN-ResNet50-FPN-V2~\cite{Torchvision} (object detection) and SAM-ViT-B~\cite{kirillov2023segment} (segmentation). Finally, we adopt the top-performance NoisyMix~\cite{erichson2022noisymix} and HMany~\cite{hendrycks2021many} models from RobustBench library~\cite{croce2021robustbench} for ImageNet-C, and a standard MLP from the MedMNIST library~\cite{medmnistv2,medmnistv1} for PathMNIST.

\noindent{\textbf{Image Sizes.}}
All inputs are processed as three-channel RGB images. Grayscale images are converted to RGB via channel replication before processing. We use $32\times32\times3$ for CIFAR-10, $224\times224\times3$ for ImageNet-1K and ImageNet-C, $336\times336\times3$ for ObjectNet/CLIP, and $28\times28\times3$ for PathMNIST.  For MS-COCO object detection, we keep the original image resolution, whose height and width are typically around 300--700 pixels. When the image dimensions are not divisible by the effective block size required by our block-wise initialization, we pad the bottom and right borders to the nearest divisible size using replication padding. For SA-1B/SAM, following~\cite{zhou2024darksam}, all images are resized to $1500\times2250\times3$ before attack.

\noindent\textbf{Implementation Details.} In DPAttack, for low-frequency filtering, the decomposition level is set to its maximum feasible depth, $dl = \lfloor\log_2(\bar{w})\rfloor$. Empirical analysis shows that while results are marginally sensitive to $dl$ within its valid range, larger values generally yield superior performance. For baselines requiring Monte Carlo gradient estimation (HSJA, Bounce, and Tangent), we tune their batch sizes via grid search to ensure fairness in low-query regimes. This adjustment is necessary because their default batch size ($100$) leads to $0\%$ ASR whenever the query budget $Q_{\max} < 100$.

\noindent\textbf{Dense Prediction Tasks.} For the results in Tables~\ref{tab:odandseg} and~\ref{tab:dense1000}, we define task-specific success criteria: (i) Object Detection: an attack succeeds if the per-image Average Precision (AP) drops below $0.2$; (ii) Segmentation: success is defined as reducing the prediction Intersection over Union (IoU) to $<0.5$. Strictly adhering to the hard-label setting, our algorithm only accesses predicted bounding boxes (object detection) or masks (segmentation) during the attack process. Confidence scores are solely utilized for final AP calculation during the evaluation phase and are never used to guide the attack generation.

\section{Additional Results under $\ell_\infty$ Constraints}\label{sec:linfmore}
\subsection{ Diverse Perturbation Budgets}\label{app:differentPerturbation500}
Fig.~\ref{fig:differentPerturbation} illustrates the attack performance of DPAttack versus ADBA under varying $\ell_\infty$ constraints. 
In the low-query regime (Figs.~\ref{fig:differentPerturbation}(a), (c), (e)): our method maintains a clear ASR advantage that scales with the perturbation magnitude, while requiring significantly fewer queries to find successful adversarial examples. In the high-query regime (Figs.~\ref{fig:differentPerturbation}(b), (d), (f)): DPAttack remains superior across different perturbation levels. The only marginal exception occurs at an extremely small bound of $\ell_\infty=0.01$, where our method shows a slightly higher average query count. This is primarily because DPAttack successfully converts ``hard samples'' that ADBA fails to attack; these more challenging cases naturally demand higher query counts, which increases the overall average but results in a superior final ASR.

\subsection{Diverse Architectures and Domains}
To complement the results in Sec.~\ref{sec:expstdmodel}, we provide extensive evaluations across diverse architectures and datasets. Tables~\ref{tab:cifarvgg},~\ref{tab:cifar10res18},~\ref{tab:vit}, and~\ref{tab:dense121imgnet} report the performance on VGG-16-BN and ResNet-18 (CIFAR-10), as well as ViT-B-32 and DenseNet-121 (ImageNet) across multiple query budgets. Table~\ref{tab:morelinf} further extends these evaluations to a wider range of SOTA ImageNet architectures. Beyond standard natural images, we evaluate DPAttack on corrupted (ImageNet-C) and biomedical (PathMNIST) domains (see Table~\ref{tab:otherdataset}). Our method consistently outperforms SOTA hard-label attacks in both ASR and query efficiency across all tested models, confirming the generalization of our proposed method.

\subsection{Robustness against Adversarial Training}\label{app:wrs50}
Table~\ref{tab:wrs50} presents comparative results against an adversarially trained WideResNet-50 across various perturbation levels. Our method achieves a consistently higher ASR than the second-best baseline, ADBA, while maintaining competitive query efficiency. Notably, the ASR gains—though numerically modest—represent a significant improvement in the challenging context of attacking robust models.

\subsection{Generalization on Dense Prediction Tasks}

Beyond the results in Table~\ref{tab:odandseg} of the main manuscript, which demonstrate the superiority of our method under a limited query budget ($50$), Table~\ref{tab:dense1000} further shows that our approach maintains a performance margin of over $30\%$ even when the query limit is relaxed to $1,000$.

\section{Additional Results Under $\ell_2$ Constraint}\label{sec:l2app}

While DPAttack is primarily tailored for the structural characteristics of $\ell_\infty$ perturbations, we evaluate its performance under $\ell_2$ constraints to demonstrate its versatility. Table~\ref{tab:l2resvit} compares DPAttack against SOTA methods specifically optimized for the $\ell_2$ metric, such as SurFree and Tangent. Despite its $\ell_\infty$-focused design, DPAttack achieves superior ASR and query efficiency—particularly in low-query regimes—validating its robust generalization across different distance metrics.

\begin{table}[t]
\centering
\footnotesize\tabcolsep=0.1cm
\caption{Comparison with CA~\cite{hong2024certifiable} under the same settings on ImageNet with ResNet-18. The values are CA/\textbf{Ours$_\text{dyn}$}.}
\label{tab:ca_compare}
\begin{tabular}{lccccc}
\Xhline{0.8pt}
Defense & ASR $\uparrow$ & Det$R\downarrow$ & Avg.Q$\downarrow$ & Med.Q$\downarrow$& $\ell_2\downarrow$ \\
\hline
Blacklight & 100\% / 100\% & 0\% / 0\% & 603 / \textbf{16}& 603 / \textbf{9} & 33.0 / \textbf{31.3} \\
RAND-Pre   & 100\% / 100\% & --          & 603 / \textbf{29} & 603 / \textbf{16}& 32.2 / \textbf{28.2} \\
RAND-Post  & 100\% / 100\% & --          & 603 / \textbf{20}& 603 / \textbf{16} & 32.7 / \textbf{29.2} \\
\Xhline{0.8pt}
\end{tabular}
\end{table}

\subsection{Comparison with Certifiable Attack (CA)}
\label{app:ca}

We further compare DPAttack with the recent Certifiable Attack (CA)~\cite{hong2024certifiable}, which represents a different attack paradigm based on randomized adversarial distributions. Unlike empirical hard-label attacks that search for a single adversarial example, CA constructs an adversarial distribution with a guaranteed attack success probability, and adversarial examples sampled from this distribution can be used without final verification queries. This inherent randomness makes CA particularly effective against stateful detection mechanisms such as Blacklight. CA provides two localization strategies: Smoothed Self-Supervised Perturbation (SSSP) localization and binary-search localization. SSSP leverages a pretrained feature extractor as an additional surrogate prior, while binary-search localization does not rely on such pretrained adversarial or feature priors. Since DPAttack also does not use adversarial-example priors or surrogate feature extractors, we compare against the binary-search version of CA for a closer surrogate-free comparison.

Table~\ref{tab:ca_compare} reports the comparison under the same settings as CA on ImageNet with ResNet-18. DPAttack achieves the same ASR and detection rate as CA across Blacklight, RAND-Pre~\cite{qin2021random}, and the RAND-Post setting used in CA, while requiring about $25\times$ fewer queries and smaller perturbations. These results suggest that although CA provides a strong certifiable and verification-free attack paradigm, DPAttack offers a more query-efficient empirical alternative when the goal is to obtain successful hard-label AEs under strict query budgets.

\end{document}